%% file: LearningAdversarial_rev_2_PUBLIC.tex
\pdfoutput=1

\documentclass[mnsc,nonblindrev]{informs_modified}
\usepackage{bbm, xcolor, multirow, multicol}
\usepackage[normalem]{ulem}
\usepackage{algorithm,caption}
\usepackage{algpseudocode}
\usepackage{natbib}
 \bibpunct[, ]{(}{)}{,}{a}{}{,}%
 \def\bibfont{\small}%
\OneAndAHalfSpacedXI
\TheoremsNumberedThrough     
\ECRepeatTheorems
\EquationsNumberedThrough    

\allowdisplaybreaks

\newcommand{\bE}{\mathbb{E}}
\newcommand{\bN}{\mathbb{N}}
\newcommand{\bR}{\mathbb{R}}

\newcommand{\bI}{\mathbbm{1}}

\newcommand{\vare}{\varepsilon}

\newcommand{\Ber}{\mathsf{Ber}}
\newcommand{\KL}{\mathsf{KL}}
\newcommand{\ALG}{\mathsf{ALG}}
\newcommand{\OPT}{\mathsf{OPT}}
\newcommand{\REG}{\mathsf{REG}}

\newcommand{\cA}{\mathcal{A}}
\newcommand{\cP}{\mathcal{P}}
\newcommand{\cT}{\mathcal{T}}
\newcommand{\cX}{\mathcal{X}}
\newcommand{\cY}{{\{0,1\}^n}}

\newcommand{\cF}{\mathcal{F}}
\newcommand{\cO}{\mathcal{O}}

\newcommand{\vzero}{\mathbf{0}}
\newcommand{\ve}{\mathbf{e}}

\newcommand{\vy}{\mathbf{y}}
\newcommand{\vz}{\mathbf{z}}
\newcommand{\vZ}{\mathbf{Z}}
\newcommand{\vell}{\boldsymbol{\ell}}
\newcommand{\ovZ}{\overline{\vZ}}

\newcommand{\oq}{\overline{q}}

\newcommand{\oa}{\overline{a}}
\newcommand{\oQ}{\overline{Q}}
\newcommand{\oY}{\overline{Y}}

\newcommand{\bmin}{b_\text{min}}

\begin{document}

\TITLE{Inventory Balancing with Online Learning}

\ARTICLEAUTHORS{
\AUTHOR{Wang Chi Cheung}
\AFF{National University of Singapore, NUS Engineering, Department of Industrial Systems Engineering and Management, Singapore, SG 117576, \EMAIL{isecwc@nus.edu.sg}}
\AUTHOR{Will Ma}
\AFF{Graduate School of Business, Columbia University, New York, NY 10027, \EMAIL{wm2428@gsb.columbia.edu}}
\AUTHOR{David Simchi-Levi}
\AFF{Institute for Data, Systems, and Society, Department of Civil and Environmental Engineering, and Operations Research Center, Massachusetts Institute of Technology, Cambridge, MA 02139, \EMAIL{dslevi@mit.edu}}
\AUTHOR{Xinshang Wang}
\AFF{Institute for Data, Systems, and Society, Massachusetts Institute of Technology, Cambridge, MA 02139, \EMAIL{xinshang@mit.edu}}
}

\MANUSCRIPTNO{MS-SPI-18-02360.R2}

\ABSTRACT{

We study a general problem of allocating limited resources to heterogeneous customers over time under model uncertainty. Each type of customer can be serviced using different actions, each of which stochastically consumes some combination of resources, and returns different rewards for the resources consumed. We consider a general model where the resource consumption distribution associated with each (customer type, action)-combination is not known, but is consistent and can be learned over time. In addition, the sequence of customer types to arrive over time is arbitrary and completely unknown.

We overcome both the challenges of model uncertainty and customer heterogeneity by judiciously synthesizing two algorithmic frameworks from the literature: inventory balancing, which ``reserves'' a portion of each resource for high-reward customer types that could later arrive, based on competitive ratio analysis; and online learning, which ``explores'' the resource consumption distributions for each customer type under different actions, based on regret analysis. We define an auxiliary problem, which allows for existing competitive ratio and regret bounds to be seamlessly integrated.
Furthermore, we propose a new variant of UCB, dubbed LazyUCB, which conducts less exploration in a bid to focus on ``exploitation'', in view of the resource scarcity. Finally, we construct an information-theoretic family of counterexamples to show that our integrated framework achieves the best possible performance guarantee.

We demonstrate the efficacy of our algorithms both on synthetic instances generated for the online matching with stochastic rewards problem under unknown probabilities, and on a publicly available hotel data set. Our framework is highly practical in that it requires no historical data (no fitted customer choice models, nor forecasting of customer arrival patterns) and can be used to initialize allocation strategies in fast-changing environments.
}

\maketitle

\section{Introduction}

Online resource allocation is a fundamental topic in many applications of operations research, such as revenue management, display advertisement allocation, and appointment scheduling. In each of these settings, an online platform needs to allocate limited resources to a heterogeneous pool of customers arriving in real time, while maximizing the cumulative reward. The starting amount of each resource is exogenous, and these resources cannot be replenished during the time horizon.

In many applications, the online platform can observe a list of feature values associated with each arriving customer, which allows for allocation decisions to be customized in real time. For example, a display advertising platform operator is usually provided with the internet cookie from a website visitor, upon the visitor's arrival.  Consequently, the operator is able to display relevant advertisements to each website visitor based on this cookie, in a bid to maximize the total revenue earned from clicks on these advertisements.

To achieve an optimal allocation in the presence of resource constraints, the platform's allocation decision at any moment has to take into account the features of both the current customer as well as the customers who will arrive in the future.
In the preceding example, advertisements have daily budgets on how often they can be shown, making it suboptimal for the operator to behave myopically for the current visitor \citep{MSVV07,BJN07}.
In another example of selling airline tickets, it is profitable to judiciously reserve a number of seats for business class customers, who often purchase tickets close to departure time \citep{TvR98,BQ09}. Finally, in healthcare applications, when making advance appointments for out-patients, it is critical to reserve certain physicians' hours for urgent patients \citep{feldman2014appointment,Tru15}. In all of these examples, the platform's central task is to \textit{reserve} the right amount of each resource for future customers so as to maximize the total reward.

While resource reservation is vital for optimizing online resource allocations, the implementation of resource reservation is hindered by the following two challenges. First, the online platform often lacks an accurate forecast about the arrival patterns of future demand. 
Second, the online platform is often uncertain about the relationship between an arriving customer's expected behavior, e.g.\ click-through rate on an ad, and their observed features.

These challenges in implementing resource reservation raise the following research question: \textit{Can the online platform perform resource reservation effectively, in the absence of any demand forecast model and under uncertain customer behavior?}

\subsection{Description of Model and Contributions}
Initially there is a finite and discrete amount of inventory for each of multiple resources. Resources can be converted to rewards when they are consumed by a customer. Customers arrive sequentially, each of whom is characterized by a context vector that describes the customer's features. Upon the arrival of each customer, an action is selected,
after which there is a stochastic consumption of resources, which determines the reward collected.
For example, the action can represent offering a specific item to the customer at a particular price, and the stochastic consumption can correspond to whether the customer chooses to purchase.
The distribution of this stochastic consumption depends both on the customer's features and the action selected.
The objective is to maximize the total expected reward collected from the resources during a finite time horizon of unknown length.

We highlight two salient aspects of our model:
\begin{enumerate}
\item The number of future customers and their context vectors are unknown and chosen by an adversary. As a result, historical observations do not provide any information about future arrivals. 
\item For each potential combination of context vector and action, there is a fixed unknown distribution over the consumption outcome. That is, two customers arriving at different time periods with identical context vectors will have the same consumption distribution. 
As a concrete example, in e-commerce, the context vector represents the characteristics (e.g., age, location) of an online shopper. We are assuming that the conversion rate only depends on the characteristics of the shopper and the product offered, but not the time. 
  The platform needs to learn these conversion rates in an online fashion.
\end{enumerate}



Each of these two aspects has been studied extensively, but only separately, in the literature (reviewed in Section~\ref{sec:review}).
In models with the first aspect alone, model parameters on customer behavior such as purchase probabilities are known, and the difficulty is in conducting resource reservation without any demand forecast.
The conventional approach is to
set an opportunity cost for each resource which is increasing in how quickly it has already been consumed, using these to ideally ``balance'' the consumption rates of the different resources.
We call such techniques \textit{Inventory Balancing}.
Meanwhile, in models with the second aspect alone, the trade-off is between ``exploring'' the probabilities from playing different actions on different customers, and ``exploiting'' actions which are known to yield desirable outcomes.
\textit{Online Learning} techniques are designed for managing this trade-off.
However, in the presence of resource constraints, work on online learning has assumed that the context vectors are drawn i.i.d. from a known distribution, and there is no element of ``hedging'' against an adversarial input sequence.

In our work, we present a unified analysis of the online allocation problem in the presence of both of these aspects. We proceed to describe our contributions.

\textbf{IBOL algorithmic framework with performance guarantees.}
We propose a framework that integrates the Inventory Balancing technique with a broad class of Online Learning algorithms, which we dub IBOL, short for ``Inventory Balancing with Online Learning''. Our framework produces online allocation algorithms with performance guarantees of the form
\begin{equation}\label{eq:guaranteeInformal}
\bE[\ALG] \geq \alpha\cdot\OPT - \REG,
\end{equation}
where $\ALG$ is the total reward earned by IBOL; $\OPT$ is an LP-based
upper bound on the expected revenue of an optimal algorithm which knows both the arrival sequence and the unknown probabilities in advance; and
$\REG$ represents the \textit{regret}, i.e., the loss from having to explore the unknown probabilities. $\REG$ in fact represents the optimality gap in an \textit{auxiliary problem} we define, which is a non-stationary stochastic multi-armed bandits problem. The non-stationarity in our auxiliary problem arises from the adversarial uncertainty in customers' arrivals. The factor $\alpha\in (0, 1)$ in our guarantee~\eqref{eq:guaranteeInformal} can be viewed as the \textit{competitive ratio} when the probabilities are known, i.e., when $\REG = 0$.

\textbf{Asymptotically-tight guarantee for online matching with unknown stochastic rewards.}
As an application of our framework, we analyze an online bipartite matching problem in which edges, upon being selected, only get matched with an unknown probability. We first apply the IBOL algorithm with an Upper Confidence Bound (UCB) oracle, which is based on the optimistic estimation approach in \citet{ACF02}.\footnote{Essentially, under the optimistic estimation approach for multi-armed bandits, the decision maker adds an \emph{optimistic bonus} to the maximum likelihood estimate on each arm's latent reward, which encourages the exploration of the under-explored arms.} We establish the performance guarantee
\begin{align}
\bE[\ALG] \geq \left(1 - \frac{1}{e}\right)\OPT  -  \tilde O(\sqrt{\OPT}). \label{eqn:introEqn}
\end{align}
The $\tilde{O}(\cdot)$ notation hides the logarithmic dependence on $T$, the number of time rounds in the problem, as well as the dependence on model parameters other than $T$. A consequence of~\eqref{eqn:introEqn} is that $\bE[\ALG]/\OPT$ is bounded from below by $1 - 1/e  - \tilde{O}(1/\sqrt{\OPT})$, which approaches the best-possible competitive ratio of $1-1/e$ as $\OPT$ becomes large (i.e.\ as the regret from learning the matching probabilities becomes negligible).

Importantly, we also show the guarantee in (\ref{eqn:introEqn}), which can be re-expressed as $\OPT-\bE[\ALG] \leq \OPT/e  -  \tilde O(\sqrt{\OPT})$, to be tight.
That is, the loss of $\OPT/e$ is unavoidable due to not knowing the arrival sequence in advance, and the loss of $\tilde O(\sqrt{\OPT})$ is unavoidable due to not knowing the matching probabilities in advance.
The fact that these losses \textit{accumulate} instead of alleviating each other was surprising to us, and to our knowledge, requires a non-trivial new analysis combining Yao's minimax principle with information theory.
We elaborate further when we present our counterexample that demonstrates this tightness.

\textbf{$\epsilon$-perturbed potential function and $(1+\epsilon)$-relaxed regret.}
Our IBOL framework also has the flexibility of an additional parameter $\epsilon\in[0,1]$, which allows the Online Learning algorithm to ``borrow'' an $\epsilon$-share of the reward from the Inventory Balancing algorithm, with both algorithms then re-optimized for the worst case under this new accounting scheme.  It leads to the notion an ``$\epsilon$-perturbed potential function $\Psi$'', which extends the typical inventory balancing function from online matching by placing a \textit{steeper penalty} on almost-depleted resources when $\epsilon>0$.  On the other hand, this new accounting also leads to the notion of ``$(1+\epsilon)$-relaxed regret'' in our auxiliary multi-armed bandits problem, and we propose a new ``LazyUCB'' oracle for minimizing it, which ends up performing \textit{less exploration and more exploitation} than traditional UCB oracles when $\epsilon>0$.

Both of these changes brought by $\epsilon>0$ are intuitive, in our problem setting with both adversarial contexts and unknown probabilities.
On one hand, $\Psi$ has less reason to assign almost-depleted resources, because the unknown probabilities for an almost-depleted resources are less worth learning.
On the other hand, LazyUCB has less reason to explore, because the adversarial contexts mean there is no guarantee that an arm can be legally pulled again in the future.

For the online matching application, we show that by using IBOL with our LazyUCB oracle optimized for $(1+\epsilon)$-relaxed regret, we can obtain a guarantee of \begin{equation}\label{eq:intro_lazy}
\bE[\ALG] \geq  \left(1-\frac{1}{e}\right)\OPT - O(\epsilon\cdot\OPT) - \min\left\{\tilde{O}(\sqrt{\OPT}), \tilde{O}\left(\frac{1}{\epsilon}\right)\right\},
\end{equation}
which captures~\eqref{eqn:introEqn} as a special case when $\epsilon=0$.  Although this does not improve the asymptotic guarantee in the worst case, given an estimate of $\OPT$, parameter $\epsilon$ can be tuned to maximize the bound in~\eqref{eq:intro_lazy} based on the particular constants suppressed by the big-O notation.

\textbf{LazyUCB: the empirical benefit of UCB with less exploration.}
We show in numerical simulations that our LazyUCB oracle empirically outperforms traditional UCB; meanwhile,~\eqref{eq:intro_lazy} shows that it has a worst-case guarantee parameterized by $\epsilon$ that is identical to~\eqref{eqn:introEqn} when $\epsilon=\Theta(1/\sqrt{\OPT})$.  This echoes the results in a recent line of work \citep{BastaniBK17,KannanMRWW18}, who show that (mostly) exploration-free algorithms improve empirical performance while maintaining an asymptotically-optimal theoretical guarantee, for bandits under stochastic contexts.  In contrast to these works, we allow for adversarial contexts, and the driving force behind our result is the inventory constraints.


\textbf{Further simulations on hotel data set.}
To demonstrate the flexibility of our framework, we also apply it to a dynamic assortment optimization problem in which each resource can be sold at different reward rates.  We use the same setup as \citet{MSL17}, except now the choice probabilities must be learned, and we test on the same hotel data set \citep{BFG09}.

\subsection{Roadmap}

In Section~\ref{sec:model} we present our general online resource allocation model as well as specific Applications 1 and 2.  In Section~\ref{sec:alg} we define our general IBOL (Inventory Balancing with Online Learning) algorithmic framework, including the parameter $\epsilon\in[0,1]$.  In Section~\ref{sec:analysis} we provide a general performance guarantee for IBOL which depends on $\epsilon$.  In Section~\ref{sec:MAB_oracle} we derive MAB oracles for the specific Applications 1 (Sections~\ref{sec:UCB_1}--\ref{sec:lazyUCB_1}) and 2 (Section~\ref{sec:UCB_lazyUCB_2}), including a proof that these oracles lead to a tight overall performance guarantee for IBOL (Section~\ref{sec:lb}).  In Section~\ref{sec:numerical} we present experimental results on synthetic instances of Application 1 (Section~\ref{sec:numerical1}) and on a real-world hotel data set (Section~\ref{sec:numerical2}).

\subsection{Literature Review}
\label{sec:review}

We summarize the positioning of our paper in Table~\ref{tab:positioning}.  Our analysis incorporates the loss from two unknown aspects: the adversarial sequence of customer contexts, and the probabilistic decision for a given customer context.  When one or both of these aspects are known, many papers have analyzed the corresponding metrics of interest (competitive ratio, regret, approximation ratio).
To our understanding, we are the first to give a unified analysis for online algorithms involving (i) resource constraints, (ii) learning customer behavior, and (iii) adversarial customer arrivals.  We now review past work which has considered some subset of these aspects, as outlined in Table~\ref{tab:positioning}.

\begin{table}
\caption{Breakdown of the literature on resource-constrained online allocation.  See Section~\ref{sec:review} for a review.}
\label{tab:positioning}
\centering
\begin{tabular}{cc|c|c|}
\cline{3-4}
\up & & \multicolumn{2}{c|}{Sequence of customer contexts} \\
\down & & \multicolumn{2}{c|}{$x^1,x^2,\ldots$} \\
\cline{3-4}
\up & & \textit{(Distributionally)} & \textit{Unknown Adversarial} \\
\down & & \textit{Known} & \textit{(must hedge)} \\
\hline
\up & \multicolumn{1}{|c|}{\textit{(Distributionally)}} & \multirow{2}{*}{Approximation Algorithms} & \multirow{2}{*}{Competitive Analysis} \\
\down Decisions of customer & \multicolumn{1}{|c|}{\textit{Known}} & & \\
\cline{2-4}
\up with context $x$ & \multicolumn{1}{|c|}{\textit{Unknown i.i.d.}} & \multirow{2}{*}{Online Learning} & \multirow{2}{*}{\textbf{[this paper]}} \\
\down & \multicolumn{1}{|c|}{\textit{(can learn)}} & & \\
\hline
\end{tabular}
\end{table}

\subsubsection{Approximation algorithms.}

When both the arrival sequence and customer decisions are distributionally known, many algorithms have been proposed for overcoming the ``curse of dimensionality'' in solving the corresponding dynamic programming problem.
Performance guarantees of bid-pricing algorithms were initially analyzed in \citet{TvR98}. Later, \citet{AHL12} and \citet{WTB15} proposed new algorithms with improved bounds, for models with time-varying customer arrival probabilities.
These performance guarantees are relative to a deterministic LP relaxation (see Section~\ref{sec:analysis}) instead of the optimal dynamic programming solution, and hence still represent a form of ``competitive ratio'' relative to a clairvoyant which knows the arrival sequence in advance (see \citet{WTB15}).

In addition, the special case in which customer arrival probabilities are time-invariant has been studied in \citet{FMMM09} and its subsequent research. We refer to \citet{BrubachSSX16} for discussions of recent research in this direction.

\subsubsection{Competitive analysis.}

We briefly review the literature analyzing the competitive ratio for resource allocation problems under adversarial arrivals.  This technique is often called \textit{competitive analysis}, and for a more extensive background, we refer the reader to \cite{BEY05}.  For more on the application of competitive analysis in online matching and allocation problems, we refer to \cite{Meh13}.  For more on the application of competitive analysis in airline revenue management problems, we refer to the discussions in \cite{BQ09}.

Our work is focused on the case where competitive analysis is used to manage the consumption of resources.  The prototypical problem in this domain is the Adwords problem \citep{MSVV07}.  Often, the resources are considered to have large starting capacities---this assumption is equivalently called the ``small bids assumption'' \citep{MSVV07}, ``large inventory assumption'' \citep{GNR14}, or ``fractional matching assumption'' \citep{KP00}.  In our work, we use the best-known bound that is parametrized by the starting inventory amounts \citep{MSL17}.  The Adwords problem originated from the classical online matching problem \citep{KVV90}---see \cite{DJK13} for a recent unified analysis.  The competitive ratio aspect of our analysis uses ideas from this analysis as well as the primal-dual analysis of Adwords \citep{BJN07}. We also refer to \cite{DJ12,KP16,MSL17} for recent generalizations of the Adwords problem.

Our model also allows for probabilistic resource consumption, resembling many recent papers in the area starting with \cite{MP12}.  We incorporate the \textit{assortment} framework of \cite{GNR14}, where the probabilistic consumption comes in the form of a random customer choice---see also \cite{CMSLX16,MSL17}.  However, unlike these papers on assortment planning, our model does not require the substitutability assumption on the choice model, since we allow resources which have ran out to still be consumed for zero reward.

\subsubsection{Online learning.}
The problem of learning customer behavior is conventionally studied in the field of \emph{online learning}. For a comprehensive review on recent advances in online learning, we refer the reader to \citet{Bubeck2012,Slivkins2017}.


Our research focuses on online learning problems with resources constraints. \cite{Badanidiyuru2014,Agrawal2014} incorporate resource constraints into the standard multi-armed bandit problem, and propose allocation algorithms with provable upper bounds on the regret. \cite{BKS13,Agrawal2015a,Agrawal2015b} study extensions in which customers are associated with independently and identically distributed context vectors; the values of reward and resource consumption are determined by the customer context. \cite{Besbes2009,Besbes2012,Babaioff2011,Wang2011,FSLW16} study pricing strategies for revenue management problems, where a resource-constrained seller offers a price from a potential infinite price set to each arriving customer. Customers are homogeneous, in the sense that each customer has the same purchase probability under the same offered price.

Those models with resource constraints in the current literature assume that the type (if there is any) of each customer is drawn from a fixed distribution that does not change over time. As a result, there exists an underlying fixed randomized allocation strategy (typically based on an optimal linear programming solution) that converges to optimality as the number of customers becomes large. The idea of the online learning techniques involved in the above-mentioned research works is to try to converge to that fixed allocation strategy. In our model, however, there is no such fixed allocation strategy that we can discover over time. For instance, the optimal algorithm in our model may reject all the low-fare customers who arrive first and reserve all the resources for high-fare customers who arrive at the end. As a result, the optimal algorithm does not earn any reward at first, and thus cannot be identified as the best strategy by any learning technique. Our analysis is innovative as we construct learning algorithms with strong performance guarantees without trying to converge to any benchmark allocation strategy.

Finally, the LazyUCB oracle proposed in the paper is related to, and inspired by, a recent body of research \citep{BastaniBK17,KannanMRWW18} on (mostly) exploration-free approaches for the stochastic contextual multi-armed bandit problem. These works highlight the observation that, in stochastic contextual bandit settings, exploration free algorithms often \emph{empirically} out-perform traditional algorithms such as Upper-Confidence Bound (UCB) and Thompson Sampling (TS). These research works propose theoretical justifications by establishing regret bounds based on certain regularity assumptions on the contextual vectors and the latent parameters. While the theoretical guarantees for the (mostly) exploration-free approaches established in \citet{BastaniBK17,KannanMRWW18} are no better than the best-known theoretical guarantee for the stochastic contextual bandit problem, the authors demonstrate that their proposed algorithms are consistently superior to traditional algorithms in terms of the empirical performance.

Similar to these works, our proposed LazyUCB oracle reduces the amount of exploration in existing UCB algorithms. However, our work differ from \citet{BastaniBK17,KannanMRWW18} in three ways. First, our LazyUCB oracle still includes an exploration bonus in its computation of upper-confidence intervals, while \citet{BastaniBK17,KannanMRWW18} require full exploitation and no exploration. Second, we allow the contextual information of different customers to vary arbitrarily and adversarially without any assumption on how the contextual information varies among customers. By contrast, \citet{BastaniBK17,KannanMRWW18} require the contextual information of different customers to be drawn i.i.d. from a latent probability distribution, satisfying certain regularity assumptions, in order for the theoretical guarantees to hold. Third, we consider an inventory-constrained setting, while \citet{BastaniBK17,KannanMRWW18} consider settings without any constraint on the choices of arms.

\section{Model Formulation}
\label{sec:model}

Throughout this paper, we let $\bN$ denote the set of positive integers. For any $n \in \bN$, let $[n]$ denote the set $\{1,2, \ldots ,n\}$.

We consider the following class of online resource allocation problems. An online platform has a collection of $n\in \bN$ resources, denoted $[n]$, to be allocated to $T\in \bN$ customers who arrive sequentially. For each $i\in [n]$, the platform has $b_i\in \bN$ units of resource $i$, that are not replenishable during the allocation period. Each unit of resource $i$ is associated with reward $r_i$ normalized to lie in $(0, 1]$. In Sections~\ref{sec:numerical2} and~\ref{sec:multipleRates}, we consider a generalized setting where each resource is associated with multiple reward values as in \citet{MSL17}.

We now define the notation regarding an allocation to a customer.
Each customer is associated with a context $x$, and a context carries personal information about the customer. We denote ${\cal X}$ as the set of all possible contexts. The set  ${\cal X}$ is finite and is known to the online platform. The variation among contexts models the heterogeneity among the customers, and the context sequence is generated adversarially. There is an action set ${\cal A}$, which represents the set of allocations decisions. Each pair of $x\in {\cal X}$ and $a\in {\cal A}$ is associated with an outcome distribution $\rho_{x ,a}$, which is a probability distribution over $\{0, 1\}^{n}$. For each $\vy\in \{0, 1\}^n$, we let $\rho_{x, a}(\vy)$ denote the probability that the outcome if $\vy$.

\textbf{Dynamics.} The platform interacts with the customers in $T$ discrete time steps. For each $t\in \{0, 1, \ldots, T\}$ and each $i\in [N]$, we denote $N^t_i$ as the number of units of resource $i$ that have been consumed by the \textit{end} of time $t$. In particular, we have $N^0_i = 0$ for all $i$.
At time step $t\in [T]$, four events happen. First, customer $t$ arrives, and their context $x^t$ is revealed to the platform. Second, the platform selects an action $a^t\in \cA$, based on $x^t$ and the observations in time steps $1, \ldots, t-1$. Third, the platform observes the vectorial outcome ${\vy}^t = (\vy^t_i)_{i\in [n]} \in \{0, 1\}^n$, which is distributed according to the distribution $\rho_{x^t, a^t}$.\footnote{The outcome ${\vy}^t$ is described more precisely as follows. Before the online process, for each $x\in {\cal X}, a\in {\cal A}$ the nature generates i.i.d. samples $\{\vy^s_{x, a}\}_{s =1}^\infty$. At time $t$, when the context is $x^t$ and action $a^t$ is chosen, the nature reveals $\vy^{n^t}_{x^t, a^t}$ as the outcome, where $n^t = \sum^{t}_{s=1}\mathbf{1}(x^s = x, a^s = a)$ is the number of occurrences of $(x,a)$ from time 1 to $t$.} Fourth, if $\vy^t_i =1$ and resource $i$ is not yet depleted ($N^{t-1}_i < b_i$), then one unit of the inventory of resource $i$ is consumed ($N^t_i = N^{t-1}_i + 1$), and a reward of $r_i$ is earned. If $\vy^t_i = 1$ but resource $i$ is depleted, or if $\vy^t_i = 0$, then no resource $i$ is consumed ($N^t_i = N^{t-1}_i$) and no reward is earned.

It is worth noting that the feedback $\vy^t$ at each time $t$ is a partial feedback, which is more precisely known as \emph{bandit feedback} in the online learning literature. The feedback is partial in the sense that the platform only observes $\vy^t$ under the action $a^t$, but it does not observe the feedback under any other actions.

\textbf{Model Uncertainty.} The online allocation problem involves model uncertainty in two dimensions. First, the sequence of contexts $\{x^t\}^T_{t=1}$ is generated by an oblivious adversary, who cannot see any information related to $\{a^t\}^T_{t=1}, \{\vy^t\}^T_{t=1}$.
In particular, the contexts $x^1, x^2, \ldots, x^T$ do not generally come from any fixed distribution. Instead, they could vary arbitrarily. The adversarial uncertainty models the volatile and the unpredictable nature of customer arrivals in e-service operations settings. 

Second, for each $x\in {\cal X}, a\in {\cal A}$, the probability distribution $\rho_{x, a}$ is not known to the platform. Rather, the platform has to learn the distribution for each $x, a$ during the online process. It is of interest to learn the latent parameter $$p_{x, a, i} = \mathbb{E}_{\vy_{x, a}\sim \rho_{x, a}}[\vy_{x, a, i}] = \sum_{\vy\in \{0, 1\}^n}\rho_{x, a}(\vy)\vy_i,$$ which is the probability that the outcome for resource $i$ is 1, when the context is $x$ and the action is $a$. The Bernoulli random variables $\vy_{x, a, 1}, \vy_{x, a, 2}, \ldots, \vy_{x, a, n}$ can be correlated in general.

Altogether, our online resource allocation model requires the platform to \emph{hedge} against the adversarial uncertainty of customers' contexts $\{x^t\}^T_{t=1}$, while simultaneously balancing the \emph{explore vs.\ exploit} tradeoff on the uncertainty in the stochastic model  $\{\rho_{x, a}\}_{x\in {\cal X}, a\in {\cal A}}$.
The main thesis of this work is about how the platform manages the three-way trade-off among hedging, exploration, and exploitation.

\textbf{Objective.} The platform's objective is to maximize the total expected revenue. Mathematically, the platform maximizes $$\mathbb{E}\left[ \sum_{i\in [n]} r_i  \sum_{t\in [T]} \vy^t_i \cdot \mathbf{1}(N^{t-1}_i < b_i)\right] = \mathbb{E}\left[ \sum_{i\in [n]} r_i \min\left\{b_i , \sum_{t\in [T]} \vy^t_i \right\}\right].$$
The platform is subject to the inventory constraints that at most  $b_i$ units of resource $i$ are consumed for each $i\in [n]$. The expectation is taken over the randomness in the actions $a^1, \ldots, a^T$ and the stochastic outcomes $\vy^1, \ldots, \vy^T$.

Finally, we relate our online resource allocation model to the existing literature. If the probability distributions $\{\rho_{x, a}\}_{x\in \cX, a\in \cA}$ are known to the platform, then we essentially recover the setting in \cite{GNR14}.\footnote{The model in \cite{GNR14} uses the language of assortment optimization, but it is not hard to abstract it to match our resource allocation setting.}  If in addition the outcome $\vy$ is deterministic given $x, a$ (that is, $\rho_{x, a}$ is the distribution for a deterministic random variable for each $x, a$), we recover the Adwords problem of \citet{MSVV07}.


\textbf{Applications. }
Our problem represents a generic resource allocation model with general context set ${\cal X}$ and action set ${\cal A}$. For a concrete discussion, we consider the following two specializations of $[n], [T], {\cal X}, {\cal A}, \rho$, which capture important applications in e-service operations. We elaborate on these applications in Sections \ref{sec:app1}, \ref{sec:app2}, and summarize these applications in Table \ref{tab:app}.
\subsection{Application 1: Internet advertising / Crowd-sourcing} \label{sec:app1}
This model is based on the Online Matching with Stochastic Rewards problem of \citet{MehtaP12}, except there could be $K>1$ probabilities associated with each offline vertex, and these probabilities must be learned.

\textbf{Internet advertising.}
The first application concerns the dynamic allocation of internet advertisement from advertisers to web surfers, with the objective of maximizing the total pay-per-click \citep{MehtaP12,Mehta13,GoyalU19}. The advertisers are modeled as the resources  $[n]$, and the web-surfers are modeled as the customers $[T]$, who arrive sequentially at the platform during a certain planning horizon, say during a day. Each advertiser $i\in [n]$ is willing to spend at most $b_i\cdot r_i$ dollars for receiving clicks on their advertisements.

The context set is $\cX=\{0,1\}^n$. 
A customer with context $x$ only clicks on an advertisements from advertisers in $\{i : x_i = 1\}$. Therefore, it is sensible to match a customer with context $x$ with advertiser $i$ only if $x_i = 1$. For example, with a customer who is known to have recently purchased an android phone, it is sensible for the platform to allocate an advertisement on complementary products such as phone accessory, but not an advertisement on another android phone.

Next, we describe the action set ${\cal A}$. Each advertiser has $K\in\bN$ different advertisements, e.g., $K$ videos/banners.
The action set is ${\cal A} =\{(i, k) : i\in [n], k\in [K]\}$. When the action $(i, k)$ is taken, it means that the platform allocates the $k$'th advertisement of advertiser $i$ to the customer. The resulting click probability is $\mathbf{1}(x_i = 1)q_{(i, k)}$. The quantity $q_{(i, k)}$ is latent, whereas the quantity $\mathbf{1}(x_i = 1)$ is not since the context is revealed before an action is chosen. Collectively, the probability model $\{\rho_{x, (i, k)}\}_{x\in \{0, 1\}^n, i\in [n], k\in [K]}$ is defined as
\begin{align}
\rho_{x,(i,k)}(\ve_i) &= \mathbf{1}(x_i=1)\cdot q_{(i,k)}, \nonumber\\
\rho_{x,(i,k)}(\vzero) &=1-\rho_{x,(i,k)}(\ve_i), \label{eq:app1_prob}\\
\rho_{x,(i,k)}(\vy) &=0\text{ for all other outcomes $\vy$ in $\{0,1\}^n$}\nonumber.
\end{align}
There are $Kn$ many latent terms $\{q_{(i, k)}\}_{i\in [n], k\in [K]}$ to be learned.

The platform earns a revenue of $r_i$ when an advertisement from advertiser $i$ is clicked, and the platform earns nothing if there is no click. The objective is to maximize the total expected revenue based on the pay-per-click, subject to the budget constraints of the advertisers.
The adversarial uncertainty on $\{x^t\}^T_{t=1}$ reflects that web-surfers arrivals are highly volatile, and they are influenced by so many different factors that they are hard to be precisely forecast. The model uncertainty on $\{\rho_{x, a}\}_{x\in {\cal X}, a\in {\cal A}}$ reflects that customers' tastes have to be learned during the planning horizon.

\textbf{Crowd-sourcing. }The same mathematical model on $[n], [T], {\cal X}, {\cal A}, \rho$ captures a class of crowd-sourcing problems \citep{HoV12,KargerOS14}.\footnote{Nevertheless, we still refer to the above model as the online advertisement allocation problem.} We interpret $[n]$ as a collection of task owners, who pose their tasks on an online crowd-sourcing platform, for example Amazon Mechanical Turk. Each task owner $i\in [n]$ has $K$ tasks to be completed.  The workers, represented as $[T]$, arrive at the online platform sequentially. 
A worker with context $x \in {\cal X} = \{0, 1\}^n$ is only capable of accomplishing tasks from task owners in $\{ i : x_i = 1\}$.

When the action $(i, k)\in {\cal A} = \{(i, k)\}_{i\in [n], k\in [K]}$ is taken with a customer of context $x$, it means that task $k$ from task owner $i$ is assigned to the customer. The outcome $\vy_{x, (i, k)}$ is either the task owner $i$ has their task accomplished ($\vy_{x, (i, k)} = \ve_i$) or no task is accomplished ($\vy_{x, (i, k)} = \vzero$), according to the probability distribution $\rho$ in equations (\ref{eq:app1_prob}). If a worker successfully accomplishes a task from owner $i$, the worker earns a reward $r_i$, otherwise the worker does not earn any reward. The crowd-sourcing platform acts as a welfare maximizer, who aims to maximize the total amount of revenue earns by the workers in order to encourage participation into the platform.

\subsection{Application 2: Personalized Product Recommendation with Customer Segmentation.}\label{sec:app2}
Our model also applies when an inventory constrained seller conducts sales to a pool of heterogeneous customers. In contrast to Application 1, here the probabilities for successfully allocating the resources depend on the customer segment.
These can be learned over time as there are repeated customers from the same segments. The seller has $n$ types of products, denoted $[n]$. They have $b_i$ units of product $i$ for each $i\in [n]$, which are not replenishable during the planning horizon. There are $T$ customers, collectively denoted as $[T]$, who arrives at the seller's platform sequentially. The customers are heterogeneous, and they are segmented in terms of the customers' characteristics, such as their gender, age and occupation.

The context set ${\cal X}$ denotes the set of all customer segments. The action set ${\cal A} = [n]$ corresponds to the set of products. 
Based on the observed customer segment $x$, the seller recommends to the customer a product $i$, which corresponds to action $i$. 
The outcome $\vy\in \{0, 1\}^n$ is equal to $\ve_i$ (the customer buys the recommended product) with probability $p_{x, i}$, and is equal to $\vzero$ (the customer does not buy) with the complementary probability $1 - p_{x, i}$ . Altogether, $\rho_{x, i}( \ve_i) = p_{x, i} = 1 - \rho_{x, i}(\vzero)$. The resulting expected revenue is $r_i p_{x, i}$. The platform's objective is to maximize the total expected revenue, subject to the inventory constraints and the model uncertainty on $\{x^t\}_{t\in [T]}$, $\{\rho_{x, i}\}_{x\in \cX, i\in [n]}$.



\begin{table}
\centering
\begin{tabular}{|c|c|c|c|}
\hline
Application & Internet Ad (\S~\ref{sec:app1}) & Crowd-sourcing (\S~\ref{sec:app1}) & Personalized OM (\S~\ref{sec:app2})\\
\hline
$[n]$ & Advertisers & Task owners & Products \\
\hline
$[T]$ & Web-surfers & Workers & Customers \\
\hline
${\cal X}$ & Compatibility & Compatibility & Customer segments \\
\hline
${\cal A}$ & Ad allocation & Task allocation & Product recommendation\\
\hline
$\rho$ & Click probability & Success probability & Purchase probability \\
\hline
Objective & Pay-per-click & Workers' reward & Platform's revenue \\
\hline
To estimate & $\{q_{(i, k)}\}_{i\in [n], k\in [K]}$ & $\{q_{(i, k)}\}_{i\in [n], k\in [K]}$ & $\{p_{x, i}\}_{x\in \cX, i\in [n]}$ \\
\hline
\end{tabular}
\caption{Applications of our online model}\label{tab:app}
\end{table}

\section{Online Allocation Algorithm: IBOL}
\label{sec:alg}
We present a framework in Algorithm \ref{alg:ibol}, called IBOL (``Inventory Balancing with Online Learning''), for solving our online resource allocation problem. The IBOL algorithm involves two inputs: a potential function $\Psi$, and a multi-armed bandit (MAB) oracle $\{{\cal O}^t\}^T_{t=1}$. They are respectively used to hedge against the adversarial uncertainty on $\{x^t\}^T_{t=1}$ and to learn the uncertain model on $\{\rho_{x, a}\}_{x \in \cX, a\in \cA}$. The IBOL algorithm also requires the input of the \emph{exploitation parameter} $\epsilon\in [0, 1]$, which affects both our potential function $\Psi$ and MAB oracle $\{\cal{O}^t\}_{t\geq 1}$. On a high level, when the exploitation parameter $\epsilon$ increases, the algorithm conducts more exploitation but less exploration on the latent model  $\{\rho_{x, a}\}_{x \in \cX, a\in \cA}$.
A precise description about the role of $\epsilon$ and how to set it is deferred to Section~\ref{sec:analysis}, where we discuss the performance guarantee for IBOL.

\begin{algorithm}[t]
\caption{Inventory-Balancing with Online Learning (IBOL)}\label{alg:ibol}
\begin{algorithmic}[1]
\State Inputs: Inventory levels $\{b_i\}_{i\in [n]}$, exploitation parameter $\epsilon\in [0,1]$, $\epsilon$-perturbed potential function $\Psi$, MAB oracle $\{\cal{O}^t\}_{t\geq 1}$ (see Section \ref{sec:MAB_oracle} for concrete examples of $\{\cal{O}^t_\epsilon\}_{t\geq 1}$ on Applications 1, 2).
\State Initialize: $N^0_i = 0$ for all $i\in [n]$.
\For{$t = 1, \ldots, T$}
\State Observe the feature vector $x^t$ of the customer at time $t$.
\State Compute the discounted rewards $r^t_i$ for all $i$, where
\begin{equation}\label{eq:discountedReward}
r^t_i = r_i \left( 1 - \Psi\left(\frac{N_i^{t-1}}{b_i}\right)\right),
\end{equation}
with $\Psi$ being our $\epsilon$-perturbed potential function in~\eqref{eqn:epsPerturbed}.
\State Compute the action using the online learning algorithm ${\cal O}^t$:
\begin{equation}
\label{eq:ol}
a^t = {\cal O}^t({\cal F}^{t-1}, x^t, (r^t_i)_{i\in[n]} ; U).
\end{equation}
\State Observe the feedback $\vy^t_{x^t, a^t} = (\vy^t_{x^t, a^t, i})_{i\in [n]}$.
\State For each $i\in [n]$, if $\vy^t_{x^t, a^t, i} = 1$ and $N^{t-1}_i < b_i$, the platform earns $r_i$, and depletes 1 unit of resource $i$: $N^t_i = N^{t-1}_i + 1$. Otherwise, the platform earns nothing, and $N^t_i = N^{t-1}_i$.
\State Update the history ${\cal F}^{t+ 1} = {\cal F}^t \cup (x^t, a^t, \vy^t)$.
\EndFor
\end{algorithmic}
\end{algorithm}

\textbf{Definition of $\epsilon$-perturbed Potential Function $\Psi$.}
A potential function is a standard tool used in online resource allocation to generate discounted rewards which guide an algorithm's optimization.
More specifically, $\Psi : [0, 1] \rightarrow [0, 1]$ is a non-decreasing function satisfying $\Psi(0) = 0, \Psi(1) = 1$.
The rewards of resources are discounted by a factor of $1-\Psi(N^{t-1}_i / b_i)$, to penalize the allocation of resources which have been overutilized relative to their starting amounts (recall that $N^{t-1}_i$ denotes the units of resource $i$ allocated by the start of a time $t$), in anticipation of the adversarial contexts $x^1,x^2,\ldots$.
Note that if $N^{t-1}_i = b_i$, i.e.\ resource $i$ is depleted, then the discounted reward is 0, which is consistent with the assumption that no reward is earned from depleted resources.

For $\epsilon\in [0,1]$, we will be using a new ``$\epsilon$-perturbed'' potential function
\begin{align} \label{eqn:epsPerturbed}
\Psi(x) = \frac{e^{(1+\epsilon)x} - 1}{e^{1+\epsilon}-1},
\end{align}
which recovers the classical potential function of $\Psi(x) = \frac{e^x - 1}{e-1}$ from \citet{MSVV07} when $\epsilon=0$.
Our $\epsilon$-perturbed potential function is designed to maximize the competitive ratio of the online algorithm when its reward has been reduced by $\epsilon$.
For $\epsilon>0$, our $\epsilon$-perturbed Potential Function is steeper than the classical potential function, i.e.\ it places a relatively greater penalty on almost-depleted resources.
This is intuitive, because in our generalized problem where probabilities must be learned, there is relatively less reason to select almost-depleted resources, since there is less benefit to learning the probabilities for such resources.

Given a potential function $\Psi$,
a standard approach \citep[see e.g.][]{GNR14} for an online algorithm is to play at each period $t$ an action $a^t_*\in \cA$ which maximizes the discounted reward
\begin{align*}
R^t(a) &= \sum_{i\in [n]} r_i \left[ 1 - \Psi\left(\frac{N_i^{t-1}}{b_i}\right)\right] \sum_{\vy\in \{0, 1\}^{[n]}}  \rho_{x^t, a}(\vy)\vy_i
\\&= \sum_{i\in [n]}r_i\left[ 1 - \Psi\left(\frac{N_i^{t-1}}{b_i}\right)\right] \cdot p_{x^t, a, i}.
\end{align*}
However, in our generalized problem the probabilities $p_{x, a, i}$ are unknown, making the existing approach unapplicable. Instead, the platform needs to simultaneously: explore to learn the probabilities $p_{x, a, i}$, while maximizing $R^t(a)$, and also hedging against the adversarial uncertainty on $\{x^t\}^T_{t=1}$. We now formulate an \textit{auxiliary problem}, new to our work, which captures this three-way trade-off between hedging, exploration, and exploitation.

\textbf{Definition of Auxiliary Problem.} Our auxiliary problem is a contextual stochastic bandit problem with the same context set $\cX$, action set $\cA$, and distributions $\{\rho_{x, a}\}_{x\in \cX, a\in \cA}$ as the online resource allocation problem. The distributions $\{\rho_{x, a}\}_{x\in \cX, a\in \cA}$ are still unknown from the beginning and have to be learned. An MAB oracle $\{{\cal O}^t\}^T_{t=1}$ is a learning algorithm designed for solving the auxiliary problem, where ${\cal O}^t$ is used for the decision at time $t$.
In our auxiliary problem, we think about four events happening at each time $t$. First, customer $t$ arrives, and the platform is provided with the context $x^t$ and the discounted reward defined as $$r^t_i=r_i \Big( 1 - \Psi(\frac{N_i^{t-1}}{b_i})\Big).$$
Second, the platform chooses an action $a^t\in \cA$ using the oracle function ${\cal O}^t$. 
Mathematically, it is expressed as $a^t = {\cal O}^t({\cal F}^{t-1}, x^t, r^t)$, where ${\cal F}^{t-1} = \{(x^s, a^s, \vy^s)\}^{t-1}_{s=1}$ consists of the observations in time steps $1, \ldots, t-1$. Third, the platform observes the vectorial outcome $\vy^t$ that is distributed according to $\rho_{x^t, a^t}$. Fourth, the platform receives the reward $ \sum_{i\in [n]} r^t_i \cdot \vy^t_i$. Overall, the platform aims to design an MAB oracle that maximizes the total expected reward
$$
\mathbf{E}\left[\sum_{t\in [T]}\sum_{i\in [n]} r^t_i \cdot \vy^t_i \right] =\mathbf{E}\left[\sum_{t\in [T]}\mathbf{E}\left[\sum_{i\in [n]} r^t_i \cdot \vy^t_i \mid x^t, a^t, r^t_i\right]\right] = \mathbf{E}\left[\sum^T_{t=1} R^t(a^t)\right].
$$

\textbf{Justification for Auxiliary Problem.}
Our auxiliary problem can be seen as a way of ``abstracting'' the inventory constraints away from a typical contextual bandit problem, by introducing the discounted rewards $r^t_i$.
In contrast to traditional ``Bandits with Knapsacks'' approaches (reviewed in Section~\ref{sec:review}), which do not have adversarial contexts, we define these rewards based on the potential function and the current resource consumption at time $t$, to hedge against the adversarial contexts.
This leads to a contextual bandit problem with \textit{non-stationary} rewards,
in which the optimal action $a^t_*$ changes across time not only because of changes in the context $x^t$, but also because of the non-stationarity of $r^t_i$ over time.
Moreover, this change in $r^t_i$ is adaptively influenced by the platform's decisions in time $1, \ldots, t-1$, and is difficult to control.
Therefore, in our abstracted contextual bandit problem, we simply allow the values of $r^t_i$
to be generated by an adaptive adversary, who can decide them based on historical information.
A key part of our analysis is then to show that this is the ``correct'' learning problem to focus on, where inventory is unconstrained but regret is measured with respect to these non-stationary rewards $r^t_i$, instead of the the original problem where inventory was constrained but the rewards $r_i$ were fixed.


A priori, it might appear that, for the auxiliary problem, a learning algorithm needs to deviate from the traditional stochastic MAB framework and to adapt to the reward non-stationarity, in the same vein as the existing literature on non-stationary stochastic bandits \citep[e.g.][]{GarivierM11,GurZB14}. Nevertheless, in subsequent Sections, we demonstrate that it is still possible to adapt the existing stationary stochastic MAB tools, despite the auxiliary problem's non-stationarity, by decoupling the non-stationarity of $r^t_i$ from the learning problem on $\rho$, in the contexts of Applications 1, 2.  This is important for our solution of the auxiliary problem.

\textbf{$(1+\epsilon)$-Relaxed Regret in the Auxiliary Problem.}
The rewards for our auxiliary problem were generated by a ``perturbed'' potential function $\Psi$ which aimed to maximize the competitive ratio of an online algorithm whose reward has been reduced by a factor of $\epsilon$.
To compensate, in the auxiliary problem, the algorithm's reward is boosted by a factor of $\epsilon$ and we aim to minimize the notion of \emph{$(1+\epsilon)$-relaxed regret:}
\begin{equation}\label{eq:ep_reg}
\text{Reg}_\epsilon = \bE \left[\sum^T_{t=1}R^t(a^t_*) - (1 +\epsilon ) R^t(a^t)\right].
\end{equation}
Recall that $a^t_*$ denotes the action that maximizes the discounted reward $R^t(a)$ among all actions $a$.
When we set $\epsilon = 0$ in the definition of $(1+\epsilon)$-relaxed regret in (\ref{eq:ep_reg}), we recover the classical notion of regret, $\text{Reg}_0$.  We elaborate on how $\epsilon$ affects the performance of IBOL in Section~\ref{sec:analysis}.

The value of $\epsilon$ affects our choice of MAB oracle, because the goal of the MAB oracle is to make $\text{Reg}_\epsilon $ in (\ref{eq:ep_reg}) as small as possible in the worst case.
A salient difference in our case when $\epsilon>0$ is that a sublinear $\text{Reg}_\epsilon$ can be attained by identifying an action that $\epsilon$-optimal, instead of needing to eventually learn what the exact optimal action is.
For a given input $\epsilon\in[0, 1]$, we design MAB oracles that are variants of UCB optimized for $\text{Reg}_\epsilon $, in the context of Applications 1, 2.
These variants coincide with a traditional UCB algorithm when $\epsilon = 0$. On the other hand, when $\epsilon \in (0, 1]$, these variants end up being more greedy than a traditional UCB algorithm, in the sense that they conduct less exploration but more exploitation.  Thus, we refer to our variants using ``LazyUCB''.

The design of learning algorithms with less exploration than traditional approaches is similar in spirit to the recent papers by \citet{BastaniBK17,KannanMRWW18}, who show that less exploration leads to empirically better algorithms. As shown in our numerical experiments in Section \ref{sec:numerical}, our proposed greedy variants also achieve better empirical performances than the traditional approaches, in our setting where there are inventory constraints under unknown contexts.
Altogether, the motivation for perturbing both our potential function and notion of regret by $\epsilon$ is that the overall performance can be improved, when the MAB algorithm is essentially ``borrowing'' an $\epsilon$-share of the reward from the potential function, and then both of these are re-optimized.

\section{Analysis of the IBOL Algorithm}
\label{sec:analysis}
In this section, we bound the performance of our IBOL algorithm, with parameter $\epsilon$, in terms of the $(1+\epsilon)$-relaxed regret incurred by its underlying MAB oracle.
In the next section we develop MAB oracles which specifically minimize $(1+\epsilon)$-relaxed regret for Applications 1, 2.

\textbf{LP Upper Bound.}
We compare the performance of the IBOL algorithm to a benchmark defined by a linear program (LP) called \textbf{Primal}. Our benchmark is the optimal value $\OPT$ of LP \textbf{Primal}, which upper bounds the total expected reward of any algorithm that knows both $\{\rho_{x, a}\}_{x\in \cX, a\in \cA}$ and $\{x^t\}^T_{t=1}$ before the process begins. The formulation of this linear program benchmark for resource allocation is standard in the revenue management literature, and we formulate it below:
\begin{align}
\textbf{Primal: } \max    & \sum_{t\in [T]} \sum_{a\in \cA} s_{a, t} \cdot \left[\sum_{i\in [n]} r_i p_{x^t, a, i} \right] \label{eq:primal_obj}\\
\text{s.t. }   & \sum_{t\in [T]} \sum_{a\in \cA} s_{a, t} \cdot p_{x^t, a, i} \leq b_i    &\quad &\forall i \in [n] \label{eq:primal_1}\\
& \sum_{a\in \cA} s_{a, t} = 1 &\quad &\forall t\in [T] \label{eq:primal_2}\\
&s_{a, t}\geq 0      &\quad &\forall a\in  \cA, t\in [T]. \label{eq:primal_3}
\end{align}
The LP \textbf{Primal} serves as a fluid relaxation of the constrained online problem. As modeled by the constraints (\ref{eq:primal_2}, \ref{eq:primal_3}), the variable $s_{a,t}$ represents the unconditional probability of an algorithm taking action $a$ in period $t$. Consequently, the objective (\ref{eq:primal_obj}) of \textbf{Primal} is to maximize the total expected revenue. The set of constraints (\ref{eq:primal_1}) only requires the resource constraints to be satisfied in expectation, which is a relaxation to the online problem.
\begin{lemma}\label{lemma:benchmark}
For any online algorithm that satisfies the resource constraints $\sum^T_{t=1}\vy^t_i \leq b_i$ for all $i\in [n]$ with certainty, its total expected reward is at most $\OPT$.
\end{lemma}
For completeness we provide a proof of Lemma~\ref{lemma:benchmark} in Appendix~\ref{app:pf_lemma_benchmark}.

\textbf{Performance Guarantee.}
Equipped with Lemma~\ref{lemma:benchmark}, we are now ready to compare the total expected reward $\bE[\ALG]$ collected by the IBOL algorithm to the benchmark $\OPT$.  Theorem~\ref{thm:main} below maintains the generality of the algorithmic framework in Section \ref{sec:alg}, in that the performance guarantee holds for the general online resource allocation problem, not just Applications 1, 2, and the performance guarantee also holds for any MAB oracle.

\begin{theorem}\label{thm:main}
For any $\epsilon\in [0,1]$, the total reward $\ALG$ earned by the IBOL algorithm, using our $\epsilon$-perturbed potential function $\Psi(x) = \frac{e^{(1+\epsilon)x} - 1}{e^{1+\epsilon}-1}$, satisfies
\begin{align}
\label{eq:guarantee}
\bE[\ALG]
&\ge f_{\Psi}(\{b_i\}_{i\in [n]}, \epsilon) \cdot \left(\OPT-\bE\left[\sum_{t\in [T]} R^t(a^t_*) - (1+\epsilon) R^t(a^t)\right]\right),
\end{align}
where
\begin{equation}\label{eq:f_Psi}
f_{\Psi}(\{b_i\}_{i\in [n]}, \epsilon) = \min_{i\in [n]}\left\{ \frac{1-e^{-(1+\epsilon)}}{(b_i+1+\epsilon)(1-e^{-(1+\epsilon)/b_i})}\right\}.
\end{equation}
\end{theorem}
The proof of Theorem \ref{thm:main} is deferred to Appendix \ref{app:pf_thm_main}.
The expected reward $\bE[\ALG]$ of the algorithm is smaller than $\OPT$ in two ways: first, it is scaled down by the competitive ratio $f_{\Psi}(\{b_i\}_{i\in [n]}, \epsilon)$ which is less than 1; there is also an additive loss of the term $\bE\left[\sum_{t\in [T]} R^t(a^t_*) - (1+\epsilon) R^t(a^t)\right]$ which denotes the $(1+\epsilon)$-relaxed regret of the MAB oracle.
We note that choosing a larger $\epsilon$ in [0,1] for our IBOL algorithm will cause the competitive ratio $f_{\Psi}(\{b_i\}_{i\in [n]}, \epsilon)$ to decrease, but in return, the $(1+\epsilon)$-relaxed regret will be smaller.

\textbf{Justification for Form of Performance Guarantee.}
Our guarantee~\eqref{eq:guarantee} measures the regret in comparison to $\OPT$ after it has been scaled down by $f_{\Psi}(\{b_i\}_{i\in [n]}, \epsilon)$.
We now explain why a meaningful (i.e.\ sublinear) regret is impossible in our setting if we do not scale down $\OPT$.

The competitive ratio $f_{\Psi}(\{b_i\}_{i\in [n]}, \epsilon)$ is at its maximum for any vector $\{b_i\}_{i\in [n]}$ when $\epsilon=0$, in which case it can be re-expressed based on $b_\text{min} = \min_{i\in [n]}b_i$ as
\begin{align}
f_{\Psi}(\{b_i\}_{i\in [n]}, 0 )
= \min_{i\in [n]}\left\{ \frac{1-e^{-1}}{(b_i+1)(1-e^{-1/b_i})}\right\}
= \frac{1-1/e}{(1+\bmin)(1-e^{-1/\bmin})}. \label{eqn:bDependence}
\end{align}
Importantly, expression~\eqref{eqn:bDependence} represents the \textit{maximum fraction of $\OPT$ that can be obtained by any online algorithm} when there are both adversarial contexts $x^1,x^2,\ldots$ and capacity limits $\{b_i\}_{i\in [n]}$.
This fraction increases from 1/2 to $1-1/e$ as $\bmin$ increases from 1 to $\infty$. The asymptotic ratio of $1-1/e$ has been shown to be best-possible by \citet{MSVV07}, with the expression given in~\eqref{eqn:bDependence} denoting the best-known dependence on $\bmin$ due to \citet{MSL17}.
That is, even if all of the underlying probabilities $p_{x,a,i}$ are known and there is \textit{nothing to learn}, an online algorithm still cannot earn a fraction of $\OPT$ greater than~\eqref{eqn:bDependence}, which lies in [0.5, 0.632].
Consequently, if one attempts to directly measure the regret $\OPT-\bE[\ALG]$, then a regret sub-linear in $\OPT$ is impossible.

This is why we measure regret in comparison to $f_{\Psi}(\{b_i\}_{i\in [n]}, \epsilon) \cdot \OPT$.  In fact, we show this form of performance guarantee, with both a multiplicative and additive loss term like in~\eqref{eq:guarantee}, to be \textit{tight} for our Application 1 corresponding to online matching, in Section~\ref{sec:lb}.

\textbf{Tuning the $\epsilon$ Parameter.}
We let $\epsilon$ to be a parameter in [0,1], instead of fixing $\epsilon=0$, because it allows our algorithmic framework IBOL to balance between the two aforementioned losses caused by the competitive ratio $f_{\Psi}(\{b_i\}_{i\in [n]}, \epsilon)$ and the $(1+\epsilon)$-relaxed regret.
When $\epsilon$ increases, IBOL focuses on maximizing a more stringent competitive ratio but minimizing a $(1+\epsilon)$-relaxed regret, which causes its Inventory Balancing part to place a greater penalty on almost-depleted resources (through an $\epsilon$-perturbed potential function), and its Online Learning part to explore less (through our LazyUCB oracle).

Although there is no notion of ``optimal $\epsilon$'' given a problem instance due to the unknown probabilities and adversarial contexts, our Theorem~\ref{thm:main} provides a plausible method for setting $\epsilon$, based on maximizing its worst-case guarantee on $\bE[\ALG]$.
Denote $\textsf{BD}(\epsilon)$ as the upper bound on the $(1+\epsilon)$-relaxed regret of the underlying MAB oracle. An appropriate value of $\epsilon$ can then be found by solving the following tuning optimization problem, formulated below based on equation~\eqref{eq:guarantee}:
\begin{equation}\label{eq:tuning}
\max_{\epsilon\in [0, 1]} \left\{ f_{\Psi}(\{b_i\}_{i\in [n]}, \epsilon) \cdot \left(\OPT- \textsf{BD}(\epsilon)\right)\right\}.
\end{equation}
Although we have stated the tuning optimization problem for a general MAB oracle, in the next section we show how the regret bound $\textsf{BD}(\epsilon)$ materializes for different values of $\epsilon$, and we define the optimization problem over $\epsilon$ for Application~1 at the end of Section~\ref{sec:lazyUCB_1}. We also remark that the formulation of (\ref{eq:tuning}) involves knowing the value of $\OPT$, and our upper bounds $\textsf{BD}(\epsilon)$ on regret will involve knowing the value of $T$. The assumptions of knowing $\OPT, T$ could be justified when the optimal total reward and the number of customers can be estimated based on historical instances.
This provides a method for optimizing $\epsilon$ against the worst case, using less information than the full knowledge of $\rho, \{x^t\}^T_{t=1}$.  Of course, if one had full knowledge of $\rho, \{x^t\}^T_{t=1}$, then they could tune $\epsilon$ using simulation instead of using our bound, but the full knowledge assumption is much stronger than only needing an estimate of $\OPT$ and $T$.

Furthermore, we empirically find that setting $\epsilon$ to be larger, usually 1, will improve performance, justifying the benefit of having the tunable parameter $\epsilon$ in our IBOL framework.
When $\epsilon$ is larger, the Online Learning part of IBOL ends up focusing more on exploitation than exploration.
Such an insight is in line with the findings in a recent stream of papers \citep{BastaniBK17,KannanMRWW18} which show that reducing the amount of exploration in conventional MAB algorithms (more specifically, UCB algorithms) leads to better empirical performance, even though these (almost) exploration-free variants do not have a better theoretical performance guarantee than the conventional algorithms.

\textbf{Re-designing UCB for Worst-case $(1+\epsilon)$-Relaxed Regret.}
In addition to the tradeoff between the two sources of error, the notion of $(1+\epsilon)$-relaxed regret inspires the design of MAB oracles that differ significantly from the classical approach of UCB. Let's revisit definition~(\ref{eq:ep_reg}), and multiply both sides by $1/(1+\epsilon)$:
\begin{equation}\label{eq:relax_reg}
\frac{1}{1+\epsilon}\cdot \text{Reg}_\epsilon = \bE \left[\sum^T_{t=1}\frac{\max_{\bar{a}_t\in \cA}R^t(\bar{a}^t)}{1+\epsilon} - R^t(a^t)\right].
\end{equation}
The benchmark $\frac{1}{1+\epsilon}\sum^T_{t=1}\max_{\bar{a}_t\in \cA}R^t(\bar{a}^t)$ only requires the decision maker to identify an action that is $1/(1+\epsilon)$-optimal, i.e. an action $a^t$ such that $R^t(a^t) \geq \frac{1}{1+\epsilon} \max_{\bar{a}^t\in\cA}R^t(\bar{a}^t) = \frac{1}{1+\epsilon} R^t(a^t_*)$, which is an easier task than solving $\max_{\bar{a}^t\in \cA} R^t(\bar{a}^t)$. In particular, the former task requires less exploration on $\rho$ than the latter, and suggests that the decision maker could potentially perform less exploration on $\rho$ for achieving near-optimality for the online resource allocation problem.

Altogether, the main message of Theorem \ref{thm:main} is as follows. While we can adapt existing tools such as UCB to construct an MAB oracle for solving the online resource allocation problem, the problem in fact admits a much wider class of MAB oracles for achieving near-optimality. In particular, an MAB oracle that achieves a low $(1 +\epsilon)$-relaxed regret for some $\epsilon > 0$, which potentially involves less exploration than UCB, also leads us to near-optimality for the online resource allocation problem.

\section{MAB Oracles for Applications 1, 2}\label{sec:MAB_oracle}

In the previous sections, we proposed the IBOL algorithm that hedges against adversarial contexts $\{x^t\}^T_{t=1}$ while learning the outcome distribution $\rho$. In addition, we provided Theorem \ref{thm:main}, which related the expected reward of IBOL to the $(1+\epsilon)$-relaxed regret of the underlying MAB oracle. In this section we complete the picture by constructing MAB oracles, which conduct simultaneous exploration-exploitation, to solve the auxiliary problem and to overcome the uncertainty on the outcome distribution $\rho$.
We specialize to the settings of  $\rho, \cX, \cA$ under Applications 1, 2 (as defined in Sections~\ref{sec:app1}--\ref{sec:app2}) for our construction of MAB oracles.

In Section \ref{sec:UCB_1}, we construct the UCB oracle for our Application 1 in the case where $\epsilon = 0$. This UCB oracle is based on the classical UCB approach \citep{ACF02}, for which we upper-bound the unrelaxed regret $\text{Reg}_0$ in the auxiliary problem. In Section \ref{sec:lazyUCB_1}, we construct our LazyUCB oracle, which performs less exploration than the UCB oracle, in the case where $\epsilon\in (0, 1]$. We then demonstrate an upper bound to the $(1+\epsilon)$-relaxed regret for LazyUCB, hence showing that it is possible to achieve near-optimality with less exploration than the classical UCB approach. 
In Section \ref{sec:lb}, we present our negative result establishing tightness in the context of Application 1.
In Section \ref{sec:UCB_lazyUCB_2}, we show how the machinery developed in Sections \ref{sec:UCB_1}, \ref{sec:lazyUCB_1} for Application 1 can be generalized to Application 2.


\subsection{UCB Oracles for Application 1 ($\epsilon= 0$)}\label{sec:UCB_1}
We start with a reminder on Application 1. The action set is $\cA = \{(i, k)\}_{i\in [n], k\in [K]}$. When the action $(i, k)$ is taken at time $t$, where the customer has feature $x^t$, the feedback $\vy\in \{0, 1\}^n$ is equal to $\mathbf{1}(x^t_i = 1)\ve_i$ with latent probability $q_{i, k}$, and equal to $\vzero$ with latent probability $1 - q_{i, k}$ (there is only any reason to take an action $(i, k)$ at time $t$ if context $x^t_i=1$). For the auxiliary problem, the discounted reward at time $t$ under action $a = (i, k)$ is
\begin{equation}\label{eq:aux_1}
R^t(a) = \sum_{i\in [n]} r^t_{i}p_{x^t, a, i} = r^t_i \cdot \mathbf{1}(x^t_i = 1) \cdot q_{(i,k)}.
\end{equation}
The UCB oracle at time $t$ is provided in Algorithm \ref{alg:UCB_app1}. This oracle is to be used on Application 1 when $\epsilon =0$. The oracle inputs the information ${\cal F}^{t-1}, x^t, (r^t_i)_{i\in[n]}$ that are known at the start of time $t$, and output the action $a^t = (i^t, k^t)$ for the time step. For estimating the latent probability $q_{(i, k)}$ for each $i\in [n], k\in [K]$, we consider
$$
M^t_{(i, k)} = \sum^{t-1}_{s=1} \mathbf{1}(a^s = (i, k) ),\qquad \bar{q}^t_{(i, k)} = \frac{\sum^{t-1}_{s=1} \mathbf{1}(a^s = (i, k), \vy^s_i = 1 ) }{\max\{M^{t}_{(i, k)} , 1\}}.
$$
The parameter $M^t_{(i, k)} $ counts the number of times the algorithm takes the action $(i, k)$ during time steps $1, \ldots, t-1$. The statistic $\bar{q}^t_{(i, k)}$ serves to estimate the latent parameter $q_{(i, k)}$.
For each and every $(i, k)$, the quantities $M^t_{(i, k)}, \bar{q}^t_{(i, k)}$ can be constructed based on the observations ${\cal F}^{t-1}$ during time $1, \ldots, t-1$.

\begin{algorithm}[t]
\caption{UCB oracle at time $t$ for Application 1, in the case where $\epsilon= 0$}\label{alg:UCB_app1}
\begin{algorithmic}[1]
\State Input: observation ${\cal F}^{t-1}$ from time 1 to $t-1$, context $x^t$ and discounted rewards $(r^t_i)_{i\in [n]}$.
\State Use ${\cal F}^{t-1}$ to compute the statistics $\{M^t_{(i, k)}\}_{i\in [n], k\in [K]}, \{\bar{q}^t_{(i, k)}\}_{i\in [n], k\in [K]}$.
\State \label{alg:UCB_app1_step2}  Set $\delta_t = \frac{1}{(1 + t)^2}$.
\State \label{alg:UCB_app1_step3} For each action $a = (i, k)$, compute a UCB for the discounted revenue with action $(i, k)$:
$$
 \mathsf{UCB}^t (a) = r^t_i \cdot \mathbf{1}(x^t_i = 1) \cdot \left[ \bar{q}^t_{(i, k)} +  \mathsf{rad}(\bar{q}^t_{(i, k)}, M^t_{(i, k)}, \delta^{(t)}) \right].
$$
\State \label{alg:UCB_app1_step4} Output an action $a^t = (i^t, k^t)$ which satisfies $$a^t \in \underset{a = (i, k)\in \cA}{\text{argmax }}\mathsf{UCB}^t (a) .$$
\end{algorithmic}
\end{algorithm}

While the empirical mean $\bar{q}^t_{(i, k)}$ is a natural estimate to $q_{(i, k)}$, the decision maker needs to quantify the accuracy of the estimate $\bar{q}^t_{(i, k)}$, in order to decide if it wishes to explore other actions' probabilities, or if it wishes to use the estimate $\bar{q}^t_{(i, k)}$ for exploitation. The accuracy of the estimate $\bar{q}^t_{(i, k)}$ is quantified by a confidence radius for the estimate.

In the forthcoming UCB and LazyUCB Oracle, the decision maker conducts simultaneous exploration and exploitation by replacing the latent probability $q_{(i, k)}$ with an \emph{optimistic estimate}, which is equal to the sum of the the empirical mean $\bar{q}^t_{(i, k)}$ (exploitation) and a confidence radius (exploration). The use of an optimistic estimate embodies the famous ``optimism in the face of uncertainty'' principle in the multi-armed bandit literature.
For the UCB oracle, we follow the approach by \citet{ACF02,KSU08} to define the confidence radius. 
For $p>0, M\in \mathbb{Z}_{\geq 0}, \delta\in (0, 1) $, let
\begin{equation}\label{eq:rad}
\mathsf{rad}(p, M; \delta) := \sqrt{\frac{2p\log(1/\delta)}{ \max\{M, 1\} }} + \frac{3\log(1/\delta)}{ \max\{M, 1\}}.
\end{equation}
In Line \ref{alg:UCB_app1_step3} in the UCB oracle, we replace the latent $q_{(i, k)}$ with the optimistic estimate $\bar{q}^t_{(i, k)}+ \mathsf{rad}(\bar{q}^t_{(i, k)}, M^t_{(i, k)}, \delta_t)$, where $\delta_t \in (0, 1)$ is a confidence parameter defined in Line \ref{alg:UCB_app1_step2}. The definition of $\mathsf{rad}$ is justified by the following Lemma.
\begin{lemma}[\cite{KSU08}]\label{lemma:bd_UCB_1}
For each $t\in [T]$, consider the event ${\cal E}^t = \cap_{i\in [n], k\in [K]}{\cal E}^t_{i, k}$, where ${\cal E}^t_{(i, k)}$ is defined as
$${\cal E}^t_{(i, k)} = \left\{\left|\bar{q}^t_{(i, k)}  - q_{(i, k)} \right| \leq \mathsf{rad}(\bar{q}^t_{(i, k)}, M^t_{(i, k)}, \delta_t ) \leq 3\cdot \mathsf{rad}(q_{(i, k)}, M^t_{(i, k)}, \delta_t)\right\}.$$
Then we have $\Pr({\cal E}^t) \geq 1 - \frac{4nK}{1+t}$.
\end{lemma}
Lemma is proved in Appendix \ref{sec:pf_lem_bd_UCB_1}. While the Lemma is first proposed in \citet{KSU08}, we still provide the proof to make the constants involved in $\mathsf{rad}$ explicit.
Lemma \ref{lemma:bd_UCB_1} is crucial for justifying the UCB in step \ref{alg:UCB_app1_step3}. Indeed, if the event ${\cal E}^t$ holds, then for any $a = (i, k)$,
\begin{align}
R^t(a) & = r^t_i \cdot \mathbf{1}(x^t_i = 1) \cdot q_{i,k} \nonumber\\
&\leq  r^t_i \cdot \mathbf{1}(x^t_i = 1) \cdot \left[\bar{q}^t_{(i, k)} +\mathsf{rad}(\bar{q}^t_{(i, k)}, M^t_{(i, k)}, \delta_t)\right] \nonumber\\
& = \mathsf{UCB}^t(a).\label{eq:app_1_UCB}
\end{align}
Therefore, the quantity $\mathsf{UCB}^t(a)$ defined in Line \ref{alg:UCB_app1_step3} is a bona fide upper bound of the discounted reward $R^t(a)$ with high probability.
En route, we show that even though the auxiliary problem involves non-stationary rewards, we can still harness existing machinery on UCB algorithms.

Finally, in Line \ref{alg:UCB_app1_step4} we choose an action $a^t$ that maximizes the UCB. Without loss of generality, we assume that $x^t_{i^t} = 1$.
To demonstrate the salience of the UCB oracle, we bound its regret $\text{Reg}_0 = \bE \left[\sum^T_{t=1}R^t(a^t_*) -  R^t(a^t)\right]$ in the theorem below.

\begin{theorem}\label{thm:UCB_app1}
Consider the UCB oracle (Algorithm \ref{alg:UCB_app1}) for Application 1. The oracle has regret
$$\text{Reg}_0 = O\left( \sqrt{nK \cdot \OPT \cdot \log(T)} + nK\log(T)\log\frac{T}{nK} \right) = \tilde{O}\left( \sqrt{nK\cdot \OPT }\right),$$
where $\OPT $ is the optimal value of the LP \textbf{Primal}, and the notation $\tilde{O}(\cdot)$ hides the logarithmic dependence on $n, K, T$.
\end{theorem}
Theorem \ref{thm:UCB_app1} is proved in Appendix \ref{app:pf_thm_UCB_app1}. On a high level, the Theorem is proved by incorporating the analytical tools on UCB algorithms from \citet{ACF02}, with extra care that yields to dependence on $\OPT$. Clearly, we know that $\OPT = O(T)$, and by replacing $\OPT$ with the upper bound $O(T)$, we have $\text{Reg}_0 = \tilde{O}(\sqrt{nKT})$, which coincides with the $\tilde{O}(\cdot)$ bound for an $nK$-armed bandits problem with $T$ time steps. In passing, we remark that in the large-volume regime \citep{BZ11} where $b_\text{min}$ grows linearly with $T$, we do have $\OPT = c T$ for some constant $c$ that depends on the model but is independent of $T$. The dependence on $\OPT$ provides a more refined guarantee than the dependence on $T$ when $b_\text{min} = o(T)$.

Combined with the IBOL algorithm (Algorithm \ref{alg:ibol}), we achieve the following performance guarantee for Application 1.
\begin{corollary}\label{cor:UCB_app1}
For Application 1, the IBOL algorithm (Algorithm \ref{alg:ibol}) with the UCB oracle (Algorithm \ref{alg:UCB_app1}) yields expected reward $\bE[\ALG]$ that satisfies
\begin{align}
 \bE[\ALG] &\geq  \frac{1-1/e}{(1+\bmin)(1-e^{-1/\bmin})} \cdot \left[\OPT -  \tilde{O}\left( \sqrt{nK\cdot \OPT }\right)\right]. \label{eq:cor_UCB_1}
\end{align}
\end{corollary}
As an illustration of Theorem \ref{thm:main}, the Corollary illustrates the two sources of error for the online resource allocation problem in Application 1. The competitive ratio $\frac{1-1/e}{(1+\bmin)(1-e^{-1/\bmin})}$ is due to the adversarial uncertainty on $x^1, \ldots, x^T$, and the regret bound $\tilde{O}\left( \sqrt{nK\cdot \OPT }\right)$ is due to the model uncertainty on the probabilities $q_{(i, k)}$.

\subsection{LazyUCB Oracles for Application 1 ($\epsilon \in (0, 1]$)}\label{sec:lazyUCB_1}
After the construction of the UCB oracle for Application 1, which is for the case where $\epsilon =0$, we construct the LazyUCB oracle, which is for the case where $\epsilon \in (0, 1]$. The LazyUCB oracle, which involves $\epsilon$, is exhibited in Algorithm \ref{alg:Lazy_UCB_app1}. 

\begin{algorithm}
\caption{LazyUCB Oracle at time $t$ for Application 1, in the case where $\epsilon\in (0, 1]$}\label{alg:Lazy_UCB_app1}
\begin{algorithmic}[1]
\State Input: exploitation parameter $\epsilon \in (0, 1]$,  observation ${\cal F}^{t-1}$ from time 1 to $t-1$, context $x^t$ and discounted rewards $(r^t_i)_{i\in [n]}$.
\State Use ${\cal F}^{t-1}$ to compute the statistics $\{M^t_{(i, k)}\}_{i\in [n], k\in [K]}, \{\bar{q}^t_{(i, k)}\}_{i\in [n], k\in [K]}$.
\State Set $\delta_t = \frac{1}{(1 + t)^2}$.
\State \label{alg:lazy_UCB_app1_step3} For each action $a = (i, k)$, compute a \emph{lazy} UCB for the discounted revenue with action $(i, k)$:
$$
 \mathsf{LazyUCB}^t (a) = r^t_i \cdot \mathbf{1}(x^t_i = 1) \cdot \left[ \bar{q}^t_{(i, k)} + \mathsf{LazyRad}^t(i, k)  \right].
$$
\State \label{alg:lazy_UCB_app1_step5} Output an action $a^t = (i^t, k^t)$ which satisfies $$a^t \in \underset{a = (i, k)\in \cA}{\text{argmax }}\mathsf{LazyUCB}^t (a) .$$
\end{algorithmic}
\end{algorithm}

Similar to the UCB oracle, the LazyUCB oracle also hinges on constructing an optimistic estimate for each latent probability $q_{(i, k)}$. Different from the UCB oracle, however, the LazyUCB oracle employs a smaller confidence radius. Hence, the LazyUCB oracle focuses more on exploitation, and less on exploration in comparison to the UCB oracle. The confidence radius employed by the LazyUCB oracle is shown in Line \ref{alg:lazy_UCB_app1_step3} in the algorithm.

To construct the confidence radii for the LazyUCB oracle, for $\epsilon \in [0, 1], M\in \mathbb{Z}_{\geq 0}, \delta\in (0, 1)$, we define
\begin{equation}\label{eq:lad}
\mathsf{lad}(\epsilon, M ; \delta) = \frac{2 + \epsilon}{\epsilon }\cdot \frac{\log (1/\delta)}{\max\{M, 1\}}.
\end{equation}
The optimistic estimate for $q_{(i, k)}$ at time $t$ under the LazyUCB oracle is
$$
\bar{q}^t_{(i, k)} + \mathsf{LazyRad}^t(i, k) ,
$$
where we define
$$
\mathsf{LazyRad}^t(i, k) = \min\left\{\mathsf{lad}(\epsilon, M^t_{(i, k)} ; \delta_t), \mathsf{rad}(\bar{q}^t_{(i, k)}, M^t_{(i, k)}, \delta^{(t)})\right\},
$$
with $\mathsf{rad}$ as defined in (\ref{eq:rad}). The other confidence radius $\mathsf{lad}(\epsilon, M^t_{(i, k)} ; \delta_t)$ can be can be interpreted as follows. The confidence radius $\mathsf{lad}$ involves an \emph{exploitation parameter} $\epsilon \in [0, 1]$, which controls the amount of exploitation conducted by the LazyUCB oracle. As $\epsilon$ increases,  $\mathsf{lad}(\epsilon, M ; \delta)$ decreases. In particular, when $\epsilon =1$, we have
$$
\mathsf{lad}(\epsilon, M ; \delta) < \mathsf{rad}(p, M; \delta)
$$
for any $p, M, \delta$. For any $
\epsilon > 0$, it is critical to observe that we still have $
\mathsf{lad}(\epsilon, M ; \delta) \leq \mathsf{rad}(p, M; \delta)
$ as long as $p > 0$ and $M$ is sufficiently large, since the dominant term in $\mathsf{rad}(p, M; \delta)$ is of order $\sqrt{p / M}$ while $\mathsf{lad}(\epsilon, M ; \delta) $ scales as $1 / (\epsilon M)$.

In the extreme case where we set $\epsilon =0$, we have $\mathsf{lad}(\epsilon, M^t_{(i, k)};\delta) = 1$, and the LazyUCB oracle is reduced to the UCB oracle. In another extreme case when we set $\epsilon$ to be 1, we have $\mathsf{lad}(\epsilon, M ; \delta) = \frac{3\log (1/\delta)}{\max\{M, 1\}}. $ It is worth noting that the lazy confidence radius $\mathsf{lad}(\epsilon, M ; \delta) $ does not shrink to zero, as we define $\mathsf{lad}(\epsilon, M ; \delta) $ in a way that still induces a minute amount of optimistic exploration when $\epsilon$ is large. Finally, similar to the UCB oracle, in Line \ref{alg:lazy_UCB_app1_step5} we chooses an action $a^t$ that maximizes the optimistic estimate.

We justify the definition of the lazy confidence radius $\mathsf{rad}$ in the following Lemma:
\begin{lemma}\label{lemma:bd_LazyUCB_1}
For each $t$, consider the event ${\cal E}^t = \left(\cap_{i\in [n], k\in [K]} {\cal U}^t_{(i, k)}\right) \cap \left(\cap_{i\in [n], k\in [K]} {\cal L}^t_{(i, k)}\right)$, where ${\cal U}^t_{(i, k)} $ is the event
\begin{equation}\label{eq:lemma_bd_Lazy_1_lower}
q_{(i, k)} \leq \left( 1 + \frac{\epsilon}{2}\right) \cdot \left[ \bar{q}^t_{(i, k)} + \mathsf{LazyRad}^t(i, k) \right],
\end{equation}
and ${\cal L}^t_{(i, k)} $ is the event
\begin{equation}\label{eq:lemma_bd_Lazy_1_upper}
\bar{q}^t_{(i, k)} \leq \left(1 + \frac{\epsilon}{2 + \epsilon}\right)\left[ q_{(i, k)} + \mathsf{LazyRad}^t(i, k) \right] .
\end{equation}
Then we have $\Pr({\cal E}^t) \geq 1 - \frac{7nK}{1+t}$.
\end{lemma}
Lemma~\ref{lemma:bd_LazyUCB_1} is proved in Appendix \ref{app:pf_lemma_bd_LazyUCB_1}.
Inequality (\ref{eq:lemma_bd_Lazy_1_lower}) shows that the lazy optimistic estimate $\bar{q}^t_{(i, k)} + \mathsf{LazyRad}^t(i, k)$ is ``optimistic'' in an approximate sense, captured by the multiplicative factor $1 + \epsilon /2$. Inequality (\ref{eq:lemma_bd_Lazy_1_upper}) shows that, despite the approximate nature, the lazy optimistic estimate is still close to the actual latent probability, as quantified in the inequality. It is useful to note that when we set $\epsilon = 0$, we recover Lemma \ref{lemma:bd_UCB_1} from the original UCB algorithm.

If the event ${\cal E}^t$ holds, then for any action $a = (i, k)$, the LazyUCB in Line \ref{alg:lazy_UCB_app1_step3} satisfies
\begin{align}
R^t(a) & = r^t_i \cdot \mathbf{1}(x^t_i = 1) \cdot q_{i,k} \nonumber\\
&\leq  r^t_i \cdot \mathbf{1}(x^t_i = 1) \cdot \left( 1 + \frac{\epsilon}{2}\right) \cdot \left[\bar{q}^t_{(i, k)} +\mathsf{LazyRad}^t(i,k)\right] \nonumber\\
& = \left( 1 + \frac{\epsilon}{2}\right)\cdot   \mathsf{LazyUCB}^t(a).\label{eq:app_1_LazyUCB}
\end{align}

We provide the following performance guarantee on the LazyUCB oracle for the auxiliary problem, with the metric of $(1+ \epsilon)$-relaxed regret $\text{Reg}_\epsilon$.
\begin{theorem}\label{thm:LazyUCB_app1}
Consider the LazyUCB oracle (Algorithm \ref{alg:Lazy_UCB_app1}) for Application 1. The oracle has $\frac{1}{1+\epsilon}\text{Reg}_\epsilon$ at most
\begin{align}
 &  \min\left\{ O\left( \sqrt{nK \cdot \OPT \cdot \log T}\right) , O\left( \left(1 + \frac{1}{\epsilon}\right) \cdot nK\log(T)  \log\left(\frac{T}{nK}\right) \right) \right\} + O\left( nK\log(T) \log\left(\frac{T}{nK}\right) \right)\nonumber\\
= &  \min\left\{\tilde{O}\left(\sqrt{nK \OPT}\right), \tilde{O}\left(\left(1 + \frac{1}{\epsilon}\right) \cdot nK \right) \right\}. \nonumber
\end{align}
\end{theorem}
Theorem~\ref{thm:LazyUCB_app1} is proved in Appendix \ref{app:pf_thm_LazyUCB_app1}. While the regret bound in Theorem~\ref{thm:LazyUCB_app1} is smaller than the regret bound for the UCB oracle in Theorem \ref{thm:UCB_app1}, it does not mean that the LazyUCB oracle earns a greater reward on the auxiliary problem than the UCB oracle. It is important to note that Theorems \ref{thm:UCB_app1}, \ref{thm:LazyUCB_app1} involve different notions of regret. These Theorems together suggest that the auxiliary problem can be solved by a variety of MAB oracles, such as UCB or LazyUCB, but they do not suggest that one oracle is better than the other.

In conjunction with Theorem \ref{thm:main}, we arrive at the following performance guarantee for the IBOL algorithm using the LazyUCB oracle.
\begin{corollary}\label{cor:LazyUCB_app1}
For Application 1, the IBOL algorithm (Algorithm \ref{alg:ibol}) with the LazyUCB oracle (Algorithm \ref{alg:Lazy_UCB_app1}) yields expected reward $\bE[\ALG]$ that satisfies
\begin{align}
 \bE[\ALG] &\geq f_\Psi(\{b_i\}_{i\in [n]}, \epsilon) \cdot \left[\OPT  - \min\left\{\tilde{O}\left( \sqrt{nK \OPT}\right), \tilde{O}\left(\left(1 + \frac{1}{\epsilon}\right)\cdot nK \right) \right\}   \right] \label{eq:cor_LazyUCB_1},
\end{align}
where $f_\Psi(\{b_i\}_{i\in [n]}, \epsilon)$ is defined in equation (\ref{eq:f_Psi}).
\end{corollary}
To this end, it is important to note that when we specify $\epsilon  = 0$ in Corollary \ref{cor:LazyUCB_app1}, we arrive at the bound in Corollary \ref{cor:UCB_app1}. Indeed, it is crucial to recall that when we set $\epsilon = 0$ in the LazyUCB oracle, we recover the UCB oracle.

We conclude our discussion with two remarks. First, it is useful to compare the performance guarantee under the LazyUCB oracle in Corollary \ref{cor:LazyUCB_app1} (where $\epsilon\in (0, 1]$) with that under the UCB oracle in Corollary \ref{cor:UCB_app1} (where $\epsilon = 0)$. On one hand, the competitive ratio $f_\Psi(\{b_i\}_{i\in [n]}, 0)$ in Corollary \ref{cor:UCB_app1} is greater than or equal to the competitive ratio $f_\Psi(\{b_i\}_{i\in [n]}, \epsilon)$ in Corollary \ref{cor:LazyUCB_app1}. On the other hand, the regret term in (\ref{eq:cor_LazyUCB_1}) is less than or equal to the regret term in (\ref{eq:cor_UCB_1}).

Second, going back to the question of tuning $\epsilon$, when $\OPT, T$ are known, an appropriate choice for $\epsilon$ can be made by solving the optimization problem \begin{equation}\label{eq:opt_app1}
\min_{\epsilon\in [0, 1]} \left\{f_\Psi(\{b_i\}_{i\in [n]}, \epsilon) \cdot \left[\OPT  - \min\left\{\tilde{O}\left( \sqrt{nK \OPT}\right), \tilde{O}\left(\left(1 + \frac{1}{\epsilon}\right)\cdot nK \right) \right\}   \right]\right\}.
\end{equation}
Although the optimization problem (\ref{eq:opt_app1}) is not convex in $\epsilon$ in general, an optimal $\epsilon$ can still be identified by a one-dimensional line search on $[0, 1]$.  While the exact expressions of $\tilde{O}\left( \sqrt{nK \OPT}\right), \tilde{O}\left(\left(1 + \frac{1}{\epsilon}\right)\cdot nK \right) $ are suppressed due to the use of $\tilde{O}(\cdot)$ notation, the optimization problem (\ref{eq:opt_app1}) can be explicitly defined by replacing $\tilde{O}\left( \sqrt{nK \OPT}\right), \tilde{O}\left(\left(1 + \frac{1}{\epsilon}\right)\cdot nK \right) $ respectively with their explicit expressions \eqref{eq:explicit_1}, \eqref{eq:pf_app_1_LazyUCB_8} in Appendix \ref{app:pf_thm_LazyUCB_app1}.


\subsection{Tightness of our Guarantee for Application 1}\label{sec:lb}
We conclude our discussion on Application 1 by providing the following negative result on the expected reward achieved by any feasible online algorithm.
\begin{theorem} \label{thm:negative}
Let $n,b,K$ be any positive integers satisfying $b\ge K\ge3$. For any online algorithm that is feasible to Application 1, there exists a problem instance under which
\begin{align*}
\bE[\ALG] & \leq \left(1 - \frac{1}{e}\right)\OPT - \Omega(\sqrt{K \OPT}).
\end{align*}
\end{theorem}
Theorem~\ref{thm:negative} is proved in Appendix \ref{app:thm_lowerBound}. The main message of the theorem is that any feasible online algorithm must suffer a loss in reward from both the adversarial uncertainty on $\{x^t\}^T_{t=1}$ and the model uncertainty on $\{q_{(i, k)}\}_{i\in [n], k\in [K]}$. Our paper is the first to study online problems with both sources of uncertainty in a resource constrained setting.  The proof involves crafting a special class of problem instances.

To elaborate, the adversarial uncertainty construction requires having an upper-triangular graph whose ordering of offline vertices is hidden to the online algorithm, in which case it is impossible to do better than arbitrarily ``guessing'' an offline vertex to probe at each stage.
In our combined construction, each offline vertex actually corresponds to $2b$ arms, one of which is a ``secret'' arm which successfully matches with probability $1/2+\varepsilon$ (instead of $1/2-\varepsilon$) upon a probe.
The online algorithm also suffers from not being able to learn the secret arm for each offline vertex, and hence loses an additional $\varepsilon$ in the matches made at each stage.
However, this means that more offline vertices remain unmatched, making the online algorithm less likely to get stuck in the future.
To see that this $\varepsilon$-loss is not later recouped by the online algorithm (in terms of first-order regret) requires an intricate analysis, combining the information-theoretic framework in \citet{auer2002nonstochastic} with the Yao's-minimax proof in \citet{MSVV07}.
To the best of our understanding, such an analysis is new to our paper, and our negative result is not possible to confirm without this detailed analysis.

\subsection{UCB and LazyUCB Oracles for Application 2}
\label{sec:UCB_lazyUCB_2}
Analogous versions of the UCB and LazyUCB oracles for Application 1 can be constructed for Application 2. We start with a reminder on the mathematical model of Application 2. The action set $\cA$ is $[n]$. When the context (customer segment) is $x$ and the action is $i$, the outcome $\vy$ is equal to $\ve_i$ with probability $p_{x, i}$ and is equal to $\vzero$ with probability $1 - p_{x, i}$. 
The probability terms in $\{p_{x, i}\}_{x\in \cX, i\in [n]}$ are not known but are to be learned. We consider the statistics
$$
L^t_{x, i} = \sum^{t-1}_{s=1} \mathbf{1}(x^s = x,  a^s = i ),\qquad \bar{p}^t_{x, i} = \frac{\sum^{t-1}_{s=1} \mathbf{1}(x^s = x, a^s = i, \vy^s_i = 1 ) }{\max\{L^{t}_{x, i} , 1\}}.
$$
The statistics $L^t_{x, i}, \bar{p}^t_{x, i}$ can be constructed from the observations ${\cal F}^{t-1}$ during time $1, \ldots, t-1$. The UCB and LazyUCB oracles for Application 2 are provided in Algorithms \ref{alg:UCB_app2}, \ref{alg:Lazy_UCB_app2} respectively.

\begin{algorithm}
\caption{UCB oracle at time $t$ for Application 2}\label{alg:UCB_app2}
\begin{algorithmic}[1]
\State Input: observation ${\cal F}^{t-1}$ from time 1 to $t-1$, context $x^t$ and discounted rewards $(r^t_i)_{i\in [n]}$.
\State Use ${\cal F}^{t-1}$ to compute the statistics $\{L^t_{x, i}\}_{x\in {\cal X}, i\in [n]}, \{\bar{p}^t_{x, i}\}_{x\in {\cal X}, i\in [n]}$.
\State Set $\delta_t = \frac{1}{(1 + t)^2}$.
\State For each action $i\in [n]$, compute a UCB for associated the discounted revenue:
$$
 \mathsf{UCB}^t (x^t , i) = r^t_i  \cdot \left[ \bar{p}^t_{x^t , i} +  \mathsf{rad}(\bar{p}^t_{x^t , i}, L^t_{x^t , i}, \delta_{(t)}) \right].
$$
\State Output an action $i^t$ which satisfies $$i^t \in \underset{i\in [n]}{\text{argmax }}\mathsf{UCB}^t (x^t , i) .$$
\end{algorithmic}
\end{algorithm}

\begin{algorithm}
\caption{LazyUCB Oracle at time $t$ for Application 2}\label{alg:Lazy_UCB_app2}
\begin{algorithmic}[1]
\State Input: exploitation parameter $\epsilon \geq 0$,  observation ${\cal F}^{t-1}$ from time 1 to $t-1$, context $x^t$ and discounted rewards $(r^t_i)_{i\in [n]}$.
\State Use ${\cal F}^{t-1}$ to compute the statistics $\{L^t_{x, i}\}_{x\in {\cal X}, i\in [n]}, \{\bar{p}^t_{x, i}\}_{x\in {\cal X}, i\in [n]}$.
\State Set $\delta_t = \frac{1}{(1 + t)^2}$.
\State For each action $i\in [n]$,  compute a \emph{lazy} UCB for the discounted revenue with action $(i, k)$:
$$
 \mathsf{LazyUCB}^t (x^t , i) = r^t_i \cdot \left[ \bar{p}^t_{x^t, i} + \mathsf{LazyRad}^t(x^t , i) \right],
$$
where
$$
\mathsf{LazyRad}^t(x^t, i) = \min\left\{\mathsf{lad}(\epsilon, L^t_{x^t , i} ; \delta_t), \mathsf{rad}(\bar{p}^t_{x^t, i}, L^t_{x^t, i}, \delta_{t})\right\}.
$$
\State Output an action $i^t$ which satisfies $$i^t \in \underset{i\in [n]}{\text{argmax }}\mathsf{LazyUCB}^t (x^t , i) .$$
\end{algorithmic}
\end{algorithm}
Note that Algorithms \ref{alg:UCB_app2}, \ref{alg:Lazy_UCB_app2} are analogous to Algorithms \ref{alg:UCB_app1}, \ref{alg:Lazy_UCB_app1} respectively. Their performance guarantees are also analogous. Let $K = |{\cal X}|$ denote the total number of customer segments.
\begin{corollary}\label{cor:UCB_app2}
For Application 2, the IBOL algorithm (Algorithm \ref{alg:ibol}) with the UCB oracle (Algorithm \ref{alg:UCB_app2}) yields expected reward $\bE[\ALG]$ that satisfies
\begin{align*}
 \bE[\ALG] &\geq   \frac{1-1/e}{(1+\bmin)(1-e^{-1/\bmin})} \cdot \left[\OPT -  \tilde{O}\left( \sqrt{nK\cdot \OPT }\right)\right].
\end{align*}
\end{corollary}
\begin{corollary}\label{cor:LazyUCB_app2}
For Application 2, the IBOL algorithm (Algorithm \ref{alg:ibol}) with the LazyUCB oracle (Algorithm \ref{alg:Lazy_UCB_app2}) yields expected reward $\bE[\ALG]$ that satisfies
\begin{align*}
 \bE[\ALG] &\geq f_\Psi(\{b_i\}_{i\in [n]}, \epsilon) \cdot \left[\OPT  - \min\left\{\tilde{O}\left( \sqrt{nK \OPT}\right), \tilde{O}\left(\left(1 + \frac{1}{\epsilon}\right)\cdot nK\right) \right\} \right]  ,
\end{align*}
where $f_\Psi(\{b_i\}_{i\in [n]}, \epsilon)$ is defined in equation (\ref{eq:f_Psi}).
\end{corollary}
The proofs for these corollaries hinges on proving bounds on $\text{Reg}_0, \text{Reg}_{\epsilon}$ for the UCB, LazyUCB oracles respectively, and these proofs can be reproduced by replacing $\{q_{(i, k)}\}_{i\in [n], k\in [K]}$ in Appendices \ref{app:pf_thm_UCB_app1}, \ref{app:pf_thm_LazyUCB_app1} with $\{p_{x,i}\}_{x\in {\cal X}, i\in [n]}$.  The comparison between the UCB oracle and the LazyUCB oracle, as well as the tuning of $\epsilon$, are similar to that in Application 1, so we do not repeat the discussion here.

\section{Numerical Studies}
\label{sec:numerical}

In this section, we conduct numerical experiments to demonstrate the performance of the proposed algorithms. First, in Section \ref{sec:numerical2}, we use synthetic data to test the three-way trade-off between hedging, exploration, and exploitation, using our LazyUCB oracle for Application 1. Then in Section \ref{sec:numerical1}, we simulate a dynamic assortment optimization problem using a real-world dataset.

\subsection{Experiments on Synthetic Data}\label{sec:numerical2}

We conduct experiments for the online matching with unknown matching probabilities model described in Section \ref{sec:app1}.
We test the role of $\epsilon$ in our algorithmic framework IBOL by using our LazyUCB oracle with $\epsilon$ ranging from 0 to 1.
Recall that $\epsilon=0$ corresponds to the classical UCB oracle, while $\epsilon=0.1,\ldots,1$ corresponds to a LazyUCB oracle that does progressively less exploration.

In all test cases, we set the number of unknown arms per resource to be $K=5$,
independently draw their unknown probabilities $q_{(i,k)}$ from $[0.2,0.5]$ uniformly at random, and
independently draw the resource adjacencies $x^t_i=1$ from $\{0,1\}$ uniformly at random.
We set the reward values $r_i$ to be identical for all resources $i$.
We consider different scales of the problem, with the number of resources $n$ lying in $\{5,50\}$, the number of times steps $T$ lying in $\{10^5,10^6,10^7\}$, and the capacity $b_i$ of each resource $i$ being identical to some $B$ which varies depending on the combination of $n$ and $T$.

We report the simulation results in Figures \ref{fig:lazy_new3} to \ref{fig:lazy_new2}.
For each test case, the expected total reward $\bE[\ALG]$ is an average value based on $500$ simulation replications.

\begin{figure}[h]
\begin{center}
\includegraphics[scale=0.4,trim={1cm 3cm 1cm 3cm},clip]{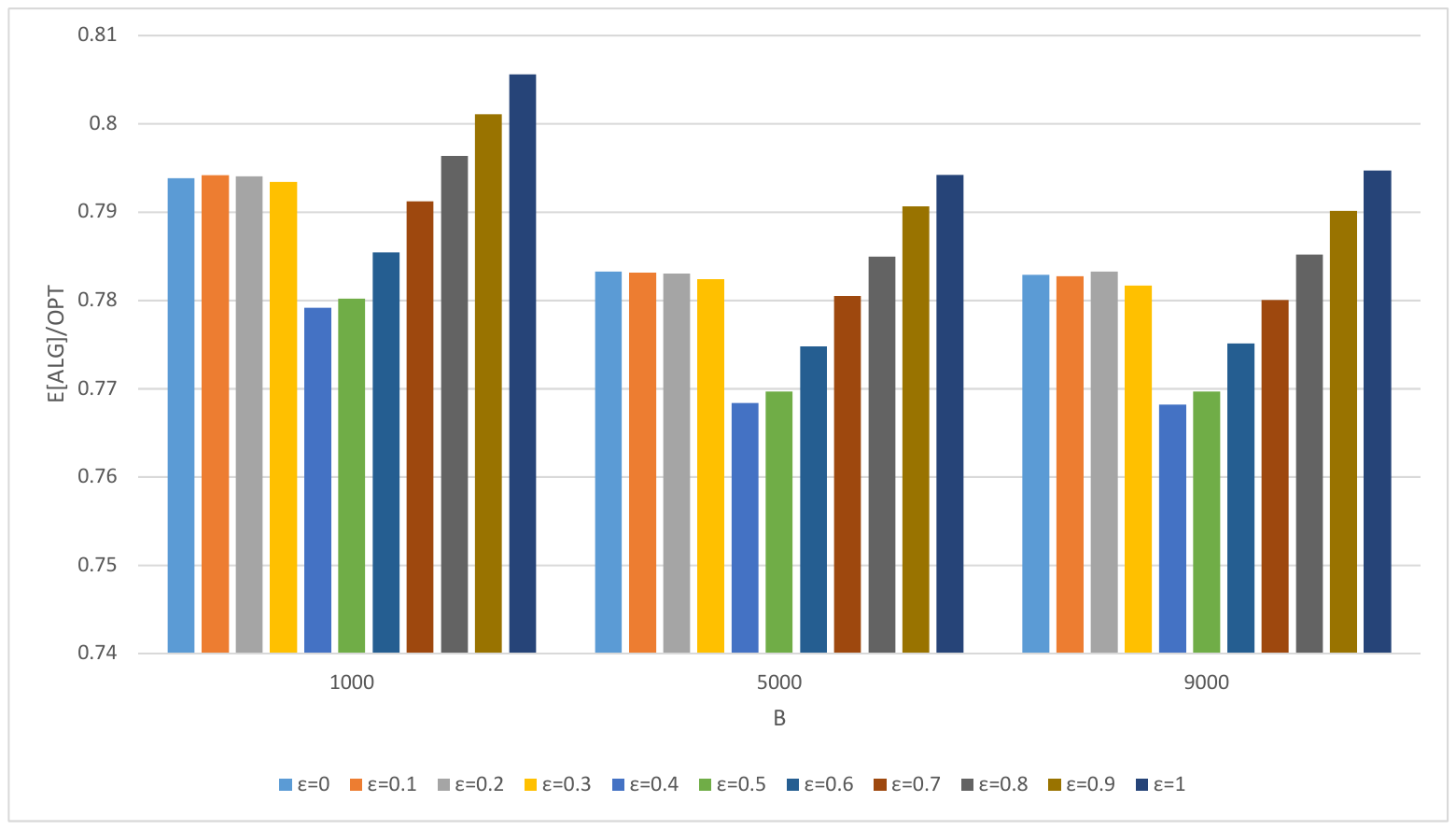}
\caption{Performance ratios of algorithms using UCB and LazyUCB oracles. $n=5$. $b_i = B$ for all resources $i \in [n]$. $T = 10000$.}
\label{fig:lazy_new3}
\end{center}
\end{figure}

\begin{figure}[h]
\begin{center}
\includegraphics[scale=0.4,trim={1cm 3cm 1cm 3cm},clip]{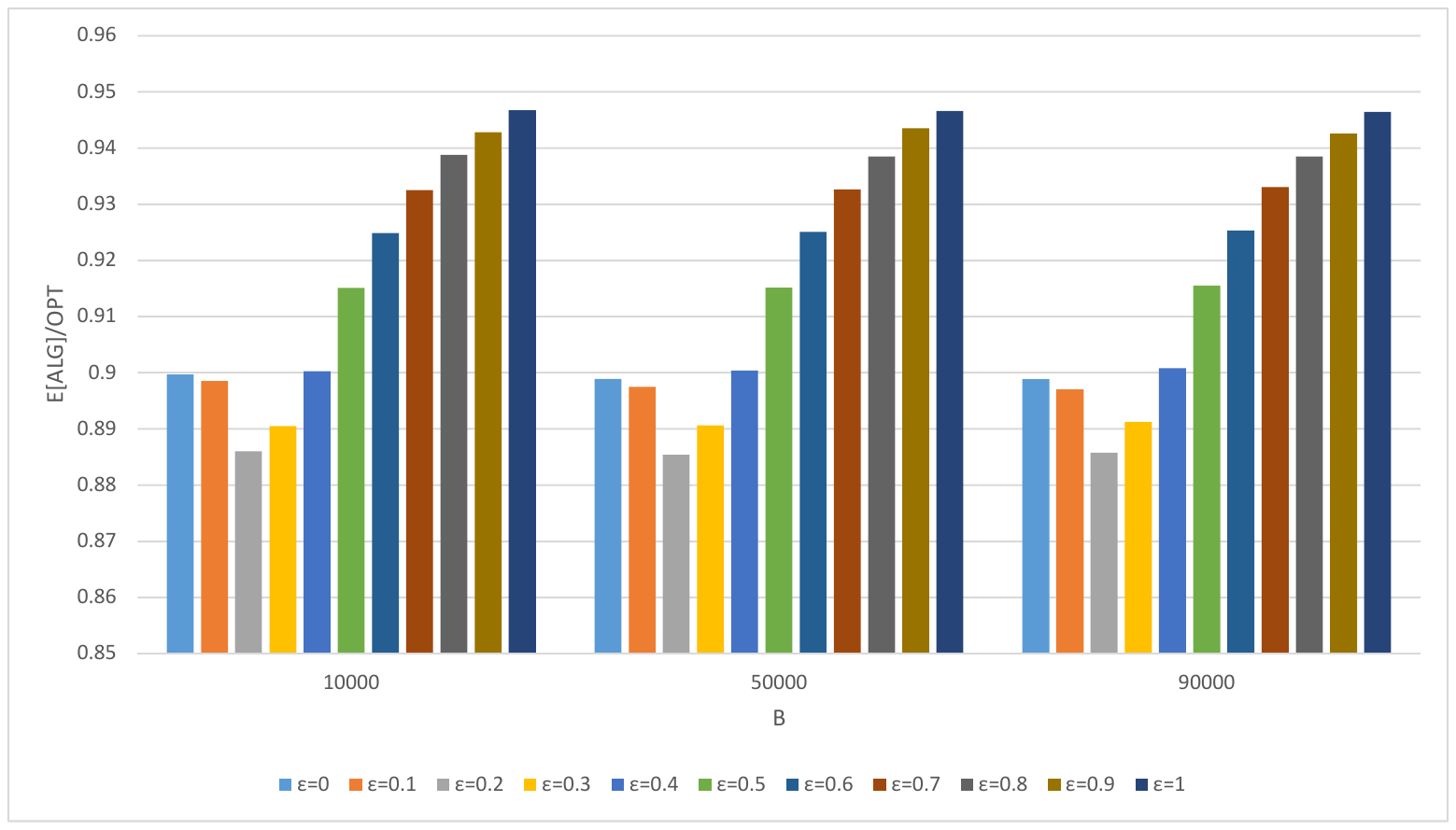}
\caption{Performance ratios of algorithms using UCB and LazyUCB oracles. $n=5$. $b_i = B$ for all resources $i \in [n]$. $T = 100000$.}
\label{fig:lazy_new1}
\end{center}
\end{figure}

\begin{figure}[h]
\begin{center}
\includegraphics[scale=0.4,trim={1cm 3cm 1cm 3cm},clip]{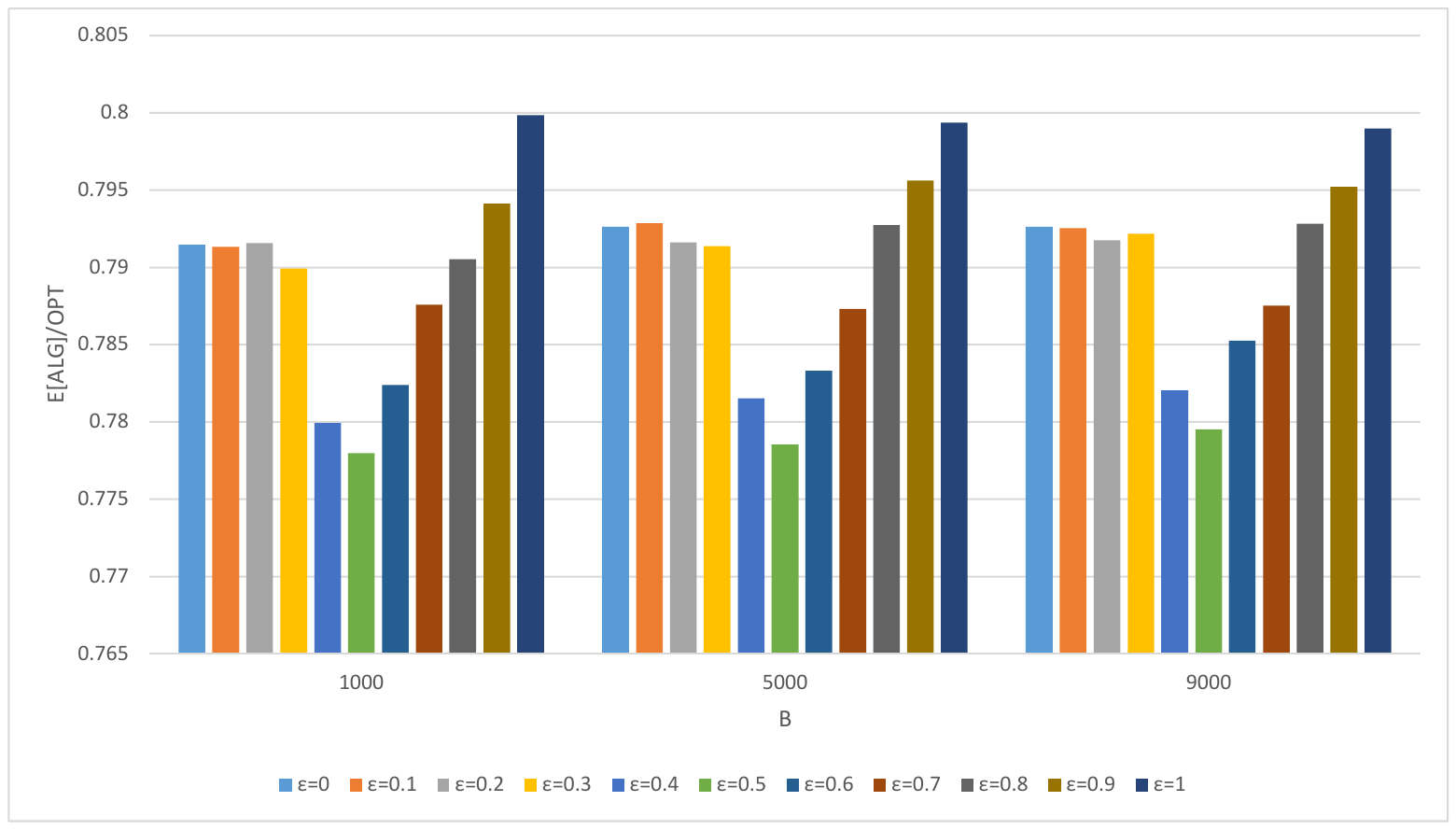}
\caption{Performance ratios of algorithms using UCB and LazyUCB oracles. $n=50$. $b_i = B$ for all resources $i \in [n]$. $T = 100000$.}
\label{fig:lazy_new4}
\end{center}
\end{figure}

\begin{figure}[h]
\begin{center}
\includegraphics[scale=0.4,trim={1cm 3cm 1cm 3cm},clip]{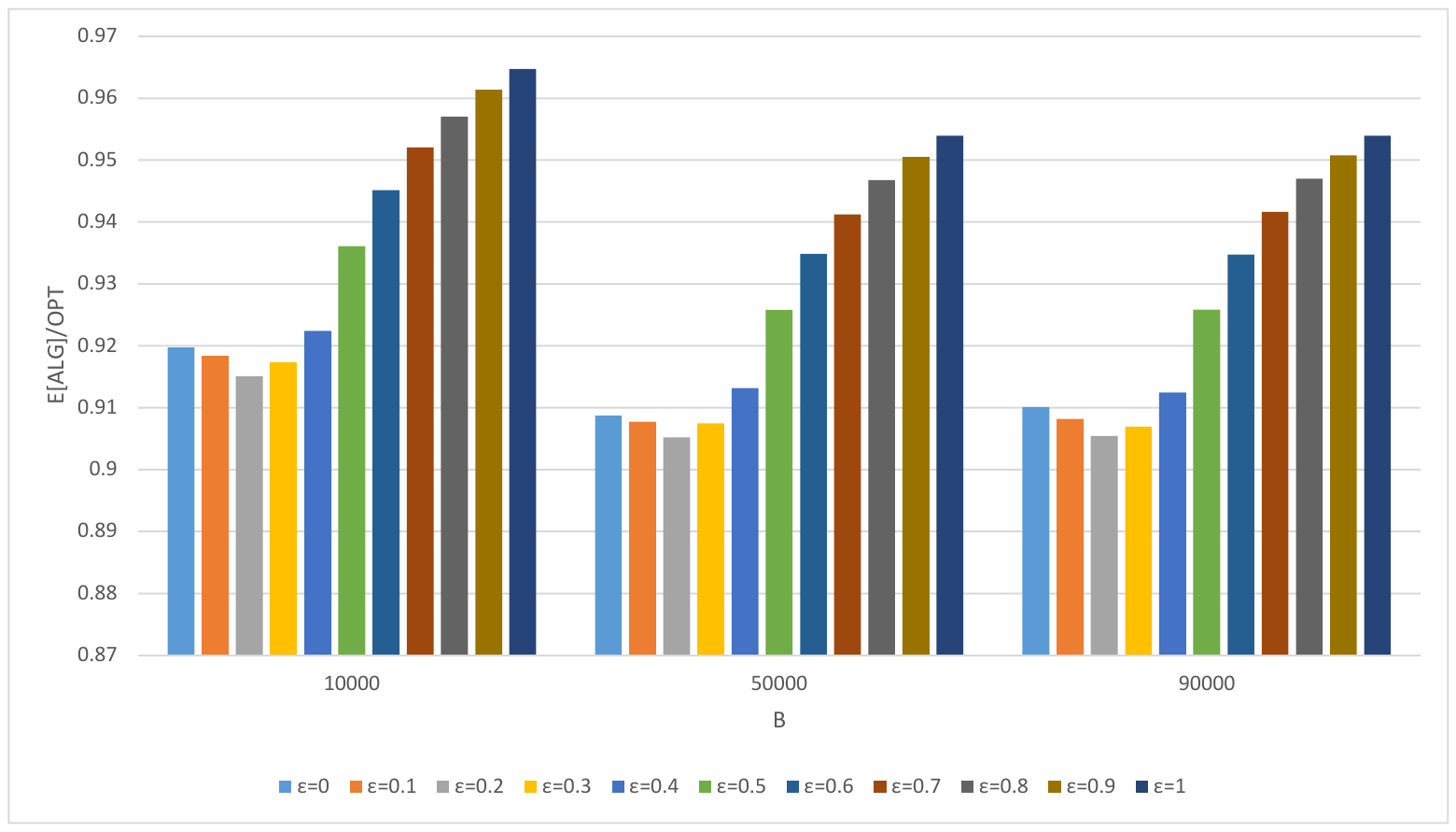}
\caption{Performance ratios of algorithms using UCB and LazyUCB oracles. $n=50$. $b_i = B$ for all resources $i \in [n]$. $T = 1000000$.}
\label{fig:lazy_new2}
\end{center}
\end{figure}

\textbf{Discussion of results.}
The empirical performance is consistently best when $\epsilon$ takes its maximum value of 1, i.e.\ when LazyUCB does the least exploration.
This is consistent with recent findings \citep{BastaniBK17,KannanMRWW18} that reducing forced exploration in contextual bandit settings will generally improve practical performance, despite not having a better worst-case guarantee.
Interestingly, in our setting there is a drop in performance for $\epsilon$ between 0 and 1, because it is better to either fully explore (which is optimal if the arrivals $x^t$ were to continue on indefinitely) or minimize exploration (which is optimal if the arrivals $x^t$ were to suddenly end).

In our graphs, all performances are worse for higher $n$ because there is more uncertainty in the resource adjacencies, and there are more unknown probabilities to learn.
On the other hand, all performances are better for higher $T$ because there is more time to learn the unknown probabilities.
The dependence on the capacities $B$ varies based on its relation to $n$ and $T$, but generally the performances are worse for higher $B$.
This is because when $B$ is small, it is less punishing to waste attempts on low-probability arms, since the capacity is the bottleneck and most of it will end up being exhausted anyway.

\subsection{Experiments on Real-World Data} \label{sec:numerical1}

We conduct numerical experiments using dataset Hotel 1 of \cite{BFG09}. Our numerical setting is a dynamic assortment planning problem, similar to that in \cite{MSL17}, but we consider their extension in which customer purchase probabilities are not observable.

We focus on a dynamic assortment planning problem where each room could be sold at multiple different prices.  Our results can be extended to this setting (see Appendix \ref{sec:multipleRates}). We consider a hotel with $n=4$ room categories: King rooms, Queen rooms, Suites, and Two-double rooms. Each room category is a resource, indexed by $i=1,2,3,4$. The inventory level of each of these resources is the number of available rooms in the corresponding category.

Rooms of each category $i$ can be offered at two prices $\cP_i = \{P_{i,1}, P_{i,2}\}$, for $i=1,2,3,4$. Each of the $m=8$ combinations, indexed by $j=1,2,...,8$, of room category and price is a product. Table \ref{tab:pricesNumerical} summarizes the prices of all the $m=8$ products from the data set. In the experiments, we double the higher price $P_{i,2}$ of each room category $i$ in order to differentiate the performance of different algorithms.

\begin{table}
\caption{Prices of 8 products from the dataset.}
\label{tab:pricesNumerical}
\centering
\begin{tabular}{|c|c|c|}
\hline
Category & $P_{i,1}$ & $P_{i,2}$\\
\hline
\hline
King & 307 & 361\\
\hline
Queen & 304& 361\\
\hline
Suites & 384& 496\\
\hline
Two-double & 306& 342\\
\hline
\end{tabular}
\end{table}

Each customer has a feature (context) vector $x \in \cX \subseteq \bR^9$. $x_1 = 1$ is a constant feature. Features $x_2,...,x_9$ represent the customer's personal information, such as the party size and the VIP level. (See Appendix \ref{app:supp_details_num} for a more detailed discussion on feature selection.) Each product $j \in \{1,2,...,8\}$ has a latent vector $\beta_j^* \in \bR^9$.  We assume that customers follow the MNL choice model. For each customer $x \in \cX$, the personalized attraction value of product  $j$ is $e^{x^\top \beta_j^*}$. The action set $\cA$ consists of all the possible assortments formed by the 8 products. When assortment $a \subseteq \{1,2,...,8\}$ is offered to customer $x \in \cX$, the customer will purchase product $j \in a$ with probability
\[ \frac{e^{x^\top \beta_j^*}}{ v_0 + \sum_{j' \in a} e^{x^\top \beta_{j'}^*}}, \]
where $v_0$ is the attraction value for the no-purchase option.  We vary $v_0$ in the experiments.

We consider a Bayesian environment. The prior distribution for each $\beta_j^*$, $j \in \{1,2,...,8\}$, is generated as follows.  First, calculate the maximum likelihood estimator $\bar \beta_j$ for $\beta_j^*$ from all the transactions in the dataset. Then, we assume that each element $\beta_{j,k}^*$, for $k=1,2,...,9$, of $\beta_j^*$ is an independent uniform random variable over $[\bar \beta_{j,k} -\epsilon, \bar \beta_{j,k} + \epsilon]$. We vary the uncertainty level $\epsilon$ in the tests. $\epsilon = 0$ corresponds to the model of \cite{MSL17}, in which the algorithms know the true values of $\beta_j^*$.

   This numerical setting essentially follows \cite{CSC17} except that we impose inventory constraints here. The Thompson sampling algorithm in \cite{CSC17} solves the auxiliary problem of this setting.
\begin{proposition}[\cite{CSC17}]
Suppose that $\beta = (\beta_1,\beta_2,...,\beta_8)$ is drawn from a known prior distribution $\pi_0$. For the auxiliary problem, there is a Thompson sampling algorithm with Bayesian regret
\[ \bE_{\beta\sim \pi_0}[\REG(\cF_T)] = \bE_{\beta\sim \pi_0}[ \bE[\REG(\cF_T)| \beta] ] =  \tilde O(Dm\sqrt{BT}).\]
\end{proposition}
In our numerical model, $D=9$ is the length of feature vectors, $m=8$ is the number of products, $B=8$ is the maximum size of any assortment, and $T =231$ is the number of customers.

Applying this Thompson sampling algorithm to our framework, and letting $\bmin \to \infty$,  we can obtain the following performance guarantee by Theorem \ref{thm:multiPrice}
\[ \bE_{\beta\sim \pi_0}[\OPT] \le \frac{1}{1-\exp(-\min_{i\in [n]}\alpha^{(1)}_i)} \cdot \bE_{\beta\sim \pi_0}[\ALG]+ \tilde O(Dm\sqrt{BT}).\]
Based on the prices in Table \ref{tab:pricesNumerical}, we can easily calculate $1-\exp(-\min_{i\in [n]}\alpha^{(1)}_i) \approx 0.58$. For details of the calculation, we refer to \citet{MSL17}.


For each test case, we simulate 500 replicates and report the average performance of each algorithm. For each replicate, we uniformly draw a sample path of customer arrivals, i.e., a sequence of feature vectors, from $31$ different instances constructed in \cite{MSL17}. Each sample path contains about 200 customers. For each replicate, we also randomly draw the latent vectors $\beta_j^*$ for all products $j \in \{1,2,...,8\}$ from their prior distributions.

We compare the following algorithms
\begin{itemize}
\item IB-TS: the inventory-balancing algorithm generated by our framework using the Thompson sampling algorithm in \cite{CSC17} as the oracles.
\item Gdy-TS: same as  IB-TS but the framework uses the original reward values, instead of the virtual rewards, as the input for the oracles.
\item Conserv-TS: same as  IB-TS but the algorithm assumes that there are only 4 higher-price products, i.e., products with prices $P_{\cdot, 2}$.
\end{itemize}

Tables \ref{tab:numerical1} to \ref{tab:numerical5} report the performance of these algorithms under different test parameters. In particular, the first column of each table is a parameter that scales the initial inventory levels of all the four resources.  In general, Gdy-TS performs better when inventory is more abundant. This is because the greedy algorithm is the optimal algorithm when there is no need to reserve resources. On the other hand, Conserv-TS has better performance when inventory is more scarce. This is because there is no need to sell resources at lower prices when we can sell all of them. Overall, our IB-TS algorithm performs much better when total inventory is close to total demand.

\section{Conclusion}
We study a general class of resource allocation problems, which involve both uncertainty on the contextual information of each customer, as well as on the functional relationship between a customer's contextual information to their behavior. We propose the Inventory Balancing with Online Learning (IBOL) algorithm that handles both sources of uncertainty simultaneously. In addition, we harness existing tools from the online learning literature to construct the Upper Confidence Bound (UCB) oracle, and we also design a new LazyUCB oracle that conducts substantially less exploration and more exploitation than the LazyUCB oracle. The performance guarantees of our algorithms are shown to be near optimal, and they are corroborated by numerical experiments on both synthetic and actual datasets.

To finish off, we would like to discuss the benefit of describing our resource allocation problem using generic ``actions'', especially in the context of the Inventory Balancing literature.  Previously, the most general description of an Inventory Balancing algorithm under adversarial arrivals was that of offering an \textit{assortment} of multiple resources, introduced by \citet{GNR14}.  However, our treatment allows for even more general actions, such as offering a \textit{sequence} of resources to each online customer, as in the online matching with timeouts problem \citep{bansal2012lp}.  Our Theorems~\ref{thm:main} and~\ref{thm:multiPrice} directly imply that an online algorithm can be $1/2$-competitive in general, and $(1-1/e)$-competitive as resource capacities approach $\infty$, in the online vertex-weighted matching with timeouts problem of \citet{bansal2012lp}, in which matching probabilities are known exactly (i.e. the regret from learning is 0).  This further demonstrates the benefit of our unified and generic framework for online resource allocation.

\begin{table}
\centering
\caption{Performance of algorithms relative to $\OPT$. $v_0 = 5$, $\epsilon = 1$. }\label{tab:numerical1}
\begin{tabular}{|c|c|c|c|}
\hline
Inventory scale &  IB-TS & Gdy-TS & Conserv-TS\\
\hline
$	0.1	$&$	93.6\%	$&$	90.3\%	$&$	99.3\%	$\\
$	0.15	$&$	95.5\%	$&$	90.7\%	$&$	98.2\%	$\\
$	0.2	$&$	95.7\%	$&$	90.8\%	$&$	98.0\%	$\\
$	0.25	$&$	96.1\%	$&$	91.4\%	$&$	97.0\%	$\\
$	0.3	$&$	95.8\%	$&$	92.1\%	$&$	96.1\%	$\\
$	0.35	$&$	95.1\%	$&$	92.6\%	$&$	96.0\%	$\\
$	0.4	$&$	94.2\%	$&$	92.5\%	$&$	95.1\%	$\\
$	0.45	$&$	93.5\%	$&$	92.8\%	$&$	94.5\%	$\\
$	0.5	$&$	93.2\%	$&$	93.2\%	$&$	94.2\%	$\\
$	0.55	$&$	91.5\%	$&$	92.9\%	$&$	92.9\%	$\\
$	0.6	$&$	91.0\%	$&$	92.9\%	$&$	93.2\%	$\\
\hline
\end{tabular}
\end{table}

\begin{table}
\centering
\caption{Performance of algorithms relative to $\OPT$. $v_0 = 40$, $\epsilon = 1$. }\label{tab:numerical2}
\begin{tabular}{|c|c|c|c|}
\hline
Inventory scale & IB-TS & Gdy-TS & Conserv-TS\\
\hline
$	0.1	$&$	87.7\%	$&$	84.1\%	$&$	92.8\%	$\\
$	0.15	$&$	90.4\%	$&$	85.5\%	$&$	91.4\%	$\\
$	0.2	$&$	91.8\%	$&$	87.2\%	$&$	89.3\%	$\\
$	0.25	$&$	91.5\%	$&$	87.5\%	$&$	88.3\%	$\\
$	0.3	$&$	92.0\%	$&$	88.4\%	$&$	87.8\%	$\\
$	0.35	$&$	91.6\%	$&$	89.2\%	$&$	86.5\%	$\\
$	0.4	$&$	91.3\%	$&$	89.0\%	$&$	86.4\%	$\\
$	0.45	$&$	91.4\%	$&$	89.8\%	$&$	86.1\%	$\\
$	0.5	$&$	92.8\%	$&$	90.8\%	$&$	86.5\%	$\\
$	0.55	$&$	91.8\%	$&$	90.1\%	$&$	86.1\%	$\\
$	0.6	$&$	92.2\%	$&$	90.7\%	$&$	86.7\%	$\\
\hline
\end{tabular}
\end{table}

\begin{table}
\centering
\caption{Performance of algorithms relative to $\OPT$. $v_0 = 100$, $\epsilon = 1$. }\label{tab:numerical3}
\begin{tabular}{|c|c|c|c|}
\hline
Inventory scale & IB-TS & Gdy-TS & Conserv-TS\\
\hline
$	0.1	$&$	87.1\%	$&$	86.2\%	$&$	87.7\%	$\\
$	0.15	$&$	89.9\%	$&$	87.6\%	$&$	86.4\%	$\\
$	0.2	$&$	90.5\%	$&$	88.0\%	$&$	86.2\%	$\\
$	0.25	$&$	91.9\%	$&$	90.2\%	$&$	85.6\%	$\\
$	0.3	$&$	91.7\%	$&$	90.2\%	$&$	84.3\%	$\\
$	0.35	$&$	91.5\%	$&$	91.1\%	$&$	84.3\%	$\\
$	0.4	$&$	92.1\%	$&$	90.8\%	$&$	83.5\%	$\\
$	0.45	$&$	92.4\%	$&$	91.5\%	$&$	84.7\%	$\\
$	0.5	$&$	93.3\%	$&$	91.4\%	$&$	85.1\%	$\\
$	0.55	$&$	93.2\%	$&$	92.2\%	$&$	84.7\%	$\\
$	0.6	$&$	92.4\%	$&$	92.6\%	$&$	84.2\%	$\\
\hline
\end{tabular}
\end{table}

\begin{table}
\centering
\caption{Performance of algorithms relative to $\OPT$. $v_0 = 40$, $\epsilon = 0.01$. }\label{tab:numerical4}
\begin{tabular}{|c|c|c|c|}
\hline
Inventory scale & IB-TS & Gdy-TS & Conserv-TS\\
\hline
$	0.1	$&$	93.3\%	$&$	91.9\%	$&$	99.2\%	$\\
$	0.15	$&$	93.5\%	$&$	89.9\%	$&$	97.4\%	$\\
$	0.2	$&$	92.7\%	$&$	88.5\%	$&$	95.3\%	$\\
$	0.25	$&$	93.2\%	$&$	89.6\%	$&$	93.5\%	$\\
$	0.3	$&$	92.9\%	$&$	91.1\%	$&$	92.7\%	$\\
$	0.35	$&$	94.9\%	$&$	95.0\%	$&$	92.4\%	$\\
$	0.4	$&$	96.4\%	$&$	95.7\%	$&$	93.1\%	$\\
$	0.45	$&$	96.8\%	$&$	97.4\%	$&$	93.7\%	$\\
$	0.5	$&$	98.4\%	$&$	98.4\%	$&$	95.2\%	$\\
$	0.55	$&$	98.3\%	$&$	99.6\%	$&$	95.2\%	$\\
$	0.6	$&$	97.9\%	$&$	99.0\%	$&$	95.0\%	$\\
\hline
\end{tabular}
\end{table}

\begin{table}
\centering
\caption{Performance of algorithms relative to $\OPT$. $v_0 = 40$, $\epsilon = 5$. }\label{tab:numerical5}
\begin{tabular}{|c|c|c|c|}
\hline
Inventory scale & IB-TS & Gdy-TS & Conserv-TS\\
\hline
$	0.1	$&$	84.8\%	$&$	82.0\%	$&$	91.4\%	$\\
$	0.15	$&$	87.5\%	$&$	84.2\%	$&$	89.9\%	$\\
$	0.2	$&$	88.5\%	$&$	83.7\%	$&$	89.6\%	$\\
$	0.25	$&$	88.9\%	$&$	84.3\%	$&$	88.5\%	$\\
$	0.3	$&$	89.3\%	$&$	84.7\%	$&$	87.7\%	$\\
$	0.35	$&$	89.4\%	$&$	86.1\%	$&$	86.3\%	$\\
$	0.4	$&$	89.9\%	$&$	86.2\%	$&$	86.0\%	$\\
$	0.45	$&$	89.1\%	$&$	85.3\%	$&$	85.4\%	$\\
$	0.5	$&$	88.9\%	$&$	85.8\%	$&$	84.6\%	$\\
$	0.55	$&$	88.6\%	$&$	85.3\%	$&$	84.5\%	$\\
$	0.6	$&$	88.6\%	$&$	85.6\%	$&$	84.7\%	$\\
\hline
\end{tabular}
\end{table}


\bibliographystyle{ormsv080}
\bibliography{bibliography}
\newpage
\begin{APPENDICES}
\section{Proofs for Section \ref{sec:analysis}}
\subsection{Proof of Lemma \ref{lemma:benchmark}}\label{app:pf_lemma_benchmark}
For an online algorithm, let's denote $\gamma_{a, t}$ as the probability that the algorithm chooses  action $a$ at time $t$. We first claim that $\{ \gamma_{a, t} \}_{a\in {\cal A}, t\in [T]}$ is a feasible solution to the LP \textbf{Primal}. Indeed, for each $t$, $\{ \gamma_{a, t} \}_{a\in {\cal A}}$ forms a probability distribution over the action set ${\cal A}$, therefore the constraints (\ref{eq:primal_2}, \ref{eq:primal_3}) are satisfied. To check the constraints (\ref{eq:primal_1}), we observe that $\sum^T_{t=1} \vy^t_i \leq b_i$ for all $i\in [n]$ with certainty, since we assume that the online algorithm is feasible. In particular, we also have $\mathbb{E}\left[\sum^T_{t=1} \vy^t_i\right] \leq b_i$ for all $i\in [n]$. Observe that we have
$$\mathbb{E}\left[ \vy^t_i\right] = \sum_{a\in {\cal A}}\Pr( y^t_i = 1  | a^t = a) \Pr(a^t = a) = \sum_{a\in {\cal A}} p_{x^t, a, i} \gamma_{a, t} $$
for all $i\in [n]$, by our model definition of $p_{x^t, a, i} $. Altogether, the constraints (\ref{eq:primal_1}) are also satisfied, which shows that $\{ \gamma_{a, t} \}_{a\in {\cal A}, t\in [T]}$ is feasible to the LP \textbf{Primal}.

To finish the proof, observe that the expected total reward is equal to the objective value:
\begin{align}
&\mathbb{E}\left[ \sum_{i\in [n]} r_i  \sum_{t\in [T]} \vy^t_i \cdot \mathbf{1}(N^{t-1}_i < b_i)\right]\nonumber\\
= & \mathbb{E}\left[ \sum_{i\in [n]} r_i  \sum_{t\in [T]} \vy^t_i \right] \label{eq:pf_lemma_benchmark_1}\\
= & \sum_{i\in [n]} r_i  \sum_{a\in {\cal A}} p_{x^t, a, i} \gamma_{a, t}\label{eq:pf_lemma_benchmark_2}.
\end{align}
Step (\ref{eq:pf_lemma_benchmark_1}) is by the Lemma's assumption that the online algorithm is feasible. Observing that (\ref{eq:pf_lemma_benchmark_2}) is the objective value of the LP \textbf{Primal} under the solution $\{ \gamma_{a, t} \}_{a\in {\cal A}, t\in [T]}$, we have altogether shown that $\{ \gamma_{a, t} \}_{a\in {\cal A}, t\in [T]}$  is feasible to the LP, and the expected total reward under the algorithm is at most $\OPT$.
$\Halmos$
\subsection{Proof of Theorem \ref{thm:main}}\label{app:pf_thm_main}

Throughout the proof, we fix $\epsilon \geq 0$ as a constant.
We demonstrate the Theorem by showing the following inequality:
\begin{equation}\label{eq:guarantee_rewrite}
\OPT \le \frac{e}{e - 1}f_{\Psi}(b_\text{min}, \epsilon) \bE[\ALG] + \bE \left[\sum^T_{t=1}R^t(a^t_*) - (1 +\epsilon ) R^t(a^t)\right],
\end{equation}

The proof of Theorem \ref{thm:main} begins by considering a dual formulation of the LP \textbf{Primal}:
\begin{align}
\textbf{Dual: } \min    & \sum_{t \in [T]} \gamma_t + \sum_{i \in [n]}b_i \lambda_i \label{eq:dual_obj}\\
\text{s.t. }   & \gamma_t \geq  \sum_{i \in [n]}  p_{x^t, a, i} (r_i - \lambda_i)   &\quad &a \in \cA, t \in [T]  \label{eq:dual_1}\\
p_{x^t, a, i}
& \lambda_i \geq 0 &\quad &\forall i\in [n]. \label{eq:dual_2}
\end{align}
We prove the performance guarantee using a primal dual approach. More precisely, we construct a solution $(\Lambda, \Gamma)$ feasible to \textbf{Dual}, where  $(\Lambda, \Gamma)$ are constructed based on the dynamics of Algorithm \ref{alg:ibol}. Then, we relate the algorithm's performace to the expected value of the solution $(\Lambda, \Gamma)$  under objective (\ref{eq:dual_obj}), which upper bounds the benchmark by the linear duality.

We define the solution $\Lambda = (\Lambda_i)_{i\in [n]}, \Gamma = (\Gamma_t)_{t\in [T]}$ as
\begin{align}
\Lambda_i &=r_i\cdot\Psi\left(\frac{N_i^{T}}{b_i}\right)\label{eq:Lambdai} \\
\Gamma_t &=\max_{a \in \cA}\left\{  \sum_{i \in [n]}  p_{x^t, a, i} r_i \left[ 1 - \Psi\left(\frac{N_i^{t-1}}{b_i}\right)\right] \right\} =\max_{a \in \cA}\left\{  \sum_{i \in [n]}  p_{x^t, a, i} r^t_i \right\} = R^t(a^t_*) \label{eq:Gammat}.
\end{align}
Recall that $a^t_*$ and $\Gamma^t$ are respectively an optimal action and the optimal reward at time $t$ in the auxiliary problem. 
We first claim the feasibility of $(\Lambda, \Gamma)$ to \textbf{Dual}.
\begin{claim}\label{claim:dual_feas}
For any realization of $\{N^t_i\}_{t\in [T], i\in [n]}$, the solution $(\Lambda, \Gamma)$ defined in (\ref{eq:Lambdai}, \ref{eq:Gammat}) is feasible to \textbf{Dual}. Moreover, we have $\OPT \leq \bE [\sum_{t\in [T]}\Gamma_t + \sum_{i\in [n]}b_i \Lambda_i ]$.
\end{claim}
\begin{proof}{Proof of Claim \ref{claim:dual_feas}.}
The constraints in (\ref{eq:dual_2}) are clearly satisfied by $\Lambda$, since $\Psi(x)\geq 0$ for all $x\in [0, 1]$. To verify the feasibility to the constraints (\ref{eq:dual_1}), for any $a\in \cA, t\in [T]$ we check that
\begin{align}
\Gamma_t & \geq \sum_{i \in [n]}  p_{x^t, a, i} r_i \left[ 1 - \Psi\left(\frac{N_i^{t-1}}{b_i}\right)\right] \label{eq:byGammat}\\
&\geq \sum_{i \in [n]}  p_{x^t, a, i} r_i \left[ 1 - \Psi\left(\frac{N_i^T}{b_i}\right)\right] \label{eq:bydefPsi}\\
& = \sum_{i \in [n]}  p_{x^t, a, i} (r_i  - \Lambda_i ). \label{eq:byLambdai}
\end{align}
Step (\ref{eq:byGammat}) is by the first equation (\ref{eq:Gammat}) in the definition of $\Gamma^t$, step (\ref{eq:bydefPsi}) is by the fact that $\Psi$ is an increasing function and $N^T_i \geq N^{t-1}_i$, step (\ref{eq:byLambdai}) is by the definition of $\Lambda_i$ in (\ref{eq:Lambdai}). Altogether, the Claim is proved.$\Halmos$
\end{proof}

First, we use Claim \ref{claim:dual_feas} to argue that
\begin{align}
\OPT \le & \bE\left[\sum_{t\in [T]} \Gamma_t + \sum_{i\in [n]}b_i \Lambda_i\right]\nonumber\\
= & \sum_{i \in [n]} b_i r_i\cdot \sum_{t \in [T]} \bE\left[\Psi\left(\frac{N_i^{t}}{b_i}\right)- \Psi\left(\frac{N_i^{t-1}}{b_i}\right)\right] + \bE\left[\sum_{t\in [T]} R^t(a^t_*)\right].\label{eq:main_1}
 \end{align}
Step (\ref{eq:main_1}) is by stating the first summation in (\ref{eq:main_1}) as a telescoping sum.

To proceed, recall that $\cF_t$ is the input to the auxiliary problem at the start of time $t$, which determines the values of $N^{t-1}_i$.  Conditioned on $\cF_t$, the algorithm's action $a^t= \cO^t(\cF_t)$ is determined.  Thus, for any resource $i$,
\begin{align}
\bE\left[ \Psi\left(\frac{N_i^{t}}{b_i}\right)\bigg | \cF_t\right] - \Psi\left(\frac{N_i^{t-1}}{b_i}\right) &= \sum_{\vy \in \cY} \rho_{x^t,a^t}(\vy) \vy_i \left[ \Psi\left(\frac{\min\{b_i,N_i^{t-1}+1\}}{b_i}\right) - \Psi\left(\frac{N_i^{t-1}}{b_i}\right) \right]. \label{eqn:towerProperty}
\end{align}
We explain equation~(\ref{eqn:towerProperty}). We claim that, conditioned on $\cF_t$, we have $N^t_i = \min\{b_i, N^{t-1}_i  + 1\}$. Indeed, the vector of outcomes $\vy$ is distributed according to $\rho_{x^t,a^t}$. If $\vy_i = 1$ and resource $i$ is not yet depleted, i.e. $N^{t-1}_i < b_i$, then a unit of resource $i$ is consumed, leading to $N^t_i=N^{t-1}_i+1 \leq b_i$. If $\vy_i = 1$ but resource $i$ is depleted, i.e. $N^{t-1}_i = b_i$, then resource $i$ cannot be consumed further, leading to $N^t_i=b_i$. Altogether, equation (\ref{eqn:towerProperty}) is justified.

Using, (\ref{eqn:towerProperty}) and the towering property of conditional expectation, we can express the summands in the first sum in (\ref{eq:main_1}) as
\begin{align}
&\bE\left[\Psi\left(\frac{N_i^{t}}{b_i}\right)- \Psi\left(\frac{N_i^{t-1}}{b_i}\right)\right] = \bE\left[\bE\left[\Psi\left(\frac{N_i^{t}}{b_i}\right) | {\cal F}_t\right] - \Psi\left(\frac{N_i^{t-1}}{b_i}\right)\right]  \nonumber\\
= & \bE\left[\sum_{\vy \in \cY} \rho_{x^t,a^t}(\vy) \vy_i \left[ \Psi\left(\frac{\min\{b_i,N_i^{t-1}+1\}}{b_i}\right) - \Psi\left(\frac{N_i^{t-1}}{b_i}\right) \right]\right]\nonumber\\
= & \bE\left[  p_{x^t, a^t, i}\left[ \Psi\left(\frac{\min\{b_i,N_i^{t-1}+1\}}{b_i}\right) - \Psi\left(\frac{N_i^{t-1}}{b_i}\right) \right] \right]\label{eq:main_2}
\end{align}

Next, we continue with (\ref{eq:main_1}, \ref{eq:main_2}):
\begin{align}
\OPT &\leq \sum_{i\in [n]}b_i r_i \sum_{t\in [T]} \bE\left[  p_{x^t, a^t, i}\left[ \Psi\left(\frac{\min\{b_i,N_i^{t-1}+1\}}{b_i}\right) - \Psi\left(\frac{N_i^{t-1}}{b_i}\right) \right] \right] + \bE \left[\sum_{t\in [T]} R^t(a^t_*)\right]\nonumber\\
& =\sum_{i\in [n]}  r_i \cdot \bE \left[ \sum_{t \in [T]} p_{x^t, a^t, i}\cdot\left\{ \frac{\Psi\left( \frac{ \min\{b_i,N_i^{t-1}+1\} }{ b_i}\right) - \Psi\left(\frac{N_i^{t-1}}{b_i} \right)}{1/b_i} + (1 + \epsilon)\left[1 - \Psi\left(\frac{N_i^{t-1}}{b_i}\right)\right] \right\} \right] \nonumber\\
& \qquad + \bE \left[\sum_{t\in [T]} R^t(a^t_*) - (1+\epsilon) R^t(a^t)\right]  \label{eq:main_3}
\end{align}
Step (\ref{eq:main_3}) uses the definition of $R^t(a^t) = \sum_{i\in [n]} r_ip_{x^t, a^t, i} [ 1 - \Psi( N^{t-1}_i / b_i )] $.

We now need to derive our $\epsilon$-perturbed potential function $\Psi$, and establish the following guarantee.

\begin{lemma}[Guarantee for $\epsilon$-perturbed potential function] \label{lem:epsPerturbed}
As long as the $\epsilon$-perturbed potential function $\Psi(x) = \frac{e^{(1+\epsilon)x} - 1}{e^{1+\epsilon}-1}$ is used, for any resource $i$, time $t$, and possible value of $N^{t-1}_i$ in $\{0,\ldots,b_i\}$,
\begin{equation}
\label{eq:main_4}
\frac{\Psi\left( \frac{ \min\{b_i,N_i^{t-1}+1\} }{ b_i}\right) - \Psi\left(\frac{N_i^{t-1}}{b_i} \right)}{1/b_i} + (1 + \epsilon)\left[1 - \Psi\left(\frac{N_i^{t-1}}{b_i}\right)\right] \leq \mathbf{1}(N^{t-1}_i < b_i)\frac{(b_i+1+\epsilon)(1-e^{-(1+\epsilon)/b_i})}{1-e^{-(1+\epsilon)}}.
\end{equation}
\end{lemma}

\begin{proof}{Proof.}
Consider a generic $i$ and $t$.
We omit scripts $i,t$ and let $s\in\{0,1/b,\ldots,1\}$ denote $N/b$.
Both sides of the desired inequality are 0 when $s=1$, so in the sequel we assume $s<1$.
We would like to show
\begin{align} \label{eqn:6980}
\frac{\Psi(s+1/b)-\Psi(s)}{1/b}+(1+\epsilon)(1-\Psi(s)) &\le\frac{(b+1+\epsilon)(1-e^{-(1+\epsilon)/b})}{1-e^{-(1+\epsilon)}}.
\end{align}

This difference between~\eqref{eqn:6980} and the typical constraint from primal-dual analysis \citep{BJN07} is the multiplication on the LHS by the term $1+\epsilon$.
Consequently, the RHS has also been relaxed by an expression dependent on $\epsilon$.
The LHS of~\eqref{eqn:6980} can be analyzed as follows:
\begin{align*}
&b\left(\frac{e^{(1+\epsilon)(s+1/b)} - 1}{e^{1+\epsilon}-1}-\frac{e^{(1+\epsilon)s} - 1}{e^{1+\epsilon}-1}\right)+(1+\epsilon)\left(1-\frac{e^{(1+\epsilon)s} - 1}{e^{1+\epsilon}-1}\right)
\\ &=\frac{be^{(1+\epsilon)(s+1/b)}-be^{(1+\epsilon)s}+(1+\epsilon)e^{1+\epsilon}-(1+\epsilon)e^{(1+\epsilon)s}}{e^{1+\epsilon}-1}
\\ &=\frac{
e^{(1+\epsilon)s}(be^{(1+\epsilon)/b}-b-1-\epsilon)
+(1+\epsilon)e^{1+\epsilon}
}{e^{1+\epsilon}-1}
\\ &\le\frac{
e^{(1+\epsilon)(1-1/b)}(be^{(1+\epsilon)/b}-b-1-\epsilon)
+(1+\epsilon)e^{1+\epsilon}
}{e^{1+\epsilon}-1}
\\ &=\frac{
be^{1+\epsilon}(1-e^{-(1+\epsilon)/b})
+(1+\epsilon)e^{1+\epsilon}(1-e^{-(1+\epsilon)/b})
}{e^{1+\epsilon}-1}
\\ &=\frac{(b+1+\epsilon)(1-e^{-(1+\epsilon)/b})}{1-e^{-(1+\epsilon)}}
\end{align*}
where the inequality holds because the maximum possible value of $s$ is $1-1/b$ and the expression $be^{(1+\epsilon)/b}-b-1-\epsilon$ is non-negative.
This completes the proof of Lemma~\ref{lem:epsPerturbed}.
\halmos
\end{proof}

Substituting the result from Lemma~\ref{lem:epsPerturbed} back into~\eqref{eq:main_3}, we get
\begin{align*}
\OPT
&\le\sum_{i\in [n]}  r_i \cdot \bE \left[ \sum_{t \in [T]} p_{x^t, a^t, i}\cdot\mathbf{1}(N^{t-1}_i < b_i)\frac{(b_i+1+\epsilon)(1-e^{-(1+\epsilon)/b_i})}{1-e^{-(1+\epsilon)}} \right] + \bE \left[\sum_{t\in [T]} R^t(a^t_*) - (1+\epsilon) R^t(a^t)\right].
\end{align*}
The RHS can be upper-bounded by $\bE[\ALG]\max_i\frac{(b_i+1+\epsilon)(1-e^{-(1+\epsilon)/b_i})}{1-e^{-(1+\epsilon)}} + \bE [\sum_{t\in [T]} R^t(a^t_*) - (1+\epsilon) R^t(a^t)]$, by definition of $\bE[\ALG]$.
This implies
\begin{align*}
\bE[\ALG]
&\ge\left(\OPT-\bE\left[\sum_{t\in [T]} R^t(a^t_*) - (1+\epsilon) R^t(a^t)\right]\right)\min_i\frac{1-e^{-(1+\epsilon)}}{(b_i+1+\epsilon)(1-e^{-(1+\epsilon)/b_i})}.
\end{align*}

\section{Concentration Inequalities and Their Proofs}
\subsection{Lemma \ref{lem:UCB_0_conc} for UCB and its Proof}
\begin{lemma}[\cite{KSU08}]\label{lem:UCB_0_conc}
Let $Y_1, \ldots, Y_M$ be i.i.d. Bernoulli random variables with mean $p$. Denote $\bar{Y} = \frac{1}{M} \sum^M_{i=1}Y_i$. For any $0 < \delta < 1$, it holds that
\begin{equation*}
\Pr \left( \left| \bar{Y} - p \right|  \leq  \mathsf{rad}(\bar{Y}, M; \delta) \leq 3\cdot \mathsf{rad}(p, M; \delta)\right) \geq 1 - 4\delta.
\end{equation*}
\end{lemma}

We first recall that for $p>0, M\in \mathbb{Z}_{\geq 0}, \delta\in (0, 1) $, we have defined
$$
\mathsf{rad}(p, M; \delta) = \sqrt{\frac{2p\log(1/\delta)}{ \max\{M, 1\} }} + \frac{3\log(1/\delta)}{ \max\{M, 1\}}.
$$
We prove Lemma \ref{lem:UCB_0_conc} with the following two concentration inequalities.
\begin{theorem}[Theorem 1 in \citep{AudibertMS09}]\label{thm:Audibert}
Let $Y_1, \ldots, Y_M$ be i.i.d. Bernoulli random variables with mean $p$. Denote $\bar{Y} = \frac{1}{M} \sum^M_{i=1}Y_i$. For any $0 < \delta < 1$, it holds that
$$
\Pr \left( \left| \bar{Y} - p \right|  \leq \mathsf{rad}(\bar{Y}, M; \delta) \right) \geq 1- 3\delta.
$$
\end{theorem}
\begin{theorem}[Theorem 4.4 (3) in \citep{MitzenmacherU05}]\label{thm:superlighttail}
Let $Y_1, \ldots, Y_M$ be i.i.d. Bernoulli random variables with mean $p$. Denote $\bar{Y} = \frac{1}{M} \sum^M_{i=1}Y_i$. For any $R > 6 p$, it holds that
$$
\Pr(\bar{Y} > R) \leq 2^{-nR}.
$$
\end{theorem}

\begin{proof}{Proof of Lemma \ref{lem:UCB_0_conc} }
We have
\begin{align}
& \Pr \left( \left| \bar{Y} - p \right|  \leq \mathsf{rad}(\bar{Y}, M; \delta) \leq 3\cdot \mathsf{rad}(p, M; \delta)\right) \nonumber\\
\geq & 1 -  \Pr \left( \left| \bar{Y} - p \right|  \geq \mathsf{rad}(\bar{Y}, M; \delta) \right)   -\Pr\left( \mathsf{rad}(\bar{Y}, M; \delta)  \geq 3\cdot \mathsf{rad}(p, M; \delta)\right)\label{eq:pf_UCB_0_conc_0}.
\end{align}
By Theorem \ref{thm:Audibert}, we know that
\begin{equation}\label{eq:pf_UCB_0_conc_1}
\Pr \left( \left| \bar{Y} - p \right|  \geq    \mathsf{rad}(\bar{Y}, M; \delta) \right) \leq 3\delta.
\end{equation}
Next, we have
\begin{align}
&\Pr\left(\mathsf{rad}(\bar{Y}, M; \delta)  \geq 3\cdot \mathsf{rad}(p, M; \delta)\right)\nonumber\\
= &\Pr\left(  \sqrt{\frac{2\bar{Y}\log(1/\delta)}{ M }} + \frac{3\log(1/\delta)}{M} \geq 3\left( \sqrt{\frac{2p\log(1/\delta)}{ M }} + \frac{3\log(1/\delta)}{M} \right)\right) \nonumber\\
= & \Pr\left( \frac{2\bar{Y}\log(1/\delta)}{M} > \frac{18 p \log(1/\delta)}{M} + 36 \sqrt{\frac{2 p \log(1/\delta)}{M}} \cdot \frac{\log(1/\delta)}{M} + \frac{36\log^2(1/\delta)}{M^2}\right)\nonumber\\
\leq &  \Pr\left( \bar{Y} > 9p + \frac{18 \log(1/\delta)}{M} \right) \nonumber \\
\leq & 2^{-9pM -18 \log(1/\delta)} \label{eq:pf_UCB_0_conc_2}\\
< &\delta \label{eq:pf_UCB_0_conc_3}.
\end{align}
Step (\ref{eq:pf_UCB_0_conc_2}) is by applying Theorem \ref{thm:superlighttail}. Combining bounds (\ref{eq:pf_UCB_0_conc_1}, \ref{eq:pf_UCB_0_conc_3}) and applying to (\ref{eq:pf_UCB_0_conc_0}), the Lemma is proved. \Halmos
\end{proof}

\subsection{Proof of Lemma \ref{lemma:bd_UCB_1} for the UCB Oracle for Application 1}\label{sec:pf_lem_bd_UCB_1}
While Lemma \ref{lemma:bd_UCB_1} is first discovered by \citep{KSU08}, the explicit constants in their confidence radii are not expressed explicitly, and they are instead hidden in $O(\cdot)$. We re-derive Lemma \ref{lemma:bd_UCB_1} in order to uncover those constants.
Lemma \ref{lemma:bd_UCB_1} is proved by a direct application of Lemma \ref{lem:UCB_0_conc} and the union bound.

\begin{proof}{Proof of Lemma \ref{lemma:bd_UCB_1}}
 To this end, let $Y_1, \ldots, Y_{t-1}$ be i.i.d. Bernoulli random variables with mean $q_{(i, k)}$. Then
\begin{align}
&\Pr({\cal E}^t_{(i, k)})\nonumber\\
= & \Pr\left( \left|\bar{q}^t_{(i, k)}  - q_{(i, k)} \right| \leq \mathsf{rad}(\bar{q}^t_{(i, k)}, M^t_{(i, k)}, \delta_t ) \leq 3\cdot \mathsf{rad}(q_{(i, k)}, M^t_{(i, k)}, \delta_t) \right)\nonumber\\
\geq & \Pr\left( \left|\frac{1}{m}\sum^m_{s = 1}Y_s  - q_{(i, k)} \right| \leq \mathsf{rad}(\frac{1}{m}\sum^m_{s = 1}Y_s ,m, \delta_t ) \leq 3\cdot \mathsf{rad}(q_{(i, k)}, m, \delta_t) \text{ for all $m\in \{0, \ldots, t-1\}$}\right)\label{eq:pf_lemma_bd_UCB_1_1}\\
\geq & 1 - 4t\delta_t\label{eq:pf_lemma_bd_UCB_1_2}\\
\geq & 1 - \frac{4}{1 + t} \nonumber.
\end{align}
Step (\ref{eq:pf_lemma_bd_UCB_1_1}) is by a union bound over the possible values of $M^t_{(i, k)}\in\{0, 1, \ldots, t-1\}$. Step (\ref{eq:pf_lemma_bd_UCB_1_2}) is by the application of Lemma \ref{lem:UCB_0_conc}. Finally, by a union bound over all $i\in [n], k\in [K]$, the Lemma is proved.
\Halmos
\end{proof}

\subsection{Lemma \ref{lem:lazy_UCB} for LazyUCB and its Proof}\label{app:pf_lem_lazy_UCB}

\begin{lemma}\label{lem:lazy_UCB}
Let $Y_1, \ldots, Y_M$ be i.i.d. Bernoulli random variables with mean $p$. Denote $\bar{Y} = \frac{1}{M}\sum^M_{m=1} Y_m $. For any $\epsilon > 0$, we have
\begin{align}
\Pr\left( p \leq \left( 1 + \frac{\epsilon}{2}\right)\left[ \bar{Y} + \frac{2 + \epsilon}{\epsilon}\cdot \frac{\log(1/\delta)}{M}\right]  \right) & \geq 1 - \delta, \label{eq:lazy_UCB_lower}\\
\Pr\left( \bar{Y} \leq \left(1 + \frac{\epsilon}{2 + \epsilon}\right)\left[ p + \frac{2 + \epsilon}{\epsilon} \cdot \frac{\log(1/\delta)}{ M}\right] \right) & \geq 1 - 2\delta.\label{eq:lazy_UCB_upper}
\end{align}
\end{lemma}

The proof of the Lemma crucially uses the following concentration inequalities:

\begin{proposition}[Theorem 1 in \citep{Janson99}, Theorem 4 in \citep{ChungLu06}]\label{prop:refined_conc}
Let $Y_1, \ldots, Y_M$ be i.i.d. Bernoulli random variables with mean $p$. Denote $\bar{Y} = \frac{1}{M}\sum^M_{m=1} Y_m $. For any $\varepsilon \geq 0$, the following inequalities hold:
\begin{align}
\Pr\left( \bar{Y} \geq p + \varepsilon\right) &\leq  \exp\left[-\frac{M \varepsilon^2}{2(p + \varepsilon / 3)}\right],\label{eq:conc_upper}\\
\Pr\left( \bar{Y}  \leq p - \varepsilon\right) &\leq  \exp\left[-\frac{M \varepsilon^2}{2p}\right]\label{eq:conc_lower}.
\end{align}
\end{proposition}

We start with proving (\ref{eq:lazy_UCB_lower}). First, by unraveling (\ref{eq:conc_upper}) in Proposition \ref{prop:refined_conc}, we  deduce that
\begin{align}
\Pr\left( \bar{Y} \geq p + \varepsilon\right) & \leq  \exp\left[-\frac{M \varepsilon^2}{2(p + \varepsilon / 3)}\right]\nonumber\\
&\leq \exp\left[-\frac{M \varepsilon^2}{\max\{4p, 4\varepsilon / 3\}}\right]\nonumber\\
& \leq \exp\left[-\frac{M \varepsilon^2}{4p}\right] + \exp\left[ -\frac{3M\varepsilon}{4}\right].\nonumber
\end{align}
Now,
\begin{align}
&\Pr\left( p \geq \left( 1 + \frac{\epsilon}{2}\right)\left[ \bar{Y} + \frac{(2+\epsilon)\log(1/\delta)}{\epsilon M}\right]  \right)\nonumber\\
= & \Pr\left( \bar{Y} \leq p - \left[ \frac{\epsilon}{2 + \epsilon}\cdot  p + \frac{(2 + \epsilon) \log(1/\delta)}{\epsilon M} \right] \right)\nonumber\\
\leq & \exp\left\{ -\frac{M}{2p} \left[ \frac{\epsilon}{2 + \epsilon}\cdot  p + \frac{(2 + \epsilon) \log(1/\delta)}{\epsilon M} \right]^2\right\}\label{eq:sUCB_by_lowertail}\\
\leq & \exp\left\{ -\frac{M}{2p} \left[4 \cdot  \frac{\epsilon}{2 + \epsilon}\cdot  p \cdot \frac{(2 + \epsilon) \log(1/\delta)}{\epsilon M} \right]\right\}\label{eq:sUCB_AMGM}\\
\leq & \exp(-2\log(1/\delta)) = \delta^2 < \delta.\nonumber
\end{align}
Step (\ref{eq:sUCB_by_lowertail}) is by applying (\ref{eq:conc_lower}) from Proposition (\ref{prop:refined_conc}). Step (\ref{eq:sUCB_AMGM}) is by the inequality that $(a+b)^2 \geq 4ab$ for any $a, b\geq 0$.

To complete the proof, we prove (\ref{eq:lazy_UCB_upper}).
\begin{align}
& \Pr\left( \bar{Y} \geq \left(1 + \frac{\epsilon}{2 + \epsilon}\right)\left[ p + \frac{(2+\epsilon) \log(1/\delta)}{\epsilon M}\right] \right)\nonumber\\
= &\Pr\left( \bar{Y} \geq p + \left[\frac{\epsilon}{2  +\epsilon}\cdot p +  \left(1 + \frac{\epsilon}{2 + \epsilon}\right) \frac{(2+\epsilon) \log(1/\delta)}{\epsilon M}\right] \right)\nonumber\\
\leq & \exp\left\{ - \frac{M}{4p} \cdot  \left[\frac{\epsilon}{2  +\epsilon}\cdot p +  \left(1 + \frac{\epsilon}{2 + \epsilon}\right) \frac{(2+\epsilon) \log(1/\delta)}{\epsilon M}\right]^2 \right\} \label{eq:sUCB_by_lowertail_1}\\
&\qquad + \exp\left\{-\frac{3M}{4}\cdot \left[\frac{\epsilon}{2  +\epsilon}\cdot p +  \left(1 + \frac{\epsilon}{2 + \epsilon}\right) \frac{(2+\epsilon) \log(1/\delta)}{\epsilon M}\right] \right\}.\label{eq:sUCB_by_lowertail_2}
\end{align}
The term (\ref{eq:sUCB_by_lowertail_1}) is bounded as
\begin{align*}
&\exp\left\{ - \frac{M}{4p} \cdot  \left[\frac{\epsilon}{2  +\epsilon}\cdot p +  \left(1 + \frac{\epsilon}{2 + \epsilon}\right) \frac{(2+\epsilon) \log(1/\delta)}{\epsilon M}\right]^2 \right\}\nonumber\\
\leq & \exp\left\{ - \frac{M}{4p} \cdot \left[4 \cdot \frac{\epsilon}{2  +\epsilon}\cdot p \cdot  \left(1 + \frac{\epsilon}{2 + \epsilon}\right) \frac{(2+\epsilon) \log(1/\delta)}{\epsilon M}\right]\right\}\nonumber\\
= & \exp\left\{ -\left(1 + \frac{\epsilon}{2 + \epsilon}\right)\log(1/\delta) \right\} \leq \delta.
\end{align*}
The term (\ref{eq:sUCB_by_lowertail_2}) is bounded as
\begin{align*}
&\exp\left\{-\frac{3M}{4}\cdot \left[\frac{\epsilon}{2  +\epsilon}\cdot p +  \left(1 + \frac{\epsilon}{2 + \epsilon}\right) \frac{(2+\epsilon) \log(1/\delta)}{\epsilon M}\right] \right\} \nonumber\\
= & \exp\left\{-\frac{3Mp\epsilon}{4(2 + \epsilon)}\right\} \cdot \exp\left\{- \frac{3 + 3\epsilon}{2\epsilon} \log\frac{1}{\delta}\right\}\nonumber\\
< & \delta.
\end{align*}
Combining these bounds for (\ref{eq:sUCB_by_lowertail_1}, \ref{eq:sUCB_by_lowertail_2}), the inequality  (\ref{eq:lazy_UCB_upper}) is proved. $\Halmos$

\subsection{Proof of Lemma \ref{lemma:bd_LazyUCB_1}}\label{app:pf_lemma_bd_LazyUCB_1}
First, note that by Lemma \ref{lemma:bd_UCB_1}, we have derived that , for any $t$, the inequalities
\begin{align}
q_{(i, k)} &\leq \bar{q}^t_{(i, k)} + \mathsf{rad}(\bar{q}^t_{(i, k)}  , M^t_{(i, k)} ; \delta_t) \leq \left(1 + \frac{\epsilon}{2}\right) \left[ \bar{q}^t_{(i, k)} + \mathsf{rad}(\bar{q}^t_{(i, k)}  , M^t_{(i, k)} ; \delta_t)\right] \nonumber\\
\bar{q}^t_{(i, k)} &\leq q_{(i, k)} + \mathsf{rad}(q_{(i, k)}  , M^t_{(i, k)} ; \delta_t) \leq \left(1 + \frac{\epsilon}{2 + \epsilon}\right) \left[ q_{(i, k)} + \mathsf{rad}(q_{(i, k)}  , M^t_{(i, k)} ; \delta_t)\right] \nonumber
\end{align}
hold for each an every $i\in [n], k\in [K]$ with probability at least $1 - \frac{4nK}{1+t}$. Therefore, it suffices to show that the inequalities
\begin{align}
q_{(i, k)} &\leq\left(1 + \frac{\epsilon}{2}\right) \left[\bar{q}^t_{(i, k)} + \mathsf{lad}(\epsilon, M^t_{(i, k)} ; \delta_t)\right]  \label{eq:pf_lemma_bd_LazyUCB_1}\\
\bar{q}^t_{(i, k)} & \leq \left(1 + \frac{\epsilon}{2 + \epsilon}\right) \left[ q_{(i, k)} + \mathsf{lad}(\epsilon, M^t_{(i, k)} ; \delta_t)\right]  \label{eq:pf_lemma_bd_LazyUCB_2}
\end{align}
hold simultaneously for all $i\in [n], k\in [K]$ with probability at least $1 - \frac{3nK}{1+t}$.

To show this, we fix a pair $(i, k)$. By applying (\ref{eq:lazy_UCB_lower}, \ref{eq:lazy_UCB_upper}) with the union bout on all the possible values of $M^t_{(i, k)} \in \{0, 1, 2, \ldots, t-1\}$, we see that for the fixed pair $(i, k)$ the inequalities (\ref{eq:pf_lemma_bd_LazyUCB_1}, \ref{eq:pf_lemma_bd_LazyUCB_2}) hold with probability  at least $1 - t\delta_t, 1 - 2t\delta_t$ respective. Finally, by a union bound on all possible $i\in [n], k\in [K]$, we show that the inequalities (\ref{eq:pf_lemma_bd_LazyUCB_1}, \ref{eq:pf_lemma_bd_LazyUCB_2}) hold for all $i, k$ with probability at least $1 - 3nKt \delta_t \geq 1 - \frac{3nK}{1+t}$, hence the Lemma is proved. $\Halmos$

\section{Regret Analysis of the MAB Oracles for Application 1}
\subsection{Proof of Theorem \ref{thm:UCB_app1}, Bounding $\text{Reg}_0$ of the UCB Oracle}\label{app:pf_thm_UCB_app1}

Denote the event $\bar{\cal E}^t$ as the complement of the event ${\cal E}^t$. Denote $a^t_* = (i^t_*, k^t_*)$ as an optimal action for the auxiliary problem at time $t$.  We have 
\begin{align}
\text{Reg}_0 &= \bE \left[\sum^T_{t=1}R^t(a^t_*) -  R^t(a^t)\right] \nonumber\\
& = \sum^T_{t=1}\left\{\bE \left[\left(R^t(a^t_*) -  R^t(a^t)\right)\cdot \mathbf{1}({\cal E}^t)\right] + \bE \left[\left(R^t(a^t_*) -  R^t(a^t)\right)\cdot \mathbf{1}(\bar{\cal E}^t)\right]\right\}\nonumber\\
&\leq \sum^T_{t=1}\left\{\bE \left[\left(R^t(a^t_*) -  R^t(a^t)\right)\cdot \mathbf{1}({\cal E}^t)\right] + \frac{4nK}{1+t}\right\}\label{eq:UCB_1_step_1}.
\end{align}
Step (\ref{eq:UCB_1_step_1}) is by the model assumption that $r^t_i\in [0, 1]$ for all $i, t$.
We focus on upper bounding the first term:
\begin{align}
& \bE \left[R^t(a^t_*) \mathbf{1}({\cal E}^t )\right]\nonumber\\
\leq & \bE \left[  \mathsf{UCB}^t(a^t_*) \mathbf{1}({\cal E}^t )\right]\label{eq:UCB_1_step_2}\\
\leq & \bE \left[\mathsf{UCB}^t(a^t) \mathbf{1}({\cal E}^t )\right] \label{eq:UCB_1_step_3}\\
= & \bE \left[ r^t_{i^t} \cdot \mathbf{1}(x^t_{i^t} = 1) \cdot \bar{q}^t_{(i^t, k^t)}\cdot \mathbf{1}({\cal E}^t )\right]  + \bE \left[r^t_{i^t} \cdot \mathbf{1}(x^t_{i^t} = 1) \cdot \mathsf{rad}(\bar{q}^t_{(i^t, k^t)}, M^t_{(i^t, k^t)}, \delta_t)\cdot \mathbf{1}({\cal E}^t ) \right]\nonumber\\
\leq & \bE \left[r^t_{i^t} \cdot \mathbf{1}(x^t_{i^t} = 1) \cdot \left(q_{i^t, k^t} +  3\cdot \mathsf{rad}(q_{(i^t, k^t)}, M^t_{(i^t, k^t)}, \delta_t) \right)\right]\nonumber\\
&\quad  + \bE \left[r^t_{i^t} \cdot \mathbf{1}(x^t_{i^t} = 1) \cdot 3\mathsf{rad}(q_{(i^t, k^t)}, M^t_{(i^t, k^t)}, \delta_t)\cdot \mathbf{1}({\cal E}^t ) \right] \label{eq:UCB_1_step_4}\\
\leq & \bE \left[R^t(a^t) \mathbf{1}({\cal E}^t )\right] + 6 \bE \left[r^t_{i^t}\cdot\mathsf{rad}(q_{(i^t, k^t)}, M^t_{(i^t, k^t)}, \delta_t)\right]\label{eq:UCB_1_step_5}.
\end{align}
Step (\ref{eq:UCB_1_step_2}) is by the property of $\mathsf{UCB}$ as shown in (\ref{eq:app_1_UCB}). Step (\ref{eq:UCB_1_step_3}) is by the choice of $a^t$ in Line \ref{alg:UCB_app1_step4} in the UCB oracle Algorithm \ref{alg:UCB_app1}. Step (\ref{eq:UCB_1_step_4}) is by applying Lemma \ref{lemma:bd_UCB_1} twice. Step (\ref{eq:UCB_1_step_5}) is by the definition of $r^t_i$.

Applying the bound in (\ref{eq:UCB_1_step_5}) to the intermediate step (\ref{eq:UCB_1_step_1}), we continue bounding $\text{Reg}_0$ as follows:
\begin{align}
\text{Reg}_0 & \leq 6 \bE \left[\sum^T_{t=1}r^t_{i^t}\cdot\mathsf{rad}(q_{(i^t, k^t)}, M^t_{(i^t, k^t)}, \delta_t)\right] + 4nK\log(1+T)\nonumber\\
& = 6 \bE \left[ \sum^T_{t=1} r^t_{i^t}\cdot \left\{\sqrt{\frac{2q_{(i^t, k^t)}\log(1/\delta_t)}{ \max\{M^t_{(i^t, k^t)}, 1\} }} + \frac{3\log(1/\delta_t)}{ \max\{M^t_{(i^t, k^t)}, 1\}} \right\}\right] + 4nK\log(1+T)\nonumber\\
& \leq 6 \bE \left[ \sum^T_{t=1} \left\{\sqrt{\frac{4 r^t_{i^t} q_{(i^t, k^t)}\log(1+ T )}{ \max\{M^t_{(i^t, k^t)}, 1\} }} + \frac{6\log(1+T)}{ \max\{M^t_{(i^t, k^t)}, 1\}} \right\}\right] + 4nK\log(1+T)\label{eq:UCB_1_step_6}
\end{align}
Step (\ref{eq:UCB_1_step_6}) is by invoking the definition of $\delta_t$ and the fact that $r^t_i\in [0, 1]$.
Let's examine the two sums in the expectation.

\textbf{The first sum.}
\begin{align}
&\bE \left[ \sum^T_{t=1} \sqrt{\frac{4 r^t_{i^t} q_{(i^t, k^t)}\log(1+T )}{ \max\{M^t_{(i^t, k^t)}, 1\} }}\right]\nonumber\\
= & \bE \left[ \sum_{i\in [n], k\in [K]}\sum_{t = 1}^T \sqrt{\frac{4 r^t_{i^t} q_{(i^t, k^t)}\log(1+T )}{ \max\{M^t_{(i^t, k^t)}, 1\} }} \cdot \mathbf{1}((i^t, k^t) = (i, k)) \right]\nonumber\\
= & \bE \left[ \sum_{i\in [n], k\in [K]}\sqrt{4 r_{i} q_{(i, k)}\log(1+ T )} \sum_{t = 1}^T \sqrt{\frac{ \mathbf{1}((i^t, k^t) = (i, k)) \cdot \mathbf{1}(N^{t-1}_i > 0) }{ \max\{M^t_{(i, k)}, 1\} }}   \right]  \label{eq:UCB_1_step_7}.
\end{align}
Step (\ref{eq:UCB_1_step_7}) is by the fact that $r^t_i \leq r_i \mathbf{1}(N^{t-1}_i > 0)$.
To proceed from (\ref{eq:UCB_1_step_7}), note that the summand $\sqrt{\frac{ \mathbf{1}((i^t, k^t) = (i, k)) \cdot \mathbf{1}(N^{t-1}_i > 0) }{ \max\{M^t_{(i, k)}, 1\} }}$ is positive only when action $(i, k)$ is taken at time $t$, and the amount of inventory of $i$ at time $t$ is still positive. Denote $\tau_i = \text{argmax}\left\{t\in [T] : N^{t-1}_i > 0\right\}$. Then we know that $M^{\tau_i}_{(i, k)}$ is the number of time steps when action $(i, k)$ is taken, and there is still remaining inventory for item $i$. By the fact that $\sum^n_{i=1}1/\sqrt{i}\leq \sqrt{2n}$ for all $n$, we have
$$
\sum_{t = 1}^T \sqrt{\frac{ \mathbf{1}((i^t, k^t) = (i, k)) \cdot \mathbf{1}(N^{t-1}_i > 0) }{ \max\{M^t_{(i, k)}, 1\} }}  \leq 1 + \sqrt{2 M^{\tau_i}_{(i, k)}},
$$
and we can proceed from step (\ref{eq:UCB_1_step_7}) as
\begin{align}
& \bE \left[ \sum_{i\in [n], k\in [K]}\sqrt{4 r_{i} q_{(i, k)}\log(1+ T )} \sum_{t = 1}^T \sqrt{\frac{ \mathbf{1}((i^t, k^t) = (i, k)) \cdot \mathbf{1}(N^{t-1}_i > 0) }{ \max\{M^t_{(i, k)}, 1\} }}   \right] \nonumber\\
\leq & 2nK\sqrt{\log(1+T)} + \bE \left[ \sum_{i\in [n], k\in [K]}\sqrt{4 r_{i} q_{(i, k)}\log(1+ T )} \cdot \sqrt{2 M^{\tau_i}_{(i, k)}}\right]\nonumber\\
\leq & 2nK\sqrt{\log(1+T)} + \sqrt{8 nK\cdot  \bE \left[ \sum_{i\in [n], k\in [K]}  r_{i} q_{(i, k)}M^{\tau_i}_{(i, k)}\right]\log(1+ T )}   \label{eq:UCB_1_step_8}\\
= & 2nK\sqrt{\log(1+T)} + \sqrt{8 nK \cdot \bE \left[ \ALG \right]\log(1+ T )}  \label{eq:UCB_1_step_9}\\
\leq & 2nK\sqrt{\log(1+T)} + \sqrt{8 nK \cdot  \OPT \log(1+ T )} \label{eq:UCB_1_step_9.5}
\end{align}
Step (\ref{eq:UCB_1_step_8}) is by the Cauchy Schwartz inequality, and step (\ref{eq:UCB_1_step_9}) is by the fact that $r_i q_{(i, k)} M^{\tau_i}_{(i, k)}$ is the amount of reward earned in the $T$ rounds by the algorithm in taking action $(i, k)$, and hence the total reward $\ALG$ earned by the algorithm is
$$
\ALG = \sum_{i\in [n], k\in [K]}  r_{i} q_{(i, k)}M^{\tau_i}_{(i, k)}.
$$
Step (\ref{eq:UCB_1_step_9.5}) is by the fact that the offline benchmark $\OPT$, which is the optimal value of the \textbf{Primal LP}, upper bounds $\mathbf{E}[\ALG]$.
\textbf{The second sum.} We first re-express the sum:
\begin{align}
& \bE\left[\sum^T_{t=1} \frac{6\log(1+T)}{ \max\{M^t_{(i^t, k^t)}, 1\}} \right]\nonumber\\
= & 6\log(1+T)\cdot  \bE\left[\sum_{i\in [n], k\in[K]}\sum^T_{t=1} \frac{\mathbf{1}((i^t, k^t = (i, k)))}{ \max\{M^t_{(i, k)}, 1\}} \right]\nonumber\\
\leq  & 6\log(1+T) \cdot  \bE\left[\sum_{i\in [n], k\in[K]} (2 + \log(M^T_{i, k})) \right]\label{eq:UCB_1_step_10}\\
\leq & 6nK\log(1+T)\left(2 + \log\frac{T}{nK}\right).\label{eq:UCB_1_step_11}
\end{align}
Step (\ref{eq:UCB_1_step_10}) is by the fact that $\sum^N_{i=1}\frac{1}{i}\leq 1 + \log n$, and step (\ref{eq:UCB_1_step_11}) is by the Jensen's inequality and the fact that $\sum_{i\in [n], k\in [K]} M^T_{i, k} = T$.

Combining the bounds (\ref{eq:UCB_1_step_9}) and (\ref{eq:UCB_1_step_11}), we finally have
\begin{align*}
\text{Reg}_0 & \leq 12\sqrt{2 nK \cdot \OPT \log(1+ T )}  + 12nK\sqrt{\log(1+T)} + 36 nK\log(1+T)\left(2 + \log\frac{T}{nK}\right) + 4nK\log(1+T)\nonumber\\
& \leq 12\sqrt{2 nK \cdot  \OPT \log(1+ T )}  + 52 nK\log(1+T)\left(2 + \log\frac{T}{nK}\right)\nonumber\\
& = O\left( \sqrt{nK\OPT\log(T)} + nK\log(T)\log(T / nK) \right).\Halmos
\end{align*}

\subsection{Proof of Theorem \ref{thm:LazyUCB_app1}, Bounding $\frac{1}{1+\epsilon} \text{Reg}_\epsilon$ for the LazyUCB Oracle}\label{app:pf_thm_LazyUCB_app1}
Similar to the analysis for the UCB oracle, we denote $a^t_* = (i^t_* , k^t_*)$.
\begin{align}
\frac{1}{1+\epsilon}\cdot \text{Reg}_\epsilon &= \frac{1}{1+\epsilon}\cdot \bE  \left[\sum^T_{t=1}R^t(a^t_*) \right] -\bE \left[\sum^T_{t=1} R^t(a^t)\right]\nonumber\\
& =  \frac{1}{1+\epsilon}\cdot \bE  \left[\sum^T_{t=1}R^t(a^t_*) \mathsf{1}({\cal E}^t) \right] + \frac{1}{1+\epsilon}\cdot \bE  \left[\sum^T_{t=1}R^t(a^t_*) \mathsf{1}(\bar{\cal E}^t) \right] -\bE\left[\sum^T_{t=1} R^t(a^t)\right]\nonumber\\
& \leq  \frac{1}{1+\epsilon}\cdot \bE  \left[\sum^T_{t=1}R^t(a^t_*) \mathsf{1}({\cal E}^t) \right] - \bE\left[\sum^T_{t=1} R^t(a^t)\right] + \frac{1}{1+\epsilon}\sum^T_{t=1} \frac{7nK}{1+t}. \label{eq:pf_app_1_LazyUCB_1}
\end{align}
Step (\ref{eq:pf_app_1_LazyUCB_1}) is by applying Lemma \ref{lemma:bd_LazyUCB_1}.

To proceed, we focus on the first term:
\begin{align}
& \frac{1}{1+\epsilon}\cdot \bE  \left[\sum^T_{t=1}R^t(a^t_*)\cdot \mathbf{1}({\cal E}^t) \right] \nonumber\\
\leq & \frac{1 + \epsilon / 2}{1+\epsilon}\cdot \bE  \left[\sum^T_{t=1}\mathsf{LazyUCB}^t(a^t_*) \cdot \mathbf{1}({\cal E}^t) \right]\label{eq:pf_app_1_LazyUCB_2}\\
\leq & \frac{1 + \epsilon / 2}{1+\epsilon}\cdot \bE  \left[\sum^T_{t=1}\mathsf{LazyUCB}^t(a^t) \cdot \mathbf{1}({\cal E}^t) \right]\label{eq:pf_app_1_LazyUCB_3}\\
\leq &  \frac{1 + \epsilon / 2}{1+\epsilon}\cdot \bE  \left[\sum^T_{t=1}r^t_{i^t} \cdot \mathbf{1}(x^t_{i^t} = 1) \cdot \left[\bar{q}^t_{(i^t, k^t)} + \mathsf{LazyRad}^t(i^t, k^t) \right] \cdot \mathbf{1}({\cal E}^t) \right]\label{eq:pf_app_1_LazyUCB_4}\\
\leq & \frac{1 + \epsilon / 2}{1+\epsilon}\cdot \bE  \left[\sum^T_{t=1}r^t_{i^t} \cdot \mathbf{1}(x^t_{i^t} = 1) \cdot \left[\left(1 + \frac{\epsilon}{2 + \epsilon}\right)\left[ q_{(i^t, k^t)} + \mathsf{LazyRad}^t(i^t, k^t) \right] +\mathsf{LazyRad}^t(i^t, k^t)\right] \cdot \mathbf{1}({\cal E}^t) \right]\label{eq:pf_app_1_LazyUCB_5}\\
\leq &  \bE  \left[\sum^T_{t=1}r^t_{i^t} \cdot \mathbf{1}(x^t_{i^t} = 1) \cdot q_{(i^t, k^t)}\right] + 2\cdot \bE\left[ \sum^T_{t=1} \mathsf{LazyRad}^t(i^t, k^t) \right]\nonumber\\
= & \mathbb{E}\left[\sum^T_{t=1} R^t(a^t)\right] + 2\cdot \bE\left[ \sum^T_{t=1} \mathsf{LazyRad}^t(i^t, k^t) \right] .\label{eq:pf_app_1_LazyUCB_6}
\end{align}
Step (\ref{eq:pf_app_1_LazyUCB_2}) is by Lemma \ref{lemma:bd_LazyUCB_1}, step (\ref{eq:pf_app_1_LazyUCB_3}) is by Line \ref{alg:lazy_UCB_app1_step5} in the LazyUCB oracle. Step (\ref{eq:pf_app_1_LazyUCB_4}) is by applying inequality (\ref{eq:lemma_bd_Lazy_1_lower}) in Lemma \ref{lemma:bd_LazyUCB_1}, and step (\ref{eq:pf_app_1_LazyUCB_5}) is by applying inequality (\ref{eq:lemma_bd_Lazy_1_upper}) in Lemma \ref{lemma:bd_LazyUCB_1}. We next proceed with bounding the confidence radii in (\ref{eq:pf_app_1_LazyUCB_6}):
\begin{align*}
 \bE\left[ \sum^T_{t=1} \mathsf{LazyRad}^t(i^t, k^t) \right] & =  \bE\left[ \sum^T_{t=1} \min\left\{\mathsf{lad}(\epsilon, M^t_{(i, k)} ; \delta_t), \mathsf{rad}(\bar{q}^t_{(i, k)}, M^t_{(i, k)}, \delta^{(t)})\right\} \right]\nonumber\\
 &\leq \bE\left[ \min\left\{\sum^T_{t=1} \mathsf{lad}(\epsilon, M^t_{(i, k)} ; \delta_t), \sum^T_{t=1} \mathsf{rad}(\bar{q}^t_{(i, k)}, M^t_{(i, k)}, \delta^{(t)})\right\} \right]\nonumber\\
 &\leq \min\left\{ \bE\left[ \sum^T_{t=1} \mathsf{lad}(\epsilon, M^t_{(i, k)} ; \delta_t)\right], \bE\left[\sum^T_{t=1} \mathsf{rad}(\bar{q}^t_{(i, k)}, M^t_{(i, k)}, \delta^{(t)}) \right] \right\} \nonumber.
\end{align*}
By the analysis in the proof of Theorem \ref{thm:UCB_app1}, we see that
\begin{align}
\bE\left[\sum^T_{t=1} \mathsf{rad}(\bar{q}^t_{(i, k)}, M^t_{(i, k)}, \delta^{(t)}) \right]& \leq 2\sqrt{2 nK \cdot  \OPT \log(1+ T )}  + 8 nK\log(1+T)\left(2 + \log\frac{T}{nK}\right)\label{eq:explicit_1}\\
&= O\left( \sqrt{nK \cdot \OPT \cdot \log(T)} + nK\log(T)\log\frac{T}{nK} \right).\label{eq:pf_app_1_LazyUCB_7}
\end{align}
Next, we can upper bound the sum on $\mathsf{lad}$ as follows:
\begin{align}
&  \bE\left[ \sum^T_{t=1} \mathsf{lad}(\epsilon, M^t_{(i^t, k^t)} ; \delta_t) \right] \nonumber\\
\leq & \bE\left[ \sum_{i\in [n], k\in [K]}\sum^T_{t=1}  \mathsf{lad}(\epsilon, M^t_{(i, k)} ; \delta_t)\mathbf{1}((i^t, k^t) = (i, k)) \right] \nonumber\\
 \leq & \frac{2(2+\epsilon)\log(1+T)}{\epsilon} \cdot \bE\left[ \sum_{i\in [n], k\in [K]}\sum^T_{t=1}\frac{1}{M^t_{(i, k)}} \cdot\mathbf{1}((i^t, k^t) = (i, k)) \right] \nonumber\\
\leq &\frac{2(2+\epsilon)\log(1+T)}{\epsilon} \cdot \bE\left[ \sum_{i\in [n], k\in [K]} \left(1 +  \log(M^T_{i, k})\mathbf{1}(M^T_{(i, k)} > 0)\right)\right] \nonumber\\
\leq &\frac{2(2+\epsilon)\log(1+T)}{\epsilon} \cdot nK \cdot \log\frac{T}{nK}  \label{eq:pf_app_1_LazyUCB_8}\\
 =& O\left( \frac{(1+\epsilon)nK\log(T) \log(T/nK)}{\epsilon} \right). \label{eq:pf_app_1_LazyUCB_9}
\end{align}
Step (\ref{eq:pf_app_1_LazyUCB_8}) is by the Jensen inequality and the concavity of $\log$. Finally, applying (\ref{eq:pf_app_1_LazyUCB_5}, \ref{eq:pf_app_1_LazyUCB_7}, \ref{eq:pf_app_1_LazyUCB_9})  back to (\ref{eq:pf_app_1_LazyUCB_1}), we have
\begin{align*}
&\frac{1}{1+\epsilon}\cdot \text{Reg}_\epsilon  = \min\left\{O\left( \sqrt{nK \cdot \OPT \cdot \log T} + nK\log(T)\log\frac{T}{nK} \right), O\left( \frac{(1+\epsilon) nK\log(T) \log(T/nK)}{\epsilon} \right) \right\}\nonumber\\
= & \min\left\{ O\left( \sqrt{nK \cdot \OPT \cdot \log T}\right) , O\left( \frac{(1+\epsilon) nK\log(T/nK)\log(T)}{\epsilon} \right) \right\} + O\left( nK\log(T) \log\left(\frac{T}{nK}\right) \right)\nonumber\\
= & \min\left\{\tilde{O}(\sqrt{nK \OPT}), \tilde{O}\left(\frac{(1+\epsilon) nK}{\epsilon}\right) \right\}. \Halmos
\end{align*}

\section{Proof of Theorem \ref{thm:lowerBound} (Lower Bound on Regret of Algorithm)}\label{app:thm_lowerBound}

In this section, we establish a lower bound on the overall loss of any online algorithm for the online matching problem. Specifically, we prove that the performance guarantee in Corollary \ref{cor:UCB_app1} is tight in the sense that both of the loss terms $\OPT/e$, $\tilde O(\sqrt{\bE[\OPT]})$ are unavoidable due to the uncertainty on the probabilities $q_{(i,k)}$ and the uncertainty on the sequence of customer contexts.

We construct a randomized worst-case instance as follows. The capacity values are the same $b_i = b$ for all $i \in [n]$. Let $\pi$ be a random permutation of $[n]$. There are $T=2bn$ customers, split into $n$ ``groups'' of $2b$ customers each. The customers in each group $j\in[n]$ all have the same context (feature) vector $x^{(j)}$, where
\begin{equation*}
x^{(j)}_i=1\text{ if and only if }\pi(i)\ge j.
\end{equation*}
In other words, if we view $\pi(i)$ as a random score of resource $i$, then the customers become increasingly selective as customers in group $j$ are only interested in resources $i$ with scores higher than $j$.

Let $\vell=(\ell_1,\ldots,\ell_n)\in[K]^n$ be a random vector of ``secret arms''. The distribution $\rho_{x,(i,k)}$ is given by
\begin{align*}
\rho_{x,(i,k)}(\ve_i) &=\bI(x_i=1)\left(\frac{1-\vare}{2}+\bI(k=\ell_i)\cdot\vare\right) \\
\rho_{x,(i,k)}(\vzero) &=1-\rho_{x,(i,k)}(\ve_i) \\
\rho_{x,(i,k)}(\vy) &=0\text{ for all other outcomes $\vy$ in $\{0,1\}^n$}
\end{align*}
Here, $\vare\in(0,1/2]$ will be defined in our analysis. We choose $\vare\le1/2$  just for technical convenience.

This problem instance is a randomized one because we draw both $\pi$ and $\vell$ uniformly at random. Note that for all realization of $\pi$ and $\vell$, $\OPT$ will be $bn$.


A \textit{deterministic policy} is a mapping, for any $t\in\bN$, from any history of observed contexts and outcomes, $(x^1,\vy^1,\ldots,x^t)$ in $\cX^t\times \{0,1\}^{n \times(t-1)}$, to an action to play on context $x^t$, in $\cA$.
Our proof strategy is to upper-bound the performance of any deterministic policy on this randomized instance (it suffices to consider deterministic policies because when given the randomized instance, there always exists an optimal policy which is deterministic).  

\begin{theorem}[Lower Bound]
\label{thm:lowerBound}
Let $n,b,K$ be any positive integers satisfying $b\ge K\ge3$.  Then there exists a randomized instance (with a random arrival sequence and a random mapping from contexts to outcomes)
such that for any deterministic or randomized algorithm,
\begin{align*}
\OPT-\bE[\ALG] &\ge\frac{\OPT}{e}+ \Theta(\sqrt{K\OPT}).
\end{align*}
\end{theorem}

We prove this theorem through Lemmas \ref{lm:lbProof1}, \ref{lem::simplified}, \ref{lem::KVV}, and Proposition \ref{prop::harmonic}. The proof is based on an information-theoretic analysis.

\input{lowerBoundProof}
$\Halmos$

\section{Extension to Multiple Reward Rates per Resource} \label{sec:multipleRates}

We consider the generalization to the setting where each resource $i$ could be depleted (sold) at varying rates (prices), instead of a single rate $r_i$, following \cite{MSL17}.
This is used for our simulations of assortment optimization on the hotel data set in Section~\ref{sec:numerical1}.

We assume that for each resource $i$, its set of reward rates $\cP_i$ is known in advance.
This introduces an aspect of ``admission control'' to the problem, where sometimes it is desirable to completely reject a customer, who is only willing to purchase a resource at a low price, to reserve resources for higher-paying customers.

We impose additional structure on the mapping from contexts and actions to distributions over outcomes.  We assume that each $\cP_i$ is finite and that the action set $\cA$ is a non-empty downward-closed set of \textit{combinations} $(i,P)$ of resources $i$ and prices $P\in\cP_i$.
$\cA$ can be thought of as the feasible assortments of (resource, price)-combinations that the firm can offer.  For example, actions $a\in\cA$ can be constrained so that $|\{(j,P)\in a:j=i\}|\le1$ for all $i$, which says that the firm can set at most one price for each resource, or alternatively constrained only in total cardinality,
so that the firm can offer the same resource at multiple prices (where presumably additional benefits are attached with the higher price).

We only allow the firm to offer combinations $(i,P)$'s for which resource $i$ has not ran out.
Note that this is in contrast to the model described in Section~\ref{sec:model}, where actions can be arbitrarily chosen and resources which have ran out are not consumed.
Since $\cA$ is downward-closed, it always contains the empty assortment $\emptyset$, which the firm can offer if it has ran out of all resources.
When the firm offers an assortment $a$, the outcome is described by a vector $\vy\in\{0,1\}^{|\cP_1|+\ldots+|\cP_n|}$ describing which combinations $(i,P)$ were consumed. Only combinations $(i,P)\in a$ could be consumed, and for each resource $i$, at most one combination corresponding to $i$ could be consumed.

\begin{assumption}[Substitutability]\label{ass::subst}
Consider any context $x\in\cX$ and any two actions $a,a'\in\cA$ with $a\subseteq a'$.  Then for any combination $(i,P)\in a$, we have $\sum_{\vy:\vy_{(i,P)}=1}\rho_{x,a}(\vy)\ge\sum_{\vy:\vy_{(i,P)}=1}\rho_{x,a'}(\vy)$.
\end{assumption}
Colloquially, Assumption~\ref{ass::subst} reads that augmenting an assortment (from $a$ to $a'$) can only decrease the chances of selling the combinations already in the assortment.  It is a very mild assumption, originating from \citet{GNR14}, which holds under any random-utility choice model.

We still define $\OPT$ as the optimal objective value of the LP relaxation:

\textbf{Primal:}
\begin{align}
\max  \sum_{a \in \cA} \sum_{t \in [T]}s_{a,t} \sum_{(i,P) \in a} \sum_{\vy : \vy_{(i,P)}=1} \rho_{x^t,a}(\vy)  P & & \label{eq:LP1b}\\
\sum_{a \in \cA} \sum_{t \in [T]}s_{a,t}\sum_{(i,P) \in a} \sum_{\vy : \vy_{(i,P)}=1} \rho_{x^t,a}(\vy) &\le b_i &i\in[n] \nonumber \\
\sum_{a \in \cA} s_{a,t} &\le1 &t\in[T]\nonumber \\
s_{a,t} &\ge0 &a\in \cA,t\in[T]\nonumber
\end{align}

We modify the IBOL algorithm from Section~\ref{sec:alg} for the current setting with multiple reward rates.  The only change is in the definition of rewards in the auxiliary online learning problem.

In Section~\ref{sec:alg}, at each point in time $t$, we defined a virtual reward $r^t_i$ for each resource $i$, based on the fraction $N^{t-1}_i/b_i$ of that resource depleted at that time. Earlier, $r^t_i$ was defined as the product $r_i$ and a \textit{penalty factor} $(1-\Psi(N^{t-1}_i/b_i))$, where $\Psi(\cdot)$ increased from 0 to 1 as the fraction depleted increased from 0 to 1.  Now that resource $i$ has multiple reward rates in $\cP_i$, the change from \cite{MSL17} is that we instead subtract a \textit{virtual cost}.  Specifically, for each combination $(i,P)$, its virtual reward at time $t$ is defined to be
\begin{equation}\label{eq:virtual_cost}
r^t_{(i,P)}=P-\Phi_{\cP_i}\left(\frac{N^{t-1}_i}{b_i}\right),
\end{equation}
where $\Phi_{\cP_i}(\cdot)$ increases from 0 to $\max\cP_i$ as the fraction of resource $i$ depleted increases from 0 to 1.  Note that it is possible for the virtual reward $r^t_{(i,P)}$ to be negative. The definition of $\Phi_{\cP_i}(\cdot)$, which is defined in Section 2.1 in \citep{MSL17}, is rather intricate. For completeness, we provide the definition of $\Phi_{\cP_i}(\cdot)$, together with the definition of parameters $\{\alpha^{(1)}_i\}$, in Appendix \ref{sec:Phi_alpha}.
Similar to the previous single reward rate setting, we define the discounted reward at time $t$ as
\[ R^t(a) = \sum_{(i,P) \in a} \sum_{\vy : \vy_{(i,P)}=1} \rho_{x^t,a}(\vy) \left[P - \Phi_{\cP_i}\left(\frac{N^{t-1}_i}{b_i}\right)\right], \]
and denote $a^*_t = \text{argmax}_{a\in \cA}R^t(a)$.

\begin{theorem}\label{thm:multiPrice}
The total reward $\ALG$ earned by the algorithm that uses virtual costs (\ref{eq:virtual_cost}) satisfies
\begin{equation}\label{eq:multiPrice}
\OPT \le \frac{(1+\bmin)(1-e^{-1/\bmin})}{1-\exp(-\min_i\alpha^{(1)}_i)} \cdot \bE[\ALG]+\bE[\REG(\cF_T)],
\end{equation}
where $\REG(\cF_T) = \sum_{t \in [T]} ( R^t(a_*^t) - R^t(a^t))$.
\end{theorem}
Compared to Theorem \ref{thm:main}, the only change in inequality (\ref{eq:multiPrice}) in Theorem \ref{thm:multiPrice} is in the denominator, where the denominator $1-e^{-1}$ in Theorem \ref{thm:main} has been replaced by denominator $\min_i(1-e^{-\alpha^{(1)}_i})$ in Theorem \ref{thm:multiPrice}. For each resource $i$, the factor $1-e^{-\alpha^{(1)}_i}$ is the \textit{competitive ratio associated with} price set $\cP_i$, and the competitive ratio is equal to $1-1/e$ when $\cP_i$ is a singleton.

\begin{proof}{Proof.}
We start with the formulation \textbf{Dual:}
\begin{align}
\min\sum_{i \in [n]}b_i \lambda_i+\sum_{t \in [T]} \gamma_t & & \label{eq:LP2b} \\
\gamma_t &\geq \sum_{(i,P) \in a} \sum_{\vy : \vy_{(i,P)}=1} \rho_{x^t,a}(\vy) (P - \lambda_i) & a \in \cA, t \in [T] \label{dual::feasibilityb}\\
\lambda_i,\gamma_t &\ge0 &i\in[n],t\in[T] \nonumber
\end{align}

Define dual variables to LP \eqref{eq:LP2b} as
\[ \Lambda_i =\Phi_{\cP_i}\left(\frac{N_i^{T}}{b_i}\right), \quad \Gamma_t =R^t(a_*^t). \]
These dual variables can be readily verified to be feasible for LP \eqref{eq:LP2b}. Based on  strong duality for linear program, we know that
\begin{align}
\OPT & \leq \bE\left[ \sum_{i \in [n]} b_i \Lambda_i + \sum_{t \in [T]} \Gamma_t\right] \nonumber \\
& = \bE\left[\sum_{i \in [n]} b_i \Phi_{\cP_i}\left(\frac{N_i^{T}}{b_i}\right) + \sum_{t \in [T]}R^t(a^t) +  \sum_{t \in [T]} ( R^t(a_*^t) - R^t(a^t))\right] \nonumber \\
& = \bE\left[\sum_{i \in [n]} b_i \Phi_{\cP_i}\left(\frac{N_i^{T}}{b_i}\right) + \sum_{t \in [T]}R^t(a^t)] + \bE[\REG(\cF_T)]\right]. \label{eqn:combineMultiplePrice}
\end{align}
The following is shown in \citet{MSL17}:
 $$\bE\left[\sum_{i \in [n]} b_i \Phi_{\cP_i}\left(\frac{N_i^{T}}{b_i}\right) + \sum_{t \in [T]}R^t(a^t)\right] \leq \frac{(1+\bmin)(1-e^{-1/\bmin})}{1-\exp(-\min_i\alpha^{(1)}_i)} \cdot \bE[\ALG],$$
which completes the proof after combining with equation \eqref{eqn:combineMultiplePrice} and rearranging.
\halmos
\end{proof}

\subsection{Definition of $\alpha^{(1)}_i, \Phi_{{\cal P}_i}$}\label{sec:Phi_alpha}
For a set of ${\cal P}$, consisting
of $m$ discrete prices $0 < r^{(1)} < \cdots < r^{(m)}$, the function $\Phi_{\cal P}$ is defined as follows. To define $\Phi_{\cal P}$, we first need to define $m$ constants $\alpha^{(1)} , \alpha^{(2)} , \ldots , \alpha^{(m)}$, which constitute a unique set of positive real numbers that satisfies the following set of equations:
\begin{align*}
\alpha^{(i)} > 0 & \quad \text{ for all $1\leq i\leq m$}\nonumber\\
\sum^m_{i=1}\alpha^{(i)} = 1 &\nonumber\\
1 - e^{-\alpha^{(1)}} &= \frac{1}{1 - r^{(1)} / r^{(2)}} \cdot (1 - e^{-\alpha^{(2)}} ) = \ldots =  \frac{1}{1 - r^{(m-1)} / r^{(m)}} \cdot (1 - e^{-\alpha^{(m)}} ) \nonumber
\end{align*}
By \citep{MSL17}, the above set of equations has a unique solution. To define the function $\Phi_{\cal P}$, we still need to define one sets of parameters and a function:
\begin{itemize}
\item $L^{(0)} = 0,$ and $L^{(j)} = \sum^j_{j' =1 } \alpha^{(j')}$, and in particular $L^{(m)} = 1$.
\item $\ell(\cdot)$: a function on $[0, 1]$, where $\ell(w)$ is the unique $j\in [m]$ for which $w\in [L^{(j - 1)}, L^{(j)})$.
\end{itemize}
The function $\Phi_{\cal P}$ for price set ${\cal P}$ is then defined over $w\in [0, 1]$ by:
$$
\Phi_{\cal P}(w) = r^{(\ell(w) - 1)} + (r^{(\ell(w))} - r^{(\ell(w) - 1)}) \frac{ \exp\left[w - L^{\ell(w) - 1} \right]  -1}{ \exp\left[\alpha^{(\ell)} \right] - 1}.
$$
Finally, we can apply the above definition on each ${\cal P} = {\cal P}_i$, which yields the parameter $\alpha^{(1)}_i = \alpha^{(1)}$ for the Theorem.

\section{Supplementary Details about Numerical Experiments}\label{app:supp_details_num}
We provide additional details about our choice estimation from Section~\ref{sec:numerical1}. We define 8 customer types, one for each combination of the 3 following binary features.
\begin{enumerate}
\item Group: whether the customer indicated a party size greater than 1.
\item  CRO: whether the customer booked using the Central Reservation Office, as opposed to the hotel's website or a Global Distribution System (for details on these terms, see \citep{BFG09}).
\item VIP: whether the customer had any kind of VIP status.
\end{enumerate}
We did not use features such as: whether the booking date is a weekend, whether the check-in date is a weekend, the length of stay, or the number of days in advance booked. Such features did not result in a more predictive model.

We estimate the mean MNL utilities for each of the 8 products separately for each customer
type. The results are displayed in Table 3 in \citep{MSL17} (Page 52). The total share of each customer type (out of all the transactions) is also displayed in that table. We should point out that it is possible for a customer to choose the higher fare for a room, even if the lower fare was also offered. This is because the higher fares are often packaged with additional offers, such as airline services, city attractions, in-room services, etc. We have shifted the mean utilities so that for each customer type, the weights of both the no-purchase option, and the most-preferred purchase option, is equal to 0. (We synthetically set the weight of the no-purchase option because it is not possible to estimate from the data.) The large weights on the no-purchase options ensure that the revenue-maximizing assortments tend to include both the low and high fares.

In the setting with greater fare differentiation (Subsection 7.5), the high prices of the King, Queen, Suite, and Two-double rooms are adjusted to \$614, \$608, \$768, \$612, respectively (twice the lower fares). The mean utility of the no-purchase option is increased by 2 for every customer type, to ensure that the revenue-maximizing assortments still include both the low and high fares.

\end{APPENDICES}

\end{document}

%% file: lowerBoundProof.tex
Let $\cT_j = \{2b(j-1)+1,\ldots,2bj\}$ denote the indices of the customers in group $j$, for all $j\in[n]$. Let $\cA_i = \{(i,k):k\in[K]\}$ denote the set of actions that correspond to resource $i$, for all $i\in[n]$. Let  $Y_t$ be the indicator random variable for whether customer $t$ accepted her offer, for all $t\in[T]$.


We can write $\ALG$, the random variable for the total reward earned by the deterministic policy, as
\begin{equation}\label{eqn::ALG}
\ALG=\sum_{i=1}^n\min\Big\{\sum_{j=1}^i\sum_{t\in\cT_j}\bI(Y_t=1\cap a^t\in\cA_{\pi^{-1}(i)}),b\Big\}.
\end{equation}

To upper-bound $\bE[\ALG]$, we need to upper-bound $\bE[\sum_{j=1}^i\sum_{t\in\cT_j}\bI(Y_t=1\cap a^t\in\cA_{\pi^{-1}(i)})]$. Thus, we will focus on analyzing $\Pr[Y_t=1\cap a^t\in\cA_{\pi^{-1}(i)}]$ for an arbitrary $i\in[n]$, $j\le i$, and $t\in\cT_j$.
\begin{align}
& \Pr[Y_t=1\cap a^t\in\cA_{\pi^{-1}(i)}] \nonumber \\
= &\Pr[Y_t=1|a^t=(\pi^{-1}(i),\ell_{\pi^{-1}(i)})]\cdot\Pr[a^t=(\pi^{-1}(i),\ell_{\pi^{-1}(i)})] \nonumber \\
&+\Pr[Y_t=1|a^t\in\cA_{\pi^{-1}(i)}\cap a^t\neq(\pi^{-1}(i),\ell_{\pi^{-1}(i)})]\cdot\Pr[a^t\in\cA_{\pi^{-1}(i)}\cap a^t\neq(\pi^{-1}(i),\ell_{\pi^{-1}(i)})] \nonumber \\
= &\frac{1+\vare}{2}\Pr[a^t=(\pi^{-1}(i),\ell_{\pi^{-1}(i)})]+\frac{1-\vare}{2}\Pr[a^t\in\cA_{\pi^{-1}(i)}\cap a^t\neq(\pi^{-1}(i),\ell_{\pi^{-1}(i)})] \nonumber \\
= &\frac{1-\vare}{2}\Pr[a^t\in\cA_{\pi^{-1}(i)}]+\vare\cdot\Pr[a^t=(\pi^{-1}(i),\ell_{\pi^{-1}(i)})] \label{eqn::splitByEps}
\end{align}

The difficult term to analyze is $\Pr[a^t=(\pi^{-1}(i),\ell_{\pi^{-1}(i)})]$. Note that the distribution of $a^t$ is affected by the entire realized vector of secret arms $\vell$, as well as the realized values of $\pi^{-1}(1),\ldots,\pi^{-1}(j-1)$.

Now, consider an alternate universe where for each resource $m \in [n]$, all of the actions $(m,1),\ldots,(m,K)$ result in the customer accepting with probability $\frac{1-\vare}{2}$, regardless of the value of $\ell_m$.  We can also consider the execution of the fixed, deterministic policy in this alternate universe, where we will use random variables $\oa^t,\oY_t$ to refer to its execution.

\begin{lemma}[Using information theory to get an initial bound]\label{lm:lbProof1}
Let $j \in [n]$ be any customer group and let $t$ be any customer from $\cT_j$. Let $S \subseteq [n]$ be any set of resources. Condition on any sequence of $j-1$ resources with lowest scores
\begin{equation*}
\pi^{-1}([j-1]):=(\pi^{-1}(1),\ldots,\pi^{-1}(j-1))
\end{equation*}
and vector of secret arms $\vell$. Then
\begin{align}
\begin{split}\label{eqn::infoTheory}
\sum_{m \in S} \Pr[a^t=(m,\ell_m)|\pi^{-1}([j-1]),\vell]\le & \sum_{m \in S} \Pr[\oa^t=(m,\ell_m)|\pi^{-1}([j-1]),\vell] \\
&+\vare\sqrt{\sum_{s=1}^{t-1} \sum_{m \not\in \pi^{-1}([s-1])} \Pr[\oa^s=(m,\ell_m)|\pi^{-1}([j-1]),\vell]}.
\end{split}
\end{align}
\end{lemma}

\begin{proof}{Proof.}
For brevity, we will omit the conditioning on $\pi^{-1}(1),\ldots,\pi^{-1}(j-1)$ and $\vell$ throughout the proof.  We will also use $\vZ^s$ to denote the vector of random variables $(Y_1,\ldots,Y_s)$ and $\vz^s$ to denote a vector in $\{0,1\}^s$, for any $s\in[t-1]$.

First, note that $a^t$ is the rule of the deterministic policy for choosing the action at time $t$, dependent on sequence of observations $\vZ^{t-1}$ and the sequence of contexts $x^1,\ldots,x^t$ (which is captured by $\pi^{-1}(1),\ldots,\pi^{-1}(j-1)$).


\begin{align} 
\sum_{m \in S} \Pr[a^t=(m,\ell_m)] &= \sum_{\vz^{t-1}\in\{0,1\}^{t-1}} \Pr[\vZ^{t-1} = \vz^{t-1}] \sum_{m \in S} \Pr[a^t=(m,\ell_m) | \vZ^{t-1} = \vz^{t-1}] \nonumber\\
&\le \sum_{\vz^{t-1}\in\{0,1\}^{t-1}} \Pr[\ovZ^{t-1} = \vz^{t-1}] \sum_{m \in S} \Pr[\oa^t=(m,\ell_m) | \ovZ^{t-1} = \vz^{t-1}]  \nonumber \\
&\ \ \  +\delta(\ovZ^{t-1},\vZ^{t-1}) \nonumber\\
&\le \sum_{m \in S} \Pr[\oa^t=(m,\ell_m)]+\sqrt{\frac{1}{2}\KL(\ovZ^{t-1}\|\vZ^{t-1})},  \label{eqn::klIntroduced}
\end{align}
where the first inequality is from the definition that
\begin{equation*}
\delta(\ovZ^{t-1},\vZ^{t-1})=\sum_{\vz^{t-1}\in\{0,1\}^{t-1}}|\Pr[\ovZ^{t-1}=\vz^{t-1}]-\Pr[\vZ^{t-1}=\vz^{t-1}]|,
\end{equation*}
and the second inequality is due to Pinsker's inequality.

\begin{align*}
&\KL(\ovZ^{t-1}\|\vZ^{t-1}) \\
=&\sum_{\vz^{t-1}\in\{0,1\}^{t-1}}\Pr[\ovZ^{t-1}=\vz^{t-1}]\cdot\ln\frac{\Pr[\ovZ^{t-1}=\vz^{t-1}]}{\Pr[\vZ^{t-1}=\vz^{t-1}]} \\
=&\sum_{s=1}^{t-1}\sum_{\vz^{s-1}\in\{0,1\}^{t-1}}\Pr[\ovZ^{s-1}=\vz^{s-1}]\left(\sum_{y_s \in \{0,1\}}\Pr[\oY_s=y_s|\ovZ^{s-1}=\vz^{s-1}]\cdot\ln\frac{\Pr[\oY_s=y_s|\ovZ^{s-1}=\vz^{s-1}]}{\Pr[Y_s=y_s|\vZ^{s-1}=\vz^{s-1}]}\right),
\end{align*}
where the second equality comes from the Chain Rule for KL-divergences.  Now, consider the term inside the parentheses.  Conditioned on $\vz^{s-1}$ (and $\pi^{-1}([j-1])$, which have been omitted in the notation), actions $\oa^s$ and $a^s$ are deterministic and equal.  If this action is $(m,\ell_m)$ for some $m \in [n]$ and $m \not\in \pi^{-1}([s-1])$, then $\oY_s$ is 1 w.p.\ $\frac{1-\vare}{2}$ while $Y_s$ is 1 w.p.\ $\frac{1+\vare}{2}$, and the term inside the parentheses is the KL-divergence of $\Ber(\frac{1+\vare}{2})$ from $\Ber(\frac{1-\vare}{2})$, equal to $\vare\cdot\ln\frac{1+\vare}{1-\vare}$.  Otherwise, $\oY_s$ and $Y_s$ are identically distributed, and the term inside the parentheses is zero.

Therefore, 
\begin{align*}
&\KL(\ovZ^{t-1}\|\vZ^{t-1})\\
=&\sum_{s=1}^{t-1}\sum_{m \not\in \pi^{-1}([s-1])} \Pr[\oa^s=(m,\ell_m)]\left(\vare\cdot\ln\frac{1+\vare}{1-\vare}\right)\\
\le& \sum_{s=1}^{t-1} \sum_{m \not\in \pi^{-1}([s-1])} \Pr[\oa^s=(m,\ell_m)]\left(2\vare^2\right)
\end{align*}
 (the inequality is because $\vare\le1/2$) and substituting into (\ref{eqn::klIntroduced}) completes the proof of the lemma.
 
 \halmos
\end{proof}

\begin{definition}
Define the following random variables for all $i,j\in[n]$:
\begin{itemize}
\item $Q_{i,j} = \sum_{t\in\cT_j}\bI(a^t\in\cA_{\pi^{-1}(i)})$ is the total number of group-$j$ customers on whom an action corresponding to resource $\pi^{-1}(i)$ is played;
\item $Q^*_{i,j} = \sum_{t\in\cT_j}\bI(a^t=(\pi^{-1}(i),\ell_{\pi^{-1}(i)}))$ is the total number of group-$j$ customers on whom action $(\pi^{-1}(i),\ell_{\pi^{-1}(i)})$ is played.
\end{itemize}
Let $q_{i,j},q^*_{i,j}$ denote the expected values of $Q_{i,j},Q^*_{i,j}$, respectively.  We will use $\oQ_{i,j},\oQ^*_{i,j},\oq_{i,j},\oq^*_{i,j}$ to refer to the respective quantities under the alternate universe.
\end{definition}

\begin{lemma}[Removing dependence on $t$, $\pi$, and $\vell$]\label{lem::simplified}
Let $D \subseteq [n]$ be any set of scores, and $\pi^{-1}(D)$ be the corresponding set of resources with scores $D$. For any group $j \in [n]$,
\begin{equation*}
\sum_{i \in D} \bE[Q^*_{i,j}]\le  \frac{1}{K} \sum_{i \in D} \bE[\oQ_{i,j}]  +  2b\vare \sqrt{\frac{2bj}{K}}.
\end{equation*}
\end{lemma}

\begin{proof}{Proof.}
Consider the probability
\begin{equation*}
\Pr[\oa^s=(m,\ell_m)|\pi^{-1}([j-1]),\vell]
\end{equation*}
from the RHS of inequality (\ref{eqn::infoTheory}).  Since $\oa^s$, which refers to the alternate universe, is unaffected by the value of $\ell_m$, the probability is identical after removing the conditioning on $\ell_m$.  We can do this for all $s=1,\ldots,t$. 

Let $\vell_{-m}$ denote the fixed vector of secret arms for resources other than $m$. We take an average over the randomness in $\ell_m$ (drawn uniformly from $[K]$) and apply the law of total probability to obtain:
\begin{align*}
&\bE[\Pr[\oa^s=(m,\ell_m)|\pi^{-1}([j-1]),\vell]]\\
=&\bE[\Pr[\oa^s=(m,\ell_m)|\pi^{-1}([j-1]),\vell_{-m}]]\\
\leq&\bE[\frac{1}{K}\Pr[\oa^s\in \cA_m|\pi^{-1}([j-1]),\vell_{-m}]]\\
=&\frac{1}{K}\Pr[\oa^s\in \cA_m],
\end{align*}
where the inequality is because the probability that $\oa^s$ turns out to be the ``secret arm'' $\ell_m$ of resource $m$ is $1/K$ if $\oa^s\in\cA_m$, and 0 otherwise.

Then, for any set $S \subseteq [n]$ of resources, we apply inequality (\ref{eqn::infoTheory}) to obtain:
\begin{align*}
&\sum_{m \in S} \Pr[a^t=(m,\ell_m)] \\ 
&= \sum_{m \in S} \bE[ \Pr[a^t=(m,\ell_m)|\pi^{-1}([j-1]),\vell] ] \\ 
&\le \sum_{m \in S} \bE[ \Pr[\oa^t=(m,\ell_m)|\pi^{-1}([j-1]),\vell]]+\vare\cdot\bE\left[\sqrt{\sum_{s=1}^{t-1} \sum_{m \not\in \pi^{-1}([s-1])} \Pr[\oa^s=(m,\ell_m)|\pi^{-1}([j-1]),\vell]}\right] \\
&\le \sum_{m \in S} \bE[ \Pr[\oa^t=(m,\ell_m)|\pi^{-1}([j-1]),\vell]]+\vare\cdot \sqrt{\sum_{s=1}^{t-1} \bE\left[\sum_{m \not\in \pi^{-1}([s-1])} \Pr[\oa^s=(m,\ell_m)|\pi^{-1}([j-1]),\vell]\right]} \\
&\le \sum_{m \in S} \bE[ \Pr[\oa^t=(m,\ell_m)|\pi^{-1}([j-1]),\vell]]+\vare\cdot \sqrt{\sum_{s=1}^{t-1} \bE\left[\sum_{m \in [n]} \Pr[\oa^s=(m,\ell_m)|\pi^{-1}([j-1]),\vell]\right]} \\
&\le\frac{1}{K} \sum_{m \in S} \Pr[\oa^t\in\cA_m]+\vare\sqrt{\frac{1}{K}\sum_{s=1}^{t-1} \sum_{m \in[n]} \Pr[\oa^s\in\cA_m]}\\
&\le\frac{1}{K} \sum_{m \in S} \Pr[\oa^t\in\cA_m]+\vare\sqrt{\frac{t}{K}}.
\end{align*}
The second inequality is Jensen's inequality (the square root function is concave).  

By the definition of $Q_{i,j}$ and $Q^*_{i,j}$, we sum over the $2b$ values of $t$ in $\cT_j$ to obtain
\begin{align*}
& \sum_{i \in D} \bE[Q^*_{i,j}]\\
= & \sum_{t \in \cT_j} \bE\left[ \sum_{m \in  \pi^{-1}(D)} \Pr[a^t = (m,\ell_m)] \right]\\
= & \sum_{t \in \cT_j} \bE\left[ \bE\left[   \sum_{m \in  S} \Pr[a^t = (m,\ell_m)] {\bigg | } \pi^{-1}(D) = S  \right] \right]\\
\leq &    \sum_{t \in \cT_j} \bE\left[ \bE\left[ \frac{1}{K}  \sum_{m \in  S} \Pr[\oa^t\in\cA_m]+\vare\sqrt{\frac{t}{K}} {\bigg | } \pi^{-1}(D) = S  \right] \right]\\
= & \frac{1}{K}   \sum_{t \in \cT_j} \bE\left[   \sum_{m \in  \pi^{-1}(D)} \Pr[\oa^t\in\cA_m] \right] + \sum_{t \in \cT_j} \vare\sqrt{\frac{t}{K}}\\
= & \frac{1}{K}  \sum_{i \in D} \bE[ \oQ_{i,j}]+ \sum_{t \in \cT_j} \vare\sqrt{\frac{t}{K}}\\
\leq & \frac{1}{K} \sum_{i \in D} \bE[\oQ_{i,j}]  +\vare  2b \sqrt{\frac{2bj}{K}}.
\end{align*}
The last inequality uses the fact that $t \leq 2bj$ for all $t \in \cT_j$.

\halmos
\end{proof}

\begin{lemma}[Argument for randomized permutation]\label{lem::KVV}
For any customer group $j\in[n]$ and compatible resource with score $i\ge j$, both $\bE[Q_{i,j}]$ and $\bE[\oQ_{i,j}]$ are upper-bounded by $2b/(n-j+1)$.
\end{lemma}

\begin{proof}{Proof.}
We prove the result for $\bE[Q_{i,j}]$ (the proof for $\bE[\oQ_{i,j}]$ is identical):
\begin{align*}
\bE[Q_{i,j}] &=\sum_{t\in\cT_j}\Pr[a^t\in\cA_{\pi^{-1}(i)}] \\
&=\sum_{t\in\cT_j}\sum_{\pi^{-1}([j-1])}\Pr[\pi^{-1}([j-1])]\cdot\Pr[a^t\in\cA_{\pi^{-1}(i)}|\pi^{-1}([j-1])] \\
&=\sum_{t\in\cT_j}\sum_{\pi^{-1}([j-1])}\Pr[\pi^{-1}([j-1])]\sum_{m\notin\pi^{-1}([j-1])}\Pr[\pi(m)=i|\pi^{-1}(i)]\cdot\Pr[a^t\in\cA_m|\pi^{-1}([j-1]),\pi(m)=i] \\
&=\sum_{t\in\cT_j}\sum_{\pi^{-1}([j-1])}\Pr[\pi^{-1}([j-1])]\sum_{m\notin\pi^{-1}([j-1])}\frac{1}{n-j+1}\Pr[a^t\in\cA_m|\pi^{-1}([j-1])] \\
&\le\sum_{t\in\cT_j}\sum_{\pi^{-1}([j-1])}\Pr[\pi^{-1}([j-1])]\cdot\frac{1}{n-j+1}(1) \\
&=\frac{2b}{n-j+1}. \\
\end{align*}
The first equality is by definition and the linearity of expectation; the second and third equalities are by the law of total probability; and the fourth equality is by the fact that $a^t$ is independent of $\pi^{-1}(i)$, which completes the proof of the lemma.

\halmos
\end{proof}

Now, combining (\ref{eqn::ALG}), (\ref{eqn::splitByEps}), and definitions, we get that
\begin{equation}\label{eqn::final}
\bE[\ALG]\le\sum_{i=1}^n\min\Big\{\sum_{j=1}^i\big(\frac{1-\vare}{2}\bE[Q_{i,j}]+\vare\cdot\bE[Q^*_{i,j}]\big),b\Big\},
\end{equation}
where we have also used the fact that $\min\{\cdot,b\}$ is concave.  For all $i\in[n]$, let
\begin{equation}\label{eqn::harmonic}
H^n_{n-i}:=(1+\frac{1}{2}+\ldots+\frac{1}{n})-(1+\frac{1}{2}+\ldots+\frac{1}{n-i})=\sum_{j=1}^i\frac{1}{n-j+1}.
\end{equation}

\

Now, let $n' \in [n]$ be the largest value such that $H^n_{n-n'} \leq 1$.
\begin{align*}
& \sum_{i=1}^{n'}\sum_{j=1}^i\big(\frac{1-\vare}{2}\bE[Q_{i,j}]+\vare\cdot\bE[Q^*_{i,j}]\big)\\
\leq & \sum_{i=1}^{n'} (1-\vare)b \cdot H^n_{n-i} + \vare \cdot \sum_{j=1}^{n'} \sum_{i=j}^{n'} \bE[Q^*_{i,j}]\\
& \ \ \ \ \ \text{(by Lemma~\ref{lem::KVV})}\\
\leq& \sum_{i=1}^{n'} (1-\vare)b \cdot H^n_{n-i} + \vare \cdot \sum_{j=1}^{n'} \left(\frac{1}{K} \sum_{i =j}^{n'} \bE[\oQ_{i,j}]  +\vare  2b \sqrt{\frac{2bj}{K}} \right)  \\
& \ \ \ \ \ \text{(by Lemma~\ref{lem::simplified})}\\
=& \sum_{i=1}^{n'} (1-\vare)b \cdot H^n_{n-i} +  \frac{\vare}{K} \cdot \sum_{i=1}^{n'} \sum_{j=1}^{i} \bE[\oQ_{i,j}]  + \vare^2 2b \cdot \sum_{j=1}^{n'} \sqrt{\frac{2bj}{K}}  \\
\leq& \sum_{i=1}^{n'} (1-\vare)b \cdot H^n_{n-i} +  \frac{\vare2b}{K} \cdot \sum_{i=1}^{n'} H^n_{n-i}  + \vare^2 2b \cdot \sum_{j=1}^{n'} \sqrt{\frac{2bj}{K}}  \\
& \ \ \ \ \ \text{(by Lemma~\ref{lem::KVV})}\\
=& b \cdot \sum_{i=1}^{n'} H^n_{n-i} \cdot \left[ 1 - \vare \left(1 - \frac{2}{K}\right) \right]  +  \vare^2 2b \cdot \sum_{j=1}^{n'} \sqrt{\frac{2bj}{K}}  \\
\leq& b \cdot \sum_{i=1}^{n'} H^n_{n-i} \cdot \left[ 1 - \vare \left(1 - \frac{2}{K}\right) \right]  +  \vare^2 2b \cdot n \sqrt{\frac{2bn}{K}}.
\end{align*}
\

Since $\min\{x,y\} \leq x$ and $\min\{x,y\} \leq y$, we can obtain
\begin{align}
& \bE[\ALG] \nonumber\\
\leq & \sum_{i=1}^n\min\Big\{\sum_{j=1}^i\big(\frac{1-\vare}{2}\bE[Q_{i,j}]+\vare\cdot\bE[Q^*_{i,j}]\big),b\Big\}\nonumber\\
\leq & \sum_{i=1}^{n'} \sum_{j=1}^i\big(\frac{1-\vare}{2}\bE[Q_{i,j}]+\vare\cdot\bE[Q^*_{i,j}]\big) + \sum_{i=n'+1}^n b\nonumber\\
\leq &  b \cdot \sum_{i=1}^{n'} H^n_{n-i} \cdot \left[ 1 - \vare \left(1 - \frac{2}{K}\right) \right]  +  \vare^2 2b \cdot n \sqrt{\frac{2bn}{K}} + \sum_{i=n'+1}^n b\nonumber\\
= &  b \cdot \sum_{i=1}^{n} \min(H^n_{n-i},1) - b \cdot \sum_{i=1}^{n'} H^n_{n-i}  \cdot  \vare \left(1 - \frac{2}{K}\right) +  \vare^2 2b \cdot n \sqrt{\frac{2bn}{K}}. \label{eq:finalb}
\end{align}
The last equality is because $H^n_{n-i} \leq 1$ for all $i \leq n'$.

Make the technical assumptions $b\ge K\ge3$, and set
\begin{equation*}
\vare:= \frac{1}{34} \sqrt{\frac{K}{bn}}
\end{equation*}
which satisfies the condition that $\vare\le1/2$.

Substituting back into \eqref{eq:finalb}, we obtain
\begin{align}
& \bE[\ALG] \nonumber\\
\leq &  b \cdot \sum_{i=1}^{n} \min(H^n_{n-i},1) - b \cdot \sum_{i=1}^{n'} H^n_{n-i}  \cdot  \vare \left(1 - \frac{2}{K}\right) +  \vare^2 2b \cdot n \sqrt{\frac{2bn}{K}}\nonumber\\
= &  b \cdot \sum_{i=1}^{n} \min(H^n_{n-i},1) - \sqrt{\frac{Kb}{n}} \left[ \frac{1}{34}\left(1 - \frac{2}{K}\right)  \sum_{i=1}^{n'} H^n_{n-i}  -  \frac{\sqrt{2}}{578} n\right]\nonumber\\
\leq &  b \cdot \sum_{i=1}^{n} \min(H^n_{n-i},1) - \sqrt{\frac{Kb}{n}} \left[ \frac{1}{34}\left(1 - \frac{2}{3}\right)  \sum_{i=1}^{n'} H^n_{n-i}  -  \frac{\sqrt{2}}{578} n\right]\nonumber\\
= &  b \cdot \sum_{i=1}^{n} \min(H^n_{n-i},1) - \sqrt{\frac{Kb}{n}} \left[ \frac{1}{102} \sum_{i=1}^{n'} H^n_{n-i}  -  \frac{\sqrt{2}}{578} n\right]. \label{eq:finalc}
\end{align}

To complete the analysis, we need elementary facts about the harmonic sums $H^n_{n-i}$ defined in (\ref{eqn::harmonic}):
\begin{proposition}\label{prop::harmonic}
\begin{align}
\sum_{i=1}^{n'}H^n_{n-i}\le n-2n/e+2; \label{eqn::harmonicSmall} \\
\sum_{i=n'+1}^n\min(H^n_{n-i},1)\le n/e+1. \label{eqn::harmonicLarge}
\end{align}
\end{proposition}

\begin{proof}{Proof.}
Since $n'$ was defined to be the largest value such that $H^n_{n-n}\le1$, it can be checked that $n'=\lfloor n(1-1/e)\rfloor$.  For all $i=1,\ldots,n'$, $\min(H^n_{n-1},1)=H^n_{n-1}$, while for all $i=n'+1,\ldots,n$, $\min(H^n_{n-1},1)=1$.

Therefore, the LHS of inequality~(\ref{eqn::harmonicLarge}) equals $(n-\lfloor n(1-1/e)\rfloor)\cdot1$, which is at most $n-(n(1-1/e)-1)=n/e+1$, which equals the RHS of inequality~(\ref{eqn::harmonicLarge}).

For inequality~(\ref{eqn::harmonicSmall}), note that its LHS is at most $\sum_{i=1}^{n'}\ln(n/(n-i))$.  In turn,
\begin{align*}
\sum_{i=1}^{n'}\ln\frac{1}{1-i/n} &\le\int_{1}^{n'+1}\ln\frac{1}{1-x/n}dx \\
&\le\int_{0}^{n(1-1/e)+1}\ln\frac{1}{1-x/n}dx \\
&=n\int_{0}^{1-1/e+1/n}\ln\frac{1}{1-y}dy
\end{align*}
where the first inequality uses the fact that the function $\ln\frac{1}{1-x/n}$ is increasing over $x\in[1,n'+1]$.  The final integral can be evaluated to equal
\begin{align*}
1-1/e+1/n+(1/e-1/n)\ln(1/e-1/n)
\end{align*}
which is at most $1-1/e+1/n+(1/e-1/n)(-1)=1-2/e+2/n$ as long as $n\ge3$.  This completes the proof of inequality~(\ref{eqn::harmonicSmall}).

\halmos
\end{proof}

Applying Proposition~\ref{prop::harmonic} to expression~(\ref{eq:finalc}) and using the fact that $b-\sqrt{Kb/n}/102>0$, we bound expression~(\ref{eq:finalc}) from above by
\begin{align*}
& b n \left( 1 - \frac{1}{e} +\frac{3}{n}\right)-\sqrt{\frac{Kb}{n}} \left[ \frac{1}{102} \left(1 - \frac{2}{e}+\frac{2}{n}\right) n  -  \frac{\sqrt{2}}{578} n\right] \\
\leq & b n \left( 1 - \frac{1}{e} \right)+3b-\frac{\sqrt{nKb}}{C} .
\end{align*}
$C>1$ is an absolute constant.  As long as $b\le n$ and $K$ is sufficiently large, the inequality $3b<\sqrt{nKb}/C$ holds. Since $\OPT = bn$ for all realization of $\pi$ and $\vell$, this completes the proof of Theorem \ref{thm:lowerBound}.